\documentclass[11pt]{article}
\usepackage[utf8]{inputenc} 
\usepackage[T1]{fontenc}    
\usepackage{wrapfig}

\usepackage{setspace}
\usepackage{cprotect}
\usepackage{amsmath,amssymb,amsthm}
\usepackage[noend]{algorithmic}
\usepackage[ruled,vlined]{algorithm2e}
\usepackage{hyperref}
\usepackage{fullpage}
\usepackage{makeidx}
\usepackage{enumerate}
\usepackage{graphicx,float,psfrag,epsfig}
\usepackage{epstopdf}
\usepackage{color}
\usepackage{enumitem}
\usepackage{subfig}
\usepackage{caption}
\usepackage{bigints}
\usepackage{mathtools}
\usepackage[mathscr]{euscript}

\newcommand{\simone}[1]{\textcolor{black}{#1}}

\usepackage{amsmath}
\usepackage{amssymb}
\usepackage{mathtools}
\usepackage{amsthm}
\usepackage{amsfonts}
\usepackage{soul}
\usepackage{csquotes}
\usepackage{amsmath}
\usepackage{amsthm}
\usepackage{amsfonts}
\usepackage{hyperref}
\usepackage{comment}
\usepackage{graphicx}
\usepackage{xcolor}
\usepackage{parskip}
\usepackage{hyperref}       
\usepackage{url}            
\usepackage{booktabs}       
\usepackage{amsfonts}       
\usepackage{nicefrac}       
\usepackage{microtype}      
\usepackage{xcolor}         
\usepackage{microtype}
\usepackage{graphicx}
\usepackage{subfig}
\usepackage{amsfonts}
\usepackage{amsmath}
\usepackage{amssymb}
\usepackage{bbm}
\usepackage{multirow}
\usepackage{mathtools}
\usepackage{relsize}

\def\P{\mathbb{P}}

\def\diag{\mathrm{diag}}

\def\R{\mathbb{R}}

\def\cJ{\Tilde{J}}

\newcommand{\opnorm}[1]{\left\lVert#1\right\rVert_{\textup{op}}}

\def\b0{{0}}

\def\RR{\mathbb{R}}

\def\>{\rangle}

\def\vec{\operatorname{\mathop{vec}}}
\def\diag{\operatorname{\mathop{diag}}}

\def\Set#1{\left\{ #1 \right\}}

\newcommand{\E}{\mathbb{E}}

\newcommand{\distas}[1]{\mathbin{\overset{#1}{\sim}}}

\newcommand{\bigO}[1]{\mathcal{O}\left(#1\right)}
\newcommand{\bigOmg}[1]{\Omega\left(#1\right)}

\newcommand{\norm}[1]{\left\|#1\right\|}
\newcommand{\subGnorm}[1]{\left\|#1\right\|_{\psi_2}}
\newcommand{\subEnorm}[1]{\left\|#1\right\|_{\psi_1}}

\newcommand{\abs}[1]{\left|#1\right|}

\newcommand{\svmin}[1]{\sigma_{\rm min}\left(#1\right)}

\newcommand{\evmin}[1]{\lambda_{\rm min}\left(#1\right)}

\def\Lip{\mathrm{Lip}}
\def\op{\mathrm{op}}

\def\PP{\mathbb{P}}

\def\min{\mathop{\rm min}\nolimits}
\def\max{\mathop{\rm max}\nolimits}

\def\ie{\textit{i.e. }}

\numberwithin{equation}{section}

\newtheoremstyle{myexample} 
    {\topsep}                    
    {\topsep}                    
    {\rm }                   
    {}                           
    {\bf }                   
    {.}                          
    {.5em}                       
    {}  

\newtheoremstyle{myremark} 
    {\topsep}                    
    {\topsep}                    
    {\rm}                        
    {}                           
    {\bf}                        
    {.}                          
    {.5em}                       
    {}  

\newtheorem{claim}{Claim}[section]
\newtheorem{lemma}[claim]{Lemma}

\newtheorem{assumptions}{Assumption}

\newtheorem{theorem}{Theorem}
\newtheorem{proposition}[claim]{Proposition}
\newtheorem{corollary}[claim]{Corollary}

\theoremstyle{myremark}

\theoremstyle{myremark}

\theoremstyle{myexample}

\author{Simone Bombari\thanks{Institute of Science and Technology Austria (ISTA). Emails: \texttt{\{simone.bombari, marco.mondelli\}@ist.ac.at}.}\;,
\;\;Mohammad Hossein Amani\thanks{EPFL, Switzerland. Email: \texttt{mh.amani1998@gmail.com}.}\;,
\;\;Marco Mondelli\footnotemark[1].}

\title{Memorization and Optimization in Deep Neural Networks with Minimum Over-parameterization}

\begin{document}

\newtheorem*{theoremcentering}{Theorem \ref{thm:maincentering}}
\newtheorem*{theoremcentered}{Theorem \ref{thm:centered}}
\newtheorem*{corhammer}{Corollary \ref{cor:hammer}}
\newtheorem*{cormem}{Corollary \ref{cor:memcap}}
\newtheorem*{theoremoptim}{Theorem \ref{thm:optimization}}

\maketitle

\begin{abstract}
The Neural Tangent Kernel (NTK) has emerged as a powerful tool to provide memorization, optimization and generalization guarantees in deep neural networks.
A line of work has studied the NTK spectrum for two-layer and deep networks with at least a layer with $\Omega(N)$ neurons, $N$ being the number of training samples. Furthermore, there is increasing evidence suggesting that deep networks with sub-linear layer widths are powerful memorizers and optimizers, as long as the number of parameters exceeds the number of samples. Thus, a natural open question is whether the NTK is well conditioned in such a challenging sub-linear setup. In this paper, we answer this question in the affirmative. Our key technical contribution is a lower bound on the smallest NTK eigenvalue for deep networks with the \emph{minimum possible over-parameterization}: \simone{up to logarithmic factors, the number of parameters is $\Omega(N)$} and, hence, the number of neurons is as little as $\Omega(\sqrt{N})$. To showcase the applicability of our NTK bounds, we provide two results concerning memorization capacity and optimization guarantees for gradient descent training. 
\end{abstract}

\section{Introduction}


Training a neural network is a non-convex problem that exhibits disconnected local minima \cite{Auer96, SafranShamir2018,Yun2019}. Yet, in practice gradient descent (GD) and its variants routinely find solutions with zero training loss \cite{Zhang2017}. A framework to understand this phenomenon comes from the 
Neural Tangent Kernel (NTK). This quantity was introduced in \cite{JacotEtc2018}, where it was proved that, during GD training, the network follows the kernel gradient of the functional cost with respect to the NTK. Furthermore, as the layer widths go large, the NTK converges to a deterministic limit which stays constant during training. Hence, in this infinite-width limit, it suffices that the smallest eigenvalue of the NTK is bounded away from $0$ for gradient descent to reach zero loss. Going to finite widths, a recipe to prove GD convergence can be summarized as follows: show that \emph{(i)} the NTK is well conditioned at initialization, and \emph{(ii)} the NTK has not changed significantly by the time GD has reached zero loss (see e.g. \cite{oymak2019overparameterized,chizat2019lazy,bartlett2021deep}). A number of papers have exploited this recipe 
for networks with progressively smaller over-parameterization: two-layer networks \cite{DuEtal2018_ICLR, OymakMahdi2019, SongYang2020,wu2019global,song2021subquadratic}, deep networks with polynomially wide layers 
\cite{AllenZhuEtal2018,DuEtal2019,zou2020gradient,ZouGu2019}, and deep networks with a single wide layer \cite{QuynhMarco2020,nguyen2021proof}. Besides optimization, showing that the NTK is well conditioned directly implies a result on memorization capacity \cite{Andrea2020,tightbounds}, and the smallest eigenvalue of the NTK has also been related to generalization \cite{arora2019fine}. 

In \cite{tightbounds}, it is shown that, given $N$ training samples, a single layer with $\Omega(N)$ neurons suffices for the NTK to be well conditioned in networks of arbitrary depth. However, there is increasing evidence that, in the challenging setup in which the layer widths are \emph{sub-linear} in $N$, neural networks still memorize the training data \cite{bubeck2020network,yun2019small,vershynin2020memory}, reach zero loss under GD training in the two-layer setting \cite{theoreticalinsghts,OymakMahdi2019,Andrea2020}, and in the deep case, GD explores a nicely behaved region of the loss landscape \cite{nguyen2021solutions}. \simone{This is in agreement with simple back-of-the-envelope calculations on CIFAR-10 and ImageNet: CIFAR-10 has $N=50000$ images and roughly $10^6$ parameters suffice to fit random labels \cite{Zhang2017}; furthermore, in order to fit random labels to a subset of $1.2 \cdot 10^6$ ImageNet data points, $2.4 \cdot 10^7$ parameters are enough \cite{Zhang2017}. These numbers suggest that having a number of parameters of the same order as the dataset size is much closer to practice than having a number of neurons of that order.}
We also note that, by counting degrees of freedom or bounding the VC dimension \cite{bartlett2019nearly}, $\Omega(N)$ parameters (and, therefore,  $\Omega(\sqrt{N})$ neurons) are in general necessary to fit $N$ data points. This naturally brings forward the following open question:

\vspace{1em}
\begin{center}
\textit{
Is the NTK well conditioned for deep networks with the minimum possible over-parameterization (i.e., containing $\Omega(N)$ parameters corresponding to $\Omega(\sqrt{N})$ neurons)?}
\end{center}

\paragraph{Main contributions.} In this paper, \emph{we settle this open question} for a large class of deep networks. We consider \emph{(i)} a smooth activation function, \emph{(ii)} i.i.d. data satisfying Lipschitz concentration (e.g., data with a Gaussian distribution, uniform on the sphere/hypercube, or obtained via a Generative Adversarial Network), \emph{(iii)} a standard initialization of the weights (e.g., He's or LeCun's initialization), and \emph{(iv)} a loose pyramidal topology in which the layer widths can increase by at most a multiplicative constant as the network gets deep. Then, in Theorem \ref{thm:main} we show that the NTK is well conditioned under the \emph{minimum possible over-parameterization} requirement: \simone{the number of parameters between the last two layers has to be $\tilde\Omega(N)$ or, equivalently, the number of neurons has to be $\tilde\Omega(\sqrt{N})$, where $\tilde\Omega$ includes extra logarithmic factors}. We achieve this goal by giving a lower bound on the smallest eigenvalue of the NTK. This lower bound is tight when all the layer widths are of the same order.

Our NTK bounds open the way towards understanding the behavior of deep networks with minimum over-parameterization. In particular, an immediate consequence of the fact that the NTK is well conditioned is a result on the memorization capacity (Corollary \ref{cor:memcap}). Furthermore, by suitably choosing the initialization, we provide convergence guarantees for gradient descent training (Theorem \ref{thm:optimization}). Finally, we highlight that, in order to obtain our bounds on the smallest NTK eigenvalue, we give a number of tight estimates on the $\ell_2$ norms of feature vectors and of their centered counterparts, which may be of independent interest.


\paragraph{Proof ideas.} To prove Theorem \ref{thm:main}, we restrict to the kernel $K_{L-2}=J_{L-2}J_{L-2}^\top$, where $J_{L-2}$ is the Jacobian of the output w.r.t. the parameters between the last two hidden layers. This suffices as $K_{L-2}$ is a lower bound on the NTK in the positive semi-definite (PSD) sense. We note that the $i$-th row ($i\in \{1, \ldots, N\}$) of $J_{L-2}$ is given by the Kronecker product $\otimes$ between the feature vector at layer $L-2$ and the backpropagation term from the same layer. One key technical hurdle is to center $J_{L-2}$, so that its rows have the form $u \otimes v - \E \left[ u \otimes v \right]$, where $\E[u]=\E[v]=0$, and all expectations are taken with respect to the (random) training data. To do so, we perform three 
steps of centering: \emph{(i)} we center the feature vectors (corresponding to $u$), \emph{(ii)} we center the backpropagation terms (corresponding to $v$), and \emph{(iii)} we center again the whole row (corresponding to $u\otimes v$). These centering steps are approximate in the sense that the centered matrix is \emph{not necessarily} a lower bound (in the PSD sense) on the original one. However, we are able to control the operator norm of the difference, and show that it scales slower than the smallest eigenvalue of the centered kernel.
At this point, we leverage the structure of the rows of the centered Jacobian to bound their sub-exponential norm via a version of the Hanson-Wright inequality for (weakly) correlated random vectors \cite{HWconvex}. Finally, after providing also a tight estimate on the $\ell_2$ norms of such rows, we can exploit a result from \cite{hammer} to lower bound the smallest singular value of a matrix whose rows are independent random vectors with well controlled sub-exponential and $\ell_2$ norms. 

Existing work bounds the smallest NTK eigenvalue for networks with two layers \cite{theoreticalinsghts,Andrea2020}, or deep networks with a layer containing $\Omega(N)$ \emph{neurons} \cite{tightbounds}. 
In particular, \cite{theoreticalinsghts} also exploits the results \cite{HWconvex,hammer}. However, the centering of the Jacobian is achieved via a combination of whitening and dropping rows, which appears to be difficult to generalize to deep networks. In contrast, the 3-step centering we described above applies to networks of arbitrary depth $L$. A different approach is put forward in \cite{Andrea2020}, and it is based on a decomposition of the kernel via spherical harmonics. This technique allows to obtain the exact limit of the smallest NTK eigenvalue, 
it has been used to analyze random feature models \cite{ghorbani2021linearized,mei2022generalization,ghorbani2020neural,mei2021generalization} and to obtain generalization bounds for two-layer networks (see again \cite{Andrea2020}). However, understanding how to carry out such a decomposition in the multi-layer setup is an open problem. Finally, \cite{tightbounds} considers the deep case, and it relates the smallest eigenvalue of the NTK to the smallest singular value of a feature matrix. As feature matrices are full rank only when the number of neurons is $\Omega(N)$, this approach is inherently limited to networks with a linear-width layer. 

The rest of the paper is organized as follows: Section \ref{sec:setting} discusses the problem setup and our model assumptions; Section \ref{sec:main} presents our main result on the smallest NTK eigenvalue and gives a roadmap of the argument; Section \ref{sec:appl} provides two applications of our NTK bounds: memorization capacity and gradient descent training; Section \ref{sec:rel} discusses additional related work, and Section \ref{sec:concl} provides some concluding remarks. The details of the proofs are deferred to the appendices.  

\section{Preliminaries}\label{sec:setting}

\paragraph{Neural network setup.} 
We consider an $L$-layer neural network with feature maps $f_l: \RR^d \to \RR^{n_l}$ defined for every $x \in \RR^d$ as
\begin{align}\label{eq:def_feature_map}
	f_l(x)=\begin{cases}
		x, & l=0,\\
		\phi(W_l^\top f_{l-1}), & l\in[L-1],\\
		W_L^\top f_{L-1}, & l=L. 
	\end{cases}
\end{align}
Here, $W_l \in \RR^{n_{l-1} \times n_l}$ is the weight matrix at layer $l$, $\phi$ is the activation function and, given an integer $n$, we use the shorthand $[n]=\{1, \ldots, n \}$. 
We assume that the network has a single output, \ie $n_L=1$ and $W_L \in \RR^{n_{L-1}}$, and for consistency we have $n_0=d$. Let $g_l: \RR^d \to \RR^{n_l}$ be the pre-activation feature map so that $f_l(x)=\phi(g_l(x))$ for $l\in [L]$. We define $g_0(x) = f_0(x) = x$.
Let $X=[x_1, \ldots, x_N]^\top \in \RR^{N \times d}$ be the data matrix containing $N$ samples in $\RR^d$, 
$\theta = [\vec(W_1), \ldots, \vec(W_L)]$ be the vector of the parameters of the network,
and $F_L(\theta)=[f_L(x_1), \ldots, f_L(x_N)]^\top$ be the network output. We denote by $J$ the Jacobian of $F_L$ with respect to all the parameters of the network:

\begin{equation}\label{eq:Jac}
	J =\left[\frac{\partial F_L}{\partial\vec(W_1)}, \ldots, \frac{\partial F_L}{\partial\vec(W_L)}\right] \in \RR^{N\times\sum_{l=1}^Ln_{l-1}n_l}.
\end{equation}
Our key object of interest is the \emph{empirical Neural Tangent Kernel (NTK) Gram matrix}, denoted by $K \in\RR^{N\times N}$ and defined as:
\begin{equation}\label{eq:NTKgramdef}
	K
	=J J^T
	=\sum_{l=1}^{L} \left[\frac{\partial F_L}{\partial\vec(W_l)}\right] \left[\frac{\partial F_L}{\partial\vec(W_l)}\right]^\top.
\end{equation}
In \cite{JacotEtc2018}, it is shown that, as $n_l\to \infty$ for all $l\in [L-1]$, $K$ converges to a deterministic \emph{limit}, which stays constant during gradient descent training. The focus of this paper is on the \emph{finite-width} behavior of the empirical NTK \eqref{eq:NTKgramdef}. Quantitative bounds for the NTK convergence rate can be obtained from \cite{AroraEtal2019,buchanan2021deep}. However, these bounds lead to a significant over-parameterization requirement, see the discussion at the end of Section 3 in \cite{tightbounds}. Here, our main result consists in showing that the NTK is well conditioned for a class of neural networks with the \emph{minimum possible over-parameterization}, i.e., $\Omega(N)$ parameter or, equivalently, $\Omega(\sqrt{N})$ neurons. 

\paragraph{Weight and data distribution.} We consider the following initialization of the weight matrices: 
$(W_l)_{i,j}\distas{}_{\rm i.i.d.}\mathcal{N}(0,\beta_l^2 / n_{l-1})$ for $l\in[L-1], i\in[n_{l-1}], j\in[n_l]$, where $\beta_l$ is a numerical constant independent of the layer widths. This covers the popular cases of He's and LeCun's initialization \cite{XavierBengio2010,he2015delving,lecun2012efficient}. For the last layer, we assume that $(W_L)_{i}\distas{}_{\rm i.i.d.}\mathcal{N}(0,\beta_L^2)$ for $i\in [n_{L-1}]$. Throughout the paper, we let $(x_1,\ldots,x_N)$ be $N$ i.i.d.\ samples from the data distribution $P_X$, which satisfies the conditions below.

\begin{assumptions}[Data scaling]\label{ass:data_dist}
	The data distribution $P_X$ satisfies the following properties:
	\begin{enumerate}
		\item $\int \norm{x}_2 dP_X(x)=\Theta(\sqrt{d}).$
		\item $\int \norm{x}_2^2 dP_X(x)=\Theta(d).$
		\item $\int \norm{x-\int x'\, dP_X(x')}_2^2 dP_X(x)=\bigOmg{d}.$
	\end{enumerate}
\end{assumptions}

\begin{assumptions}[Lipschitz concentration]\label{ass:data_dist2}
	The data distribution $P_X$ satisfies the \emph{Lipschitz concentration property}. Namely, there exists an absolute constant $c>0$ such that, for every Lipschitz continuous function $\varphi: \RR^d \to \RR$, we have $\E |\varphi(X)| < + \infty$, and for all $t>0$,
	$$
	\PP\left(\abs{\varphi(x)-\int \varphi(x')\, dP_X(x')}>t\right)
	\leq 2e^{-ct^2 / \norm{\varphi}_{\Lip}^2}.
	$$
\end{assumptions}

Assumption \ref{ass:data_dist} is simply a scaling of the training data points and their centered counterparts. Assumption \ref{ass:data_dist2} covers a number of important cases, e.g., standard Gaussian distribution \cite{vershynin2018high}, uniform distributions on the sphere and on the unit (binary or continuous) hypercube \cite{vershynin2018high}, 
data produced via a Generative Adversarial Network (GAN)\footnote{By applying a Lipschitz map to a standard Gaussian distribution, the map output satisfies Assumption \ref{ass:data_dist2}.} \cite{seddik2020random}, and more generally any distribution satisfying the log-Sobolev inequality with a dimension-independent constant. We also remark that Assumption \ref{ass:data_dist2} is rather common in the related literature \cite{tightbounds}, or it is even replaced by a stronger requirement (e.g., Gaussian distribution or uniform on the sphere) \cite{Andrea2020}.

\begin{assumptions}[Activation function]\label{ass:activationfunc}
	The activation function $\phi$ satisfies the following properties:
	\begin{enumerate}
		\item $\phi$ is a non-linear (and therefore also non-constant) $M$-Lipschitz function;
		\item its derivative $\phi'$ is a $M'$-Lipschitz function;
	\end{enumerate}
\end{assumptions}
These requirements are satisfied by common activations, e.g. smoothed ReLU, sigmoid, or $\tanh$. 


\begin{assumptions}[Network topology]\label{ass:topology}
	The network satisfies a loose pyramidal topology condition, \ie $n_l = \bigO{n_{l-1}}$,	for all $l \in [L-1]$.
\end{assumptions}

A \emph{strict} pyramidal topology (namely, non-increasing layer widths) has been considered in prior work concerning the loss landscape \cite{QuynhICML2017,QuynhICML2018} and gradient descent training \cite{QuynhMarco2020}. Our Assumption \ref{ass:topology} requires a \emph{loose} pyramidal topology in the sense that, as we go deep, the layer widths are allowed to increase by a constant multiplicative factor. We also note that the widths of neural networks used in practice are often large in the first layers and then start decreasing \cite{han2017deep,VGG}.

\begin{assumptions}[Over-parameterization]\label{ass:overparam}
    \simone{We have that
    \begin{equation}\label{eq:overparamcond}
        N\cdot \log^8 N = o(n_{L-2}n_{L-1}).
    \end{equation}}
    Furthermore, there exists $\gamma > 0$ such that
    \begin{equation}\label{eq:polycomp}
        N^{\gamma} = \bigO{n_{L-1}}.
    \end{equation}
\end{assumptions}
Condition \eqref{eq:overparamcond} requires the number of parameters between the last two hidden layers to be linear in the number of samples\simone{, up to logarithmic factors}. This represents our \emph{key over-parameterization condition}. \simone{When all the widths have the same scaling, \eqref{eq:overparamcond} reduces to $N=\tilde o(d^2)$. This is satisfied by several ``standard'' datasets, such as MNIST ($N=60000$, $d=784$), CIFAR-10 ($N=5\cdot 10^4$, $d=3\cdot 32^2$), and ImageNet ($N=1.4\cdot 10^7$, $d=2\cdot 10^5$). We also note that, if $N\gg d^2$, the NTK is low-rank and $\evmin{K} =0$. Furthermore, Corollary \ref{cor:memcap} and Theorem \ref{thm:optimization} cannot generally hold when $N\gg d^2$, as there are more data points to fit than parameters to help with the fitting. A milder requirement is possible for classification tasks, see the discussion at the end of Section \ref{sec:rel}.} 
The second condition \eqref{eq:polycomp} is rather mild, as $\gamma$ can be taken to be arbitrarily small, and it avoids an exponential bottleneck in the last hidden layer.

    

\paragraph{Notation.} 
The feature matrix at layer $l$ is
$F_l = [f_l(x_1),\ldots,f_l(x_N)]^\top \in \RR^{N\times n_l}$, 
and $\Sigma_l(x)=\diag([\phi'(g_{l,j}(x))]_{j=1}^{n_l})$ for $l\in[L-1]$,
where $g_{l,j}(x)$ is the pre-activation neuron. Given two matrices $F,B\in\RR^{m\times n}$, we denote by $F\circ B$ their Hadamard product, 
and by $F\ast B=[(F_{1:}\otimes B_{1:}),\ldots,(F_{m:}\otimes B_{m:})]^T\in\RR^{m\times n^2}$ their row-wise Kronecker product (also known as Khatri-Rao product). Given a random vector $v$, let $\norm{v}_{\psi_2}$ and $\norm{v}_{\psi_1}$ denote its sub-Gaussian and sub-exponential norm, respectively (see also Appendix \ref{app:notation} for a detailed definition). Given a matrix $A$,
let $\norm{A}_{\op}$ be its operator norm, $\norm{A}_{F}$ its Frobenius norm, $\evmin{A}$ its smallest eigenvalue, and $\svmin{A}$ its smallest singular value.
We denote by $\norm{\varphi}_{\Lip}$ the Lipschitz constant of the function $\varphi$. 
All the complexity notations $\Omega(\cdot)$, $\mathcal{O}(\cdot)$, $o(\cdot)$ and $\Theta(\cdot)$ are understood for sufficiently large $N,d,n_1,n_2,\ldots,n_{L-1}$. \simone{Tildes on such symbols are meant to neglect logarithmic factors.} 

\section{NTK Bounds with Minimum Over-parameterization}\label{sec:main}

Our main technical contribution on the smallest eigenvalue of the empirical NTK  \eqref{eq:NTKgramdef} is stated below.

\begin{theorem}[Smallest NTK eigenvalue under minimum over-parameterization]\label{thm:main}
    Consider an $L$-layer neural network \eqref{eq:def_feature_map}, where the activation function satisfies Assumption \ref{ass:activationfunc} and the layer widths satisfy Assumptions \ref{ass:topology} and \ref{ass:overparam}.
    Let $\Set{x_i}_{i=1}^{N}\sim_{\rm i.i.d.}P_X$, 
    where $P_X$ satisfies the Assumptions \ref{ass:data_dist}-\ref{ass:data_dist2},
    and let $K$ be the empirical NTK Gram matrix \eqref{eq:NTKgramdef}.
    Assume that the weights of the network are initialized as
    $(W_l)_{i,j}\sim_{\rm i.i.d.}\mathcal{N}(0,\beta_l^2/n_{l-1})$ for $l\in[L-1]$, $i\in[n_{l-1}]$, $j\in [n_l]$ and $(W_L)_i\sim_{\rm i.i.d.}\mathcal N(0, \beta_L^2)$ for $i\in [n_{L-1}]$. Then, we have 
	\begin{equation}\label{eq:lbNTK}
		\evmin{K} = \Omega(n_{L-2} n_{L-1}),
	\end{equation}
	with probability at least
	$1 - C\,N e^{-c \log^2 n_{L-1}} - C e^{-c \log^2 N}
	$ over $(x_i)_{i=1}^N$ and $(W_k)_{k=1}^L$, where $c$ and $C$ are numerical constants.
	Moreover, we have that
	\begin{equation}\label{eq:ubNTK}
		\evmin{K} = \bigO{d n_{L-1}},
	\end{equation}
	with probability at least
	$1 - C e^{-c n_{L-1}}$, over $(x_i)_{i=1}^N$ and $(W_k)_{k=1}^L$.
\end{theorem}

\begin{wrapfigure}{r}{0.5\textwidth}
  \begin{center}
    \includegraphics[width=0.5\textwidth]{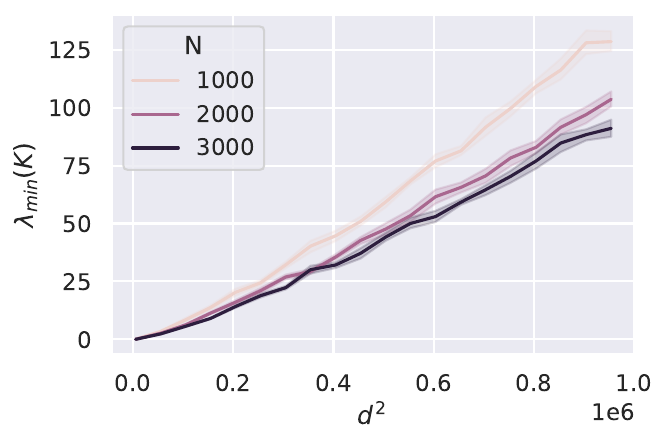}
  \end{center}
  \caption{$\evmin{K}$ as a function of $d^2$ in a 3-layer neural network, with sigmoid activation and $d = n_1 = n_2$.}
  \label{fig:scalingbody}
\end{wrapfigure}

This result implies that the NTK is well conditioned for a class of networks in which the number of parameters $n_{L-2}n_{L-1}$ between the last two hidden layers \simone{is linear in the number of data samples $N$, up to logarithmic factors.} This means that the total number of neurons of the network can be as little as $\tilde\Omega(\sqrt{N})$, which meets the \emph{minimum possible amount of over-parameterization}. 
The lower bound \eqref{eq:lbNTK} and the upper bound \eqref{eq:ubNTK} match when all the widths (up to layer $L-2$) have the same scaling, i.e., $d=\Theta(n_{L-2})$. Throughout the paper, we do not explicitly track the dependence of our bounds on $L$, in the sense that the numerical constants $c, C$ may depend on $L$. 

In Figure \ref{fig:scalingbody}, we consider a 3-layer neural network with $d = n_1 = n_2$, and we plot $\evmin{K}$ as a function of $d^2$, for three different values of $N$. The inputs are sampled from a standard Gaussian distribution, the activation function is the sigmoid $\sigma(x)=(1+e^{-x})^{-1}$, and we set $\beta_l = 1$ for all $l \in [L]$. We repeat the experiment 10 times, and report average and confidence interval at 1
standard deviation. The linear scaling of $\evmin{K}$ in $d^2$ is in agreement with the result of Theorem \ref{thm:main}. The code used to obtain the results of Figure \ref{fig:scalingbody} (and Figure \ref{fig:optimizationbody} as well) is available at \href{https://github.com/simone-bombari/smallest-eigenvalue-NTK/}{\texttt{https://github.com/simone-bombari/smallest-eigenvalue-NTK/}}.


\subsection{Roadmap of the proof of Theorem \ref{thm:main}}\label{subsec:sketch}

After an application of the chain rule and some standard manipulations, we have
\begin{equation}\label{eq:sumterms}
	JJ^\top = \sum_{k=0}^{L-1} K_k, \qquad \mbox{ with } K_k= F_{k} F_{k}^\top \circ B_{k+1}B_{k+1}^\top\in \mathbb R^{N\times N},
\end{equation}
where $B_k\in\RR^{N\times n_k}$ is a matrix whose $i$-th row is given by
\begin{equation}\label{eq:Bmatr}
	(B_k)_{i:} = 
	\begin{cases}
		\Sigma_{k}(x_i) \left(\prod_{l=k+1}^{L-1} W_l\Sigma_l(x_i)\right) W_L, &   k\in[L-2],\\
		\Sigma_{L-1}(x_i)W_L, & k=L-1,\\
		1, &  k=L.
	\end{cases}
\end{equation}

From the decomposition \eqref{eq:sumterms}, one readily obtains that $K \succeq K_{L-2}$, since $K_k$ is PSD for all $k\in \{0, \ldots, L-1\}$. Hence, it suffices to prove a lower bound on $\evmin{K_{L-2}}$.
We denote by $J_{L-2}$ the Jacobian obtained by computing the gradient only over the parameters between layer $L-2$ and layer $L-1$. Then, $K_{L-2}=J_{L-2} J_{L-2}^\top$ and, by using \eqref{eq:sumterms} and \eqref{eq:Bmatr}, the $i$-th row of $J_{L-2}$ can be expressed as
\vspace{0.5em}
\begin{equation}\label{eq:Jacobian}
	(J_{L-2})_{i:} = f_{L-2}(x_i) \otimes (D_L \phi'(W^\top_{L-1} f_{L-2}(x_i))), \qquad i\in [N], \vspace{0.5em}
\end{equation}
where 
$D_L$ is a diagonal matrix with the entries of $W_L$ on its diagonal. Note that, if we fix the weight matrices $(W_l)_{l=1}^L$, the rows $((J_{L-2})_{i:})_{i=1}^N$ are i.i.d., as the training data $(x_i)_{i=1}^N$ are i.i.d. too. 

The proof consists of two main parts. First, we construct the centered Jacobian $\tilde J_{L-2}$, which is obtained from $J_{L-2}$ by iteratively removing expectations with respect to $X$ to its parts, and we show that its smallest singular value is \emph{close} to the smallest singular value of $J_{L-2}$. The details of this part are contained in Appendix \ref{app:centeringJ}. Second, we bound the smallest eigenvalue of the kernel obtained from the centered Jacobian $\tilde J_{L-2}$ by providing an accurate estimate of the $\ell_2$ and sub-exponential norms of its rows. The details of this part are contained in Appendix \ref{app:E}. In order to carry out this program, we exploit a number of concentration results on the $\ell_2$ norms of feature and backpropagation vectors, and on the $\ell_2$ norms of their centered counterparts. These results are contained in Appendix \ref{app:concentration}, and they could be of independent interest. In particular, we provide tight high-probability estimates on \emph{(i)} $\norm{f_l(x)}_2$, \emph{(ii)} $\E_x [ \norm{f_{l}(x)}_2^2]$, \emph{(iii)} $\E_x [ \norm{f_{l}(x)}_2 ]$, \emph{(iv)} $\E_x [ \norm{f_{l}(x) - \E_x [f_{l}(x)] }_2^2]$,  $\E_x [ \norm{f_{l}(x) - \E_x [f_{l}(x)] }_2]$, and $\norm{f_{l}(x) - \E_x [f_{l}(x)] }_2$, and \emph{(v)} $\norm{D_L\phi'(g_{L-1}(x)) - \E_x [D_L\phi'(g_{L-1}(x)) ] }_2$. Some preliminary calculations are also contained in Appendix \ref{app:useful}. 



\paragraph{Part 1: Centering.} We consider the centered Jacobian $\tilde{J}_{L-2}$, whose 
$i$-th row is defined as
\begin{equation}\label{eq:centeredJ}
		(\tilde{J}_{L-2})_{i:} = \tilde f_{L-2}(x_i) \otimes (D_L \tilde{\phi'}(g_{L-1}(x_i))) - \E_{x_i} \left[ \tilde f_{L-2}(x_i) \otimes D_L \tilde{\phi'}(g_{L-1}(x_i)) \right], \quad i\in [N],
\end{equation}
where $\tilde f_{L-2}(x_i) = \tilde \phi (g_{L-2}(x_i))$, with $\tilde \phi (g_{L-2}(x_i)) = \phi (g_{L-2}(x_i)) - \E_{x_i} [\phi (g_{L-2}(x_i))]$ and $\tilde \phi' (g_{L-1}(x_i)) = \phi' (g_{L-1}(x_i)) - \E_{x_i}[ \phi' (g_{L-1}(x_i))]$ being the centered versions of the activation function and of its derivative, respectively.  
Our strategy is to relate $\lambda_{\rm min}(J_{L-2} J_{L-2}^\top)$ to $\lambda_{\rm min}(\tilde J_{L-2} \tilde{J}_{L-2}^\top)$ via the following two steps. 

\emph{\underline{Step (a): Centering $F_{L-2}$ and $B_{L-1}$.}} \simone{We show that $\lambda_{\rm min}(J_{L-2} J_{L-2}^\top)\ge \lambda_{\rm min}(\cJ_{FB} \cJ_{FB}^\top)-o(n_{L-2}n_{L-1})$, where $\cJ_{FB} \cJ_{FB}^\top = \Tilde F_{L-2} \Tilde F_{L-2}^\top \circ \Tilde B_{L-1} \Tilde B_{L-1}^\top$ and $\tilde{F}_{L-2}=F_{L-2}-\E_X[F_{L-2}]$ and $\tilde{B}_{L-1}=B_{L-1}-\E_X[B_{L-1}]$ (Lemma \ref{lemma:cent1} in Appendix \ref{app:centeringJ1}). 
Our strategy differs from existing work (e.g., \cite{theoreticalinsghts, tightbounds}), where either only the backpropagation term is centered (cf. Proposition 7.1 of \cite{theoreticalinsghts}) or only the features are centered (cf. Lemma 5.4 of \cite{tightbounds}). We remark that, in order to handle deep networks with minimum overparameterization\footnote{In contrast, \cite{theoreticalinsghts} considers shallow networks, and \cite{tightbounds} requires the existence of a layer with roughly $\Omega(N)$ neurons.}, the centering of $F_{L-2}$ and $B_{L-1}$ is crucially performed \emph{at the same time}. In fact, by doing so, certain terms containing an eigenvalue which is negative and large in modulus suitably cancel. As a result, the difference between the original kernel $J_{L-2} J_{L-2}^\top$ and the centered one $\cJ_{FB} \cJ_{FB}^\top$ can be bounded by a matrix whose operator norm is $o(n_{L-2}n_{L-1})$.} 

\emph{\underline{Step (b): Centering everything.}} \simone{We show that $\lambda_{\rm min}(\cJ_{FB} \cJ_{FB}^\top)\ge \lambda_{\rm min}(\cJ_{L-2} \cJ_{L-2}^\top)-o(n_{L-2}n_{L-1})$, where the rows of $\cJ_{L-2}$ are given by \eqref{eq:centeredJ} (Lemma \ref{lemma:cent3} in Appendix \ref{app:centeringJ3}). 
The idea is to decompose the difference $\cJ_{FB} \cJ_{FB}^\top-\cJ_{L-2} \cJ_{L-2}^\top$ into a rank-1 matrix plus a PSD matrix. Then, in order to bound the operator norm of the rank-1 term, we leverage the general version of the Hanson-Wright inequality given by Theorem 2.3 in \cite{HWconvex} (see the tail bound on quadratic forms that we provide in Lemma \ref{lemma:HW}). 
This strategy avoids whitening and dropping rows (as done in \cite{theoreticalinsghts}) and, hence, appears to be better suited to the deep case.}

The main result of this part is stated below, and it is proved by combining Lemmas \ref{lemma:cent1} and 
\ref{lemma:cent3}.

\begin{theorem}[Jacobian centering]\label{thm:maincentering}
	Consider the setting of Theorem \ref{thm:main}, and let the rows of $J_{L-2}$ and $\cJ_{L-2}$ be given by \eqref{eq:Jacobian} and \eqref{eq:centeredJ}, respectively. Then, we have 
	\begin{equation}
		\evmin{J_{L-2}J_{L-2}^\top} \geq \evmin{\cJ_{L-2} \cJ_{L-2}^\top} - o (n_{L-2}n_{L-1}),  
	\end{equation}
	with probability at least $1 - C \exp(-c \log^2 n_{L-1}) - C \exp(-c \log^2 N)$ over $(x_i)_{i=1}^N$ and $(W_k)_{k=1}^L$, where $c, C$ are numerical constants.
\end{theorem}

\paragraph{Part 2: Bounding the centered Jacobian.} If we fix the weight matrices $(W_l)_{l=1}^L$, the rows of $\cJ_{L-2}$ are i.i.d. vectors of the form $u \otimes v - \E \left[ u \otimes v \right]$, where both $u$ and $v$ are centered. By exploiting this structure, in Appendix \ref{app:bdrow} we bound the $\ell_2$ and sub-exponential norms of these rows. The $\ell_2$ norm is bounded in Lemma \ref{lemma:l2normsJacobian}, and this result relies on the tight estimates on the $\ell_2$ norms of centered features and centered backpropagation terms given by Lemma \ref{lemma:concnormcentered} and \ref{lemma:concnormcenteredwithD}, respectively. The sub-exponential norm is bounded  in Lemma \ref{lemma:psi1normacobian}, and we use again the tail bound of Lemma \ref{lemma:HW}, which exploits the version of the Hanson-Wright inequality in \cite{HWconvex}. At this point, the problem is reduced to bounding the smallest eigenvalue of a matrix such that its rows are i.i.d. and we have a good control on their $\ell_2$ and sub-exponential norms. This goal can be achieved via the following result, whose proof follows from Theorem 3.2 in \cite{hammer} and it is presented for completeness in Appendix \ref{app:hammproof}. 
\begin{proposition}\label{cor:hammer}
	Let $\cJ$ be a matrix with rows $\cJ_{i:}\in \mathbb R^{n_{L-2}n_{L-1}}$, for $i\in [N]$. Assume that $\{\cJ_{i:}\}_{i\in [N]}$ are independent sub-exponential random vectors with $\psi = \max_i \|\cJ_{i:}\|_{\psi_1}$. Let $\eta_{\textup{min}} = \min_i \|\cJ_{i:}\|_2$ and $\eta_{\textup{max}} = \max_i \|\cJ_{i:}\|_2$.
	Set $\xi = \psi K + K'$, and
	$\Delta := C \xi^2 N^{1/4} (n_{L-1}n_{L-2})^{3/4}$. Then, we have that
	\begin{equation}\label{eq:hm}
		\evmin{\cJ \cJ^\top }\geq \eta_{\textup{min}}^2 - \Delta
	\end{equation}
	holds with probability at least
	\begin{equation}
		1 -  \exp \left( -cK \sqrt{N} \log \left( \frac{2 n_{L-1}n_{L-2}}{N} \right) \right) - \mathbb{P}\left(\eta_{\text{max}} \geq K' \sqrt{n_{L-1}n_{L-2}}\right),
	\end{equation}
	where $c$ and $C$ are numerical constants.
\end{proposition}

The main result of this part is stated below, and it is proved in Appendix \ref{app:pfce}. 
\begin{theorem}[Bounding the centered Jacobian]\label{thm:centered}
    Consider the setting of Theorem \ref{thm:main}, and let the rows of $\cJ_{L-2}$ be given by  \eqref{eq:centeredJ}. Then, we have  
	\begin{equation}
		\evmin{\cJ_{L-2}\cJ_{L-2}^\top} = \Omega(n_{L-2} n_{L-1}),
	\end{equation}
	with probability at least $1 - C e^{ -c \sqrt{N}} - C\,N e^{-c \log^2 n_{L-1}}$ over $(x_i)_{i=1}^N$ and $(W_k)_{k=1}^L$, where $c$ and $C$ are numerical constants.
\end{theorem}

Recall that $K \succeq K_{L-2}$, hence $\evmin{K}\ge \evmin{K_{L-2}}=\evmin{J_{L-2}J_{L-2}^\top}$. Thus, the lower bound \eqref{eq:lbNTK} follows by combining the results from Theorem \ref{thm:maincentering} and Theorem \ref{thm:centered}. 

Finally, the upper bound \eqref{eq:ubNTK} is rather straightforward. First, we write
\begin{equation}
\evmin{K} = \evmin{JJ^\top} \leq (J J^\top)_{11} = \sum_{l=0}^{L-1} \norm{(F_l)_{1:}}_2^2 \norm{(B_{l+1})_{1:}}_2^2.    
\end{equation}
Then, by combining the bound on $\norm{(F_l)_{1:}}_2^2$ of Lemma \ref{lemma:concnorm} with a direct calculation for $\norm{(B_{l+1})_{1:}}_2^2$, the desired result \eqref{eq:ubNTK} follows. The details are contained in Appendix \ref{app:pfub}.

\section{Two Applications: Memorization and Optimization}\label{sec:appl}

\paragraph{Memorization capacity.} The fact that the NTK Gram matrix \eqref{eq:NTKgramdef} is well conditioned readily implies a result on 
memorization capacity. This was already observed in \cite{Andrea2020} for two-layer networks and in \cite{tightbounds} for deep networks with a layer containing $\Omega(N)$ neurons. Here, by using Theorem \ref{thm:main}, we show that $\Omega(N)$ parameters between the last pair of hidden layers are enough for the network to fit $N$ real-valued labels up to an arbitrarily small error. The result is stated below and, for completeness, we give the proof in Appendix \ref{app:newmem}. 


\begin{corollary}[Memorization]\label{cor:memcap}
	Consider an $L$-layer neural network \eqref{eq:def_feature_map}, where the activation function satisfies Assumption \ref{ass:activationfunc} and the layer widths satisfy Assumptions \ref{ass:topology} and \ref{ass:overparam}.
	Let $\Set{x_i}_{i=1}^{N}\sim_{\rm i.i.d.}P_X$, 
	where $P_X$ satisfies the Assumptions \ref{ass:data_dist}-\ref{ass:data_dist2}. 
	Then, it holds that for every $Y \in \R^N$, and for every $\varepsilon > 0$, there exists a set of parameters $\theta$ such that
	\begin{equation}
		\norm{F_L(\theta) - Y}_2 \leq \varepsilon,
	\end{equation}
	with probability at least
	$1 - C\,N e^{-c \log^2 n_{L-1}} - C e^{-c \log^2 N} 
	$ over $(x_i)_{i=1}^N$, where $c$ and $C$ are numerical constants.
\end{corollary}

\paragraph{Gradient descent training.} Theorem \ref{thm:main} has implications on the convergence of gradient descent algorithms. In particular, by choosing carefully the initialization, we show that gradient descent trained on the $N$ samples $\{x_i\}_{i=1}^N$ with labels $Y=(Y_1, \ldots, Y_N)$ converges to zero loss. This is -- at the best of our knowledge -- the \emph{first result of this kind for deep networks with minimum over-parameterization}.

To define our initialization, let us assume that the $(L-1)$-th layer has an even number of neurons. With a slight abuse of notation, we indicate its width as $2 n_{L-1}$. Thus, $W_{L-1}$ can be written in the form $\left[W_{L-1}^{(1)}, W_{L-1}^{(2)}\right]$, where $W_{L-1}^{(1)}, W_{L-1}^{(2)}\in\mathbb R^{n_{L-2} \times n_{L-1}}$. Similarly, $W_L\in\mathbb R^{2 n_{L-1}}$ is the concatenation of $W_{L}^{(1)}, W_{L}^{(2)}\in\mathbb R^{n_{L-1}}$. Then, we define the initialization $\theta_0 = [\vec(W_{1, 0}), \ldots, \vec(W_{L, 0})]$ as follows:
\begin{equation}\label{eq:theta0}
\begin{split}
    (W_{l, 0})_{i,j}&\distas{}_{\rm i.i.d.}\mathcal{N}(0,\beta_l^2 / n_{l-1}), \,\, l\in[L-2], \,i\in[n_{l-1}], \,j\in[n_l], \\[4pt]
    (W_{L-1, 0}^{(1)})_{i,j}&\distas{}_{\rm i.i.d.}\mathcal{N}(0,\beta_{L-1}^2 / n_{L-2}),\,\, i\in[n_{L-2}], \,j\in[n_{L-1}], \,\, \mbox{and} \,\, W_{L-1, 0}^{(2)} = W_{L-1, 0}^{(1)}, \\[4pt]
(W_{L, 0}^{(1)})_{i}&\distas{}_{\rm i.i.d.}\mathcal{N}(0,\beta_L^2 \gamma),\,\, i\in[n_{L-1}],\,\,\mbox{and} \,\, W_{L, 0}^{(2)} = - W_{L, 0}^{(1)}.
\end{split}
\end{equation}
As usual, $(\beta_l)_{l=1}^L$ are numerical constants independent of the layer widths. In contrast, the quantity $\gamma$ will be set to a suitably large value depending on the number of training samples and on the layer widths. In words, the initialization is standard for $W_{1, 0}, \ldots, W_{L-2, 0}$ and $W_{L-1, 0}^{(1)}$; the variance of $W_{L, 0}^{(1)}$ is boosted up by a factor $\gamma$; and we duplicate the neurons of layer $L-1$ so that $W_{L-1, 0}^{(2)} = W_{L-1, 0}^{(1)}$ and $W_{L, 0}^{(2)} = - W_{L, 0}^{(1)}$. At this point, we are ready to state our optimization result. 



\begin{theorem}[Optimization via gradient descent]\label{thm:optimization}
Consider an $L$-layer neural network \eqref{eq:def_feature_map}, where the activation function satisfies Assumption \ref{ass:activationfunc} and the layer widths satisfy Assumptions \ref{ass:topology} and \ref{ass:overparam}.
	Consider training data $\Set{x_i}_{i=1}^{N}\sim_{\rm i.i.d.}P_X$, 
	where $P_X$ satisfies the Assumptions \ref{ass:data_dist}-\ref{ass:data_dist2}, with labels $Y\in \mathbb R^N$ such that $\norm{Y}_2 = \Theta(\sqrt N)$. We consider solving the least-squares optimization problem
    \begin{equation}\label{eq:loss}
        \min_\theta \mathcal L(\theta) := \frac{1}{2} \min_\theta \norm{F_L(\theta) - Y}_2^2,
    \end{equation}
    by running gradient descent updates of the form 
    $\theta_{t+1} = \theta_t - \eta \nabla \mathcal L(\theta_t)$,
    where the initialization $\theta_0$ is defined in \eqref{eq:theta0} with $\gamma = d^3 N^2$ and $\eta \leq C (\gamma Nd n_{L-1})^{-1}$. Then, for all $t\ge 1$, 
        \begin{equation}\label{eq:linconv}
        \mathcal L(\theta_t) \leq \left( 1 - c \eta \gamma n_{L-2} n_{L-1}\right)^t \mathcal L(\theta_0),
    \end{equation}
    with probability at least
$1 - C\,N e^{-c \log^2 n_{L-1}} - C e^{-c \log^2 N}$
    over $(x_i)_{i=1}^N$ and the initialization $\theta_0$, where $c$ and $C$ are numerical constants.
\end{theorem}

In words, \eqref{eq:linconv} guarantees that gradient descent \emph{linearly} converges to a point with zero loss. We also note that (our upper bound on) the convergence rate can be expressed as a simple function of the learning rate $\eta$, the number of training samples $N$, and the layer widths $\{n_i\}_{i=0}^{L-1}$.

We now provide a sketch of the argument of Theorem \ref{thm:optimization}, deferring the detailed proof to Appendix \ref{app:pfopt}. As a consequence of Theorem \ref{thm:main}, 
the NTK is well conditioned at the initialization $\theta_0$, which implies that gradient descent makes progress at the beginning. The idea is that the NTK remains sufficiently well conditioned also during training and, hence, the loss vanishes with the linear rate in \eqref{eq:linconv}, since the trained weights remain confined in a ball of sufficiently small radius centered at $\theta_0$. More specifically, we exploit the optimization framework developed in \cite{oymak2019overparameterized}. There, a convergence result is proved if the following condition holds: there exists $\alpha, \beta$ such that 
\begin{equation}\label{eq:optass3}
\begin{split}
&    
\alpha \leq \sigma_{\textup{min}}(J(\theta)) \leq \opnorm{J(\theta)} \leq \beta, 
\quad\mbox{ and }\quad\opnorm{J(\theta_1) - J(\theta_2)} \leq \frac{\alpha^2}{2 \beta},
\end{split}
\end{equation}
for all $\theta, \theta_1, \theta_2\in \mathcal D = \mathcal B(\theta_0, R)$, where $\mathcal D$ denotes a ball centered at $\theta_0$ of radius $R:=4\|F_L(\theta_0)-Y\|_2/\alpha$ (cf. Proposition \ref{prop:optim}). In order to control the radius $R$, we take advantage of the form \eqref{eq:theta0} of the initialization $\theta_0$. In particular, the spectrum of the Jacobian $J$ (and, hence,  $\alpha$) scales linearly in $\sqrt{\gamma}$, and the network output at initialization $F_L(\theta_0)$ is equal to $0$ (cf. Lemma \ref{lemma:firstlemmaopt}). Thus, as $\norm{Y}_2 = \Theta(\sqrt N)$, we have that $R=\Theta(\sqrt{N}/\alpha)$, which can be controlled by choosing suitably $\gamma$. At this point, the argument consists in providing a number of estimates on feature vectors, backpropagation terms, and finally the Jacobian $J(\theta)$ for all $\theta\in\mathcal D$ (cf. Lemmas \ref{lemma:optW}-\ref{lemma:opnormsigminball}). This allows us to find $\alpha, \beta$ such that \eqref{eq:optass3} holds and conclude.

In Figure \ref{fig:optimizationbody}, we give an illustrative example that 4-layer networks achieve $0$ loss when the number of parameters is at least linear in the number of training samples, i.e., under minimum over-parameterization. To ease the experimental setup, we use a ReLU activation, with Adam optimizer. We initialize the network as in the setting of Theorem \ref{thm:main}, picking $\beta_l = 1$ for all $l \in [L]$. The inputs, as well as the targets, are sampled from a standard Gaussian distribution. The plot is averaged over $10$ independent trials. As predicted by Theorem \ref{thm:optimization}, the loss experiences a phase transition and it goes from $1$ to $0$ when the layer widths are of order $\sqrt{N}$.

\section{Related Work and Discussion} \label{sec:rel}

\paragraph{Random matrices in deep learning.} The limiting spectra of several random matrices related to neural networks have been the subject of a recent line of work. In particular, the 
Conjugate Kernel (CK) -- namely, the Gram matrix of the features from the last hidden layer -- has been studied for models with two layers \cite{pennington2017nonlinear,liao2018spectrum,louart2018random,peche2019note}, with a bias term \cite{adlam2019random,piccolo2021analysis}, and with multiple layers \cite{benigni2019eigenvalue}. The Hessian matrix has been considered in \cite{pennington2017geometry}, and the closely related Fisher information matrix in \cite{pennington2018spectrum}. Using tools from free probability, the input-output Jacobian (as opposed to the parameter-output Jacobian considered in this work) of deep networks is studied in \cite{pennington2018emergence} and the NTK of a two-layer model in \cite{adlam2020neural}. In \cite{fan2020spectra}, the spectrum of NTK and CK for deep networks is characterized via an iterated Marchenko-Pastur map. The generalization error has also been studied via the spectrum of suitable random matrices, see \cite{hastie2022surprises,mei2022generalization,liao2020random,Andrea2020} and Section 6 of the review \cite{bartlett2021deep}. Most of the existing results focus on the linear-width asymptotic regime,
where the widths of the various layers are linearly proportional. An exception is \cite{wang2021deformed}, which focuses on the ultra-wide two-layer case ($n_1\gg N$). 


\paragraph{Smallest eigenvalue of empirical kernels.} In the line of work mentioned above, the limiting spectrum of the random matrix is characterized in terms of weak convergence, which does not lead to a direct implication on the behavior of its smallest eigenvalue. In fact, lower bounding the smallest eigenvalue often requires understanding the speed of convergence to the limit spectrum. Quantitative bounds for random Fourier features are obtained in \cite{avron2017random}. In the two-layer setting, the smallest NTK eigenvalue is lower bounded in \cite{theoreticalinsghts,Andrea2020,wang2021deformed}, and concentration bounds can also be obtained from \cite{adlam2020neural,SongYang2020,hu2020surprising}. In particular, \cite{Andrea2020} establishes a tight bound for two-layer networks with roughly as many parameters as training samples (hence, under minimum over-parameterization), and for a wide class of activations. However, this result still requires that 
$n_1=\Omega(d)$. 
For deep networks, a convergence rate on the NTK can be obtained from \cite{AroraEtal2019} and potentially from \cite{buchanan2021deep}, however these results require all the layers to be wide. In \cite{tightbounds}, it is proved that a single layer of linear width suffices for the NTK to be well conditioned. Here, Theorem \ref{thm:main} provides the first result for deep networks under minimum over-parameterization, therefore allowing for sub-linear layer widths.

\begin{wrapfigure}{l}{0.4\textwidth}
  \begin{center}
    \includegraphics[width=0.4\textwidth]{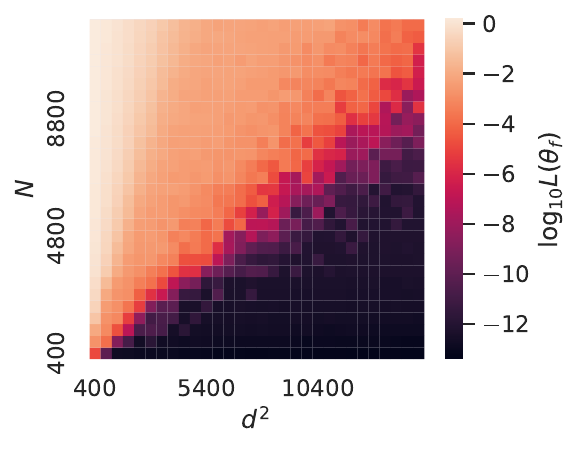}
  \end{center}
  \vspace{-0.5em}
  \caption{Value of the loss $\mathcal L(\theta_f)$ once gradient descent has converged for a 4-layer ReLU network with $d = n_1 = n_2 = n_3$.}
  \vspace{-0.5em}
  \label{fig:optimizationbody}
\end{wrapfigure}

\paragraph{Memorization capacity and gradient descent training.}
For binary classification, 
the memorization capacity of neural networks was first studied by Cover \cite{cover1965geometrical}, who solved the single-neuron case. Later, Baum  \cite{Baum1988} proved that, for two-layer networks, the memorization capacity is lower bounded by the number of parameters, and upper bounds of the same order were given in \cite{kowalczyk1994counting,sakurai1992nh}.
Recently, tight results for two-layer networks in the regression setting have been provided \cite{bubeck2020network, Andrea2020}. 
In particular, \cite{bubeck2020network} generalizes the construction of \cite{Baum1988}, while the memorization result of \cite{Andrea2020} comes from a bound on the smallest NTK eigenvalue. This same route is also followed to analyze deep networks in \cite{tightbounds} and in this work. Results for multiple layers have been provided in \cite{bartlett2019nearly, yun2019small,vershynin2020memory,ge2019mildly}. 
As concerns efficient algorithms achieving memorization, \cite{daniely2020memorizing,daniely2020neural} study classification with two-layer networks. A popular line of research has exploited the NTK view to give optimization guarantees on gradient descent. Existing work has focused on the two-layer case \cite{theoreticalinsghts,DuEtal2018_ICLR,OymakMahdi2019,SongYang2020,wu2019global,song2021subquadratic}, and optimal results in terms of over-parameterization have been obtained \cite{Andrea2020}. For deep networks, \cite{DuEtal2019} assumes a lower bound on the smallest NTK eigenvalue, \cite{AllenZhuEtal2018,zou2020gradient,ZouGu2019} require all the layer widths to be (rather large) polynomials in the number of samples, and \cite{QuynhMarco2020,nguyen2021proof} reduce the over-parameterization to a single layer of linear width. In particular, in these last two papers, the analysis of the NTK is reduced to that of a certain feature matrix (the first one in \cite{QuynhMarco2020} and the last one in \cite{nguyen2021proof}). Hence, these approaches do not seem suitable to tackle networks with sub-linear layer widths\footnote{If the width is sub-linear, then the feature matrix is low-rank and, hence, its smallest eigenvalue is $0$.}. \simone{Finally, we remark that milder over-paramerization requirements are sufficient for classification problems. More specifically, it has been proved that polylogarithmic width suffices for two-layer \cite{ji2019polylogarithmic} and deep \cite{chen2020much} ReLU networks to converge and generalize. In fact, achieving memorization for regression is harder than for classification, where it suffices to satisfy inequality constraints \cite{Andrea2020}.} 





\section{Concluding Remarks}\label{sec:concl}

In this paper, we show that the NTK is well conditioned for deep networks with sub-linear layer widths. As consequences of our NTK bounds, we also prove results on memorization and optimization for such a class of networks. Our approach requires smooth activations and layer widths which do not significantly grow in size with the depth. We believe such assumptions to be purely technical, and we leave as an open question their removal. Our novel NTK analysis could potentially be applied beyond standard gradient descent, e.g., for optimization with momentum \cite{wang2021modular}, or federated learning \cite{huang2021fl}.

\section*{Acknowledgements}

The authors were partially supported by the 2019 Lopez-Loreta prize, and they would like to thank Quynh Nguyen, Mahdi Soltanolkotabi and Adel Javanmard for helpful discussions.

{
\small

\bibliographystyle{plain}
\bibliography{bibliography.bib}

}

\newpage

\appendix

\section{Additional Notations}\label{app:notation}
Given a sub-exponential random variable $X$, let $\|X\|_{\psi_1} = \inf \{ t>0 \,\,: \,\,\mathbb E[\exp(|X|/t)] \le 2 \}$.
Similarly, for a sub-Gaussian random variable, let $\|X\|_{\psi_2} = \inf \{ t>0 \,\,: \,\,\mathbb E[\exp(X^2/t^2)] \le 2 \}$.

We use the analogous definitions for vectors. In particular, let $X \in \mathbb R^n$ be a random vector, then $\subGnorm{X} := \sup_{\norm{u}_2=1} \subGnorm{u^\top X}$ and $\subEnorm{X} := \sup_{\norm{u}_2=1} \subEnorm{u^\top X}$.

We indicate with $C$ and $c$ absolute, strictly positive, numerical constants, that do not depend on the layer widths of the network $\{n_l\}_{l=0}^{L-1}$ or the number of training samples $N$. Their value may change from line to line.

\section{Some Useful Estimates}\label{app:useful}

\begin{lemma}\label{lemma:alphabeta}
    \simone{Under Assumption \ref{ass:overparam}, we have that
    \begin{equation}\label{eq:lemmarel1}
        N \cdot   \log^8 n_{L-1} = o (n_{L-1}n_{L-2}),
    \end{equation}}
    \begin{equation}\label{eq:lemmarel4}
     n_{L-2}\cdot   \log^2 N\cdot  \log^2 n_{L-1} = o (n_{L-1}n_{L-2}),
    \end{equation}
    \begin{equation}\label{eq:lemmarel5}
     n_{L-1}\cdot   \log^2 N\cdot  \log^2 n_{L-1} = o (n_{L-1}n_{L-2}),
    \end{equation}
\simone{    \begin{equation}\label{eq:lemmarel3}
     N\cdot   \log^2 n_{L-1} \cdot  \log^4 N = o (n_{L-1}n_{L-2}).
    \end{equation}}
\end{lemma}
\begin{proof}
\simone{We start by proving \eqref{eq:lemmarel1}. If $n_{L-1}=\bigO{N^2}$, then
\begin{equation}
    N \cdot   \log^8 n_{L-1} = \bigO{N \cdot   \log^8 N}= o(n_{L-1}n_{L-2}),
\end{equation}
where the last passage follows from \eqref{eq:overparamcond}. Conversely, if $n_{L-1}=\Omega(N^2)$, then 
\begin{equation}
    N \cdot   \log^8 n_{L-1} = \bigO{\sqrt{n_{L-1}}\cdot   \log^8 n_{L-1}} = o(n_{L-1}) = o(n_{L-1}n_{L-2}),
\end{equation}
which concludes the proof of \eqref{eq:lemmarel1}.}

\simone{    To obtain \eqref{eq:lemmarel4}, we can exploit the second requirement of Assumption \ref{ass:overparam}, which implies that $\log N = \bigO{\log n_{L-1}}$. This readily implies \eqref{eq:lemmarel4}. Notice that \eqref{eq:lemmarel5} naturally follows since $n_{L-1} = \bigO{n_{L-2}}$ by Assumption \ref{ass:topology}.}

\simone{    Finally, to obtain \eqref{eq:lemmarel3}, we write
    \begin{equation}
        N\cdot   \log^2 n_{L-1} \cdot  \log^4 N = \bigO{N\cdot   \log^6 n_{L-1} }=o (n_{L-1}n_{L-2}),
    \end{equation}
    where in the first passage we use that $\log N = \bigO{\log n_{L-1}}$ (from the second requirement of Assumption \ref{ass:overparam}) and the last passage follows from \eqref{eq:lemmarel1}. }
\end{proof}

\begin{lemma}[Lipschitz constant of function of the features]\label{lemma:lipschitzconst}
	For all $l \in [L-1]$, and for every Lipschitz function $\varphi$, we have
	\begin{equation}
		\norm{\varphi(g_l(x))}_{\Lip} = \bigO{1},
	\end{equation}
	with probability at least
	\begin{equation}\label{eq:pm}
		1 - 2 l \exp{(-n_{L-1})},
	\end{equation}
	over $(W_k)_{k=1}^l$. We recall that $\varphi$ is applied component-wise to $g_l(x)$, and $\varphi(g_l(x)) : \R^d \to \R^{n_l}$ is intended as a function of $x$.
\end{lemma}
\begin{proof}
	Note that $\varphi(g_l)$ is a composition of Lipschitz functions. Thus,
	\begin{equation}
		\begin{aligned}
			\norm{\varphi(g_l)}_{\Lip} & \leq \norm{\varphi}_{\Lip} \norm{g_l}_{\Lip} \leq \norm{\varphi}_{\Lip} \opnorm{W_l}\prod_{k=1}^{l-1} \left(\opnorm{W_k} \norm{\phi}_{\Lip}\right) \\
			& \le \norm{\varphi}_{\Lip}\opnorm{W_l} M^{l-1}  \prod_{k=1}^{l-1} \opnorm{W_k}, 
		\end{aligned}
	\end{equation}
	where the last step is justified by Assumption \ref{ass:activationfunc}. 
	
	Recall that, by the assumption on the initialization of the weights, $(W_k)_{i,j}\distas{}_{\rm i.i.d.}\mathcal{N}(0, \beta^2_k/n_{k-1})$, for some constant $\beta_k$ which does not depend on the layer widths. Then, by Theorem 4.4.5 of \cite{vershynin2018high}, we have that, for any $k\in [l]$,
	\begin{equation}\label{eq:opnmW1}
		\opnorm{W_k} \leq C \frac{\beta_k}{\sqrt{n_{k-1}}} (\sqrt{n_{k-1}} + 2\sqrt{n_{k}}),
	\end{equation}
	with probability at least $1 - 2 \exp{(- n_k)}$, $C$ being a numerical constant. By Assumption \ref{ass:topology} on the topology of the network, we can rewrite this result as
	\begin{equation}\label{eq:opWknm}
		\opnorm{W_k} = \bigO{1},
	\end{equation}
	with probability at least $1 - 2 \exp{(-n_{L-1})}$. 
	To conclude, using a union bound over the layers up to layer $l$, we have that
	\begin{equation}
		\norm{\varphi(g_l)}_{\Lip} = \bigO{1},
	\end{equation}
	with probability at least $1 - 2l \exp{(-n_{L-1})}$ over $(W_k)_{k=1}^l$.
\end{proof}

\begin{lemma}\label{lemma:opnormlog}
We have that
	\begin{equation}
		\opnorm{D_L} \leq \log{n_{L-1}},
	\end{equation}
	with probability at least $1 - 2 \exp (- c \log ^2 n_{L-1})$ over $W_L$, where $c$ is a numerical constant.
\end{lemma}
\begin{proof}
Recall that	$D_L = \text{diag}(W_L)$ contains on the diagonal $n_{L-1}$ independent Gaussian random variables $(D_{L})_{ii} \sim \mathcal N(0, \beta^2_L)$. Thus, for any $i\in [n_{L-1}]$,
	\begin{equation}
		\P(|(D_{L})_{ii}| > \log n_{L-1}) < 2 \exp (-\log ^2 n_{L-1} / (2 \beta^2_L) ),
	\end{equation}
	which gives
	\begin{equation}
		\begin{aligned}
			\P(\opnorm{D_L} > \log n_{L-1}) =& \P(\max_{i\in [n_{L-1}]} |(D_{L})_{ii}| > \log n_{L-1}) \\ \leq& n_{L-1} \P(|(D_{L})_{11}| > \log n_{L-1}) \\
			<& 2 \exp (\log n_{L-1} - \log ^2 n_{L-1} / (2 \beta^2_L)) \\
			< & 2 \exp (- c \log ^2 n_{L-1}),
		\end{aligned}
	\end{equation}
	where the second step is a union bound on the entries of $D_L$. This gives the desired result.
\end{proof}

\begin{lemma}\label{lemma:lipconstbacklastlayer}
We have that 
	\begin{equation}
		\norm{D_L\phi'(g_{L-1}(x))}_{\Lip} = \mathcal O(\log n_{L-1}),
	\end{equation}
	with probability at least $1 - 2 \exp (- c \log ^2 n_{L-1}) - C \exp{(-n_{L-1})}$ \, over $(W_k)_{k=1}^L$. We recall that $\phi'$ is applied component-wise to $g_l(x)$, $D_L\phi'(g_{L-1}(x)) : \R^d \to \R^{n_{L-1}}$ is intended as a function of $x$, and $c$ is a numerical constant.
\end{lemma}	
\begin{proof}
	We know by composition of Lipschitz functions that
	\begin{equation}
		\norm{D_L\phi'(g_{L-1})}_{\Lip} \leq \opnorm{D_L}\norm{\phi'(g_{L-1})}_{\Lip}.
	\end{equation}
By Assumption \ref{ass:activationfunc}, $\phi'$ is Lipschitz. Hence, by combining Lemma \ref{lemma:lipschitzconst} (where we use $\varphi = \phi'$) and Lemma \ref{lemma:opnormlog}, the result follows.
\end{proof}

\begin{lemma}[Exponential tails of quadratic forms]\label{lemma:HW}
	Let $x \sim P_X$. Let $u: \R^d \to \R^{d_u}$ and $v: \R^d \to \R^{d_v}$ be mean-0 Lipschitz functions with respect to $x$, i.e., $\mathbb E_x[u(x)]=0$, $\mathbb E_x[v(x)]=0$, $\norm{u}_{\Lip} = c_1$ and $\norm{v}_{\Lip} = c_2$. 
	Let $U$ be a $d_u \times d_v$ matrix, and
	\begin{equation}
		\Gamma(x) = u(x)^\top U v(x) - \E_x \left[ u(x)^\top U v(x) \right].
	\end{equation}
	Then,
	\begin{equation}
		\norm{\Gamma}_{\psi_1} < C K^2 \norm{U}_F.
	\end{equation}
	where $K = \sqrt{c_1^2 + c_2^2}$, and $C$ is a numerical constant.
\end{lemma}

\begin{proof}
	Consider the function $z(x): \R^d \to \R^{d_u + d_v}$ obtained by concatenating the vectors $u$ and $v$, i.e.,
	\begin{equation}
		z(x) := [u(x), v(x)]^\top.
	\end{equation}
	One can readily verify that $\norm{z}_{\Lip} \le \sqrt{c_1^2 + c_2^2}:=K$. 
Let us set
	\begin{equation}
		M = \frac{1}{2}
		\left(\begin{array}{@{}c|c@{}}
			0 & U \\
			\hline
			U^\top & 0
		\end{array}\right).
	\end{equation}
Then, we have that
	\begin{equation}
		\Gamma = z^\top M z - \E_x \left[ z^\top M z \right].
	\end{equation}
	Since $x$ satisfies Assumption \ref{ass:data_dist2} and $z(x)$ is Lipschitz, in order to obtain tail bounds on $\Gamma$, we can apply the version of the Hanson-Wright inequality given by Theorem 2.3 in \cite{HWconvex}:
	\begin{equation}
		\begin{aligned}
			\P(|\Gamma| > t) =& \P(|z^\top M z - \E_x \left[ z^\top M z \right]| > t) \\
			< & 2 \exp \left( -\frac{1}{C_1} \min \left( \frac{t^2}{K^4 \norm{M}^2_F}, \frac{t}{K^2 \opnorm{M}} \right) \right) \\
			\leq & 2 \exp \left( -\frac{1}{C_1} \min \left( \frac{t^2}{ K^4 \norm{U}^2_F}, \frac{t}{K^2 \opnorm{U}} \right) \right),
		\end{aligned}
	\end{equation}
	where $C_1$ is a numerical constant, and in the last step we use that $\opnorm{M} = \opnorm{U}$ and $\norm{M}_F^2 = \norm{U}_F^2 / 2\le \norm{U}_F^2$.
	Thus, by Lemma 5.5 of \cite{theoreticalinsghts}, we conclude that
	\begin{equation}
		\norm{\Gamma}_{\psi_1} < C_2 K^2 \norm{U}_F,
	\end{equation}
	for some numerical constant $C_2$, which gives the desired result.
\end{proof}

\begin{lemma}\label{lemma:opnormcovariance}
	Let $u\in \R^{d_u}$ and $v \in \R^{d_v}$ be two mean-0 sub-Gaussian vectors such that $\norm{u}_{\psi_2} = c_1$ and $\norm{v}_{\psi_2} = c_2$. Set $A_{uv} = \E \left[ uv^\top \right]$.  Then,
	\begin{equation}
		\opnorm{A_{uv}} \leq C (c_1 + c_2)^2,
	\end{equation}
	where $C$ is a numerical constant.
\end{lemma}
\begin{proof}
	Consider the vector
	\begin{equation}
		z := [u, v]^\top.
	\end{equation}
	Then, $z$ is sub-Gaussian and, by triangle inequality on the vectors $[u, 0]$ and $[0, v]$, we have that $\norm{z}_{\psi_2} \leq c_1 + c_2$. Since $u$ and $v$ are mean-0, then $z$ is also mean-0 and its covariance matrix can be written as $A_{z} := \E \left[ zz^\top \right]$. Furthermore, we can show that
	\begin{equation}\label{eq:opnormsubgaussian}
		\opnorm{A_{z}} \leq C (c_1 + c_2)^2.
	\end{equation}
	In fact, let $w$ be the unitary eigenvector associated to the maximum eigenvalue of $A_z$. Then,
	\begin{equation}\label{eq:int1}
		\opnorm{A_{z}} = w^\top A_z w = \E \left[ w^\top zz^\top w \right]=\E \left[ (w^\top z)^2 \right].
	\end{equation}
	Furthermore, we have that
	\begin{equation}\label{eq:int2}
		\subGnorm{z} := \sup_{w' \text{s.t.} \norm{w'}_2 = 1} \subGnorm{(w')^\top z} \geq \subGnorm{ w^\top z} \geq \frac{1}{C} \sqrt{ \E \left[ (w^\top z)^2 \right]},
	\end{equation}
	where $C$ is a numerical constant, and the last inequality comes from Eq. (2.15) of \cite{vershynin2018high}. By combining \eqref{eq:int1} and \eqref{eq:int2} with $\norm{z}_{\psi_2} \leq c_1 + c_2$, \eqref{eq:opnormsubgaussian} readily follows.  
	
	Finally, we have that
	\begin{equation}
		A_z = 
		\left(\begin{array}{@{}c|c@{}}
			A_u & A_{uv} \\
			\hline
			A_{uv}^\top & A_v
		\end{array}\right),
	\end{equation}
	where $A_{u} := \E \left[ uu^\top \right]$ and $A_{v} := \E \left[ vv^\top \right]$. As $A_{u}$ and $A_{v}$ are PSD, we have that 
	\begin{equation}\label{eq:ubop}
		\opnorm{A_{uv}} \leq \opnorm{A_{z}}.
	\end{equation}	
	Hence, the desired result follows from \eqref{eq:opnormsubgaussian} and \eqref{eq:ubop}. 
\end{proof}

\begin{lemma}\label{lemma:matiidrows}
    Let $A$ be an $N \times n$ matrix whose rows $A_i$ are i.i.d. mean-0 sub-Gaussian random vectors in $\R^n$. Let $K = \subGnorm{A_i}$ the sub-Gaussian norm of each row. Then, we have
    \begin{equation}
        \opnorm{AA^\top} = K^2 \bigO{N + n},
    \end{equation}
    with probability at least $1 - 2 \exp(-c n)$, for some numerical constant $c$.
\end{lemma}
\begin{proof} 
Without loss of generality, we can assume $K=1$ to simplify the proof. Let $\Sigma$ be the second moment matrix of each of the rows of $A$. Then, $\Sigma = \E \left[ A_i A_i^\top \right]$, since the rows are mean-0. Note that, as the rows of $A$ are i.i.d., $\Sigma$ is independent of $i$. Furthermore, Lemma \ref{lemma:opnormcovariance} implies that the covariance matrix $\E \left[ A_i A_i^\top \right]$ has operator norm bounded by a constant, since the sub-Gaussian norm of the rows is 1. Then, by using Remark 5.40 in \cite{vershrandmat}, we have that
\begin{equation}
    \opnorm{\frac{A^\top A}{N} - \Sigma} \leq \max(\delta, \delta^2), \qquad \mbox{ where }\delta = C \sqrt{\frac{n}{N}} + \frac{t}{\sqrt{N}},
\end{equation}
with probability at least $1 - 2 \exp(-c t^2)$, where $c$ and $C$ are numerical constants. Setting $t = \sqrt{n}$ and using a triangular inequality gives that, with probability at least $1 - 2 \exp(-c t^2)$,
\begin{equation}
    \opnorm{AA^\top} = \opnorm{A^\top A} \leq N \opnorm{\Sigma} + \max(C \sqrt{nN} + \sqrt{nN}, (C \sqrt{n} + \sqrt{n})^2) = \bigO{N + n},
\end{equation}
which implies the desired result (after re-scaling by $K$).
\end{proof}

\begin{lemma}\label{lemma:opnormcentfeat}
    Let $\tilde F_l=F_l-\E_x [F_l]\in\mathbb R^{N\times n_l}$ be the centered features matrix at layer $l$. Then, we have
    \begin{equation}
        \opnorm{\tilde F_l \tilde F_l^\top} = \bigO{N + n_l},
    \end{equation}
    with probability at least $1-C\exp{(-c n_{L-1})}$ over $(W_k)_{k=1}^l$ and $(x_i)_{i=1}^N \sim_{\rm i.i.d.} P_X$.
\end{lemma}
\begin{proof}
    From Lemma \ref{lemma:lipschitzconst}, we have that
    \begin{equation}\label{eq:flLipt}
		\norm{f_l(x)}_{\Lip} = \Theta(1),
	\end{equation}
	with probability at least
	\begin{equation}
		1 - C' \exp{(-n_{L-1})},
	\end{equation}
	over $(W_k)_{k=1}^l$. We condition on this event in the rest of the proof.
	
	Since $(x_i)_{i=1}^N \sim_{\rm i.i.d.} P_X$ and $P_X$ satisfies Assumption \ref{ass:data_dist2}, all the rows of $\tilde F_l$ are mean-0 sub-Gaussian vectors, with sub-Gaussian norm bounded by a numerical constant. Here, we fix $(W_k)_{k=1}^l$ s.t. \eqref{eq:flLipt} holds, and the ``mean-0'' and the ``sub-Gaussian norm'' is intended w.r.t. the probability space of $(x_i)_{i=1}^N$.
	
	
	An application of Lemma \ref{lemma:matiidrows} gives that
	\begin{equation}
		\opnorm{\tilde F_l \tilde F_l^\top} = \bigO{N + n_l},
	\end{equation}
	with probability at least $1 - 2 \exp{(-c n_l)}$ over $(x_i)_{i=1}^N \sim_{\rm i.i.d.} P_X$. Taking into account the previous conditioning, we conclude that
	\begin{equation}
	    \opnorm{\tilde F_l \tilde F_l^\top} = \bigO{N + n_l},
	\end{equation}
	with probability at least $1 - (C' + 2)\exp{(-c n_{L-1})}$ over $(W_k)_{k=1}^l$ and $(x_i)_{i=1}^N \sim_{\rm i.i.d.} P_X$.
\end{proof}

\begin{lemma}\label{lemma:opnormcentbackprop}
Let $\tilde B_{L-1}=B_{L-1}-\E_x [B_{L-1}]\in\mathbb R^{N\times n_{L-1}}$ be the centered back-propagation matrix at layer $L-1$. Then, we have
    \begin{equation}
        \opnorm{\tilde B_{L-1} \tilde B_{L-1}^\top} = \bigO{(N + n_{L-1})\log^2 n_{L-1} },
    \end{equation}
    with probability at least $1-C\exp{(-c n_{L-1})}$ over $(W_k)_{k=l}^L$ and $(x_i)_{i=1}^N \sim_{\rm i.i.d.} P_X$.
\end{lemma}
\begin{proof}
From Lemma \ref{lemma:lipconstbacklastlayer}, we have that
    \begin{equation}\label{eq:flLipt2}
		\norm{D_L\phi'(g_{L-1}(x))}_{\Lip} = \mathcal O(\log n_{L-1}),
	\end{equation}
	with probability at least
	\begin{equation}
		1 - 2 \exp (- c \log ^2 n_{L-1}) - C \exp{(-n_{L-1})},
	\end{equation}
	over $(W_k)_{k=1}^{L}$. We condition on this event in the rest of the proof.

	Since $(x_i)_{i=1}^N \sim_{\rm i.i.d.} P_X$ and $P_X$ satisfies Assumption \ref{ass:data_dist2}, all the rows of $\tilde B_{L-1}$ are mean-0 sub-Gaussian vectors, with sub-Gaussian norm $\mathcal O(\log n_{L-1})$. Here, we fix $(W_k)_{k=1}^L$ s.t. \eqref{eq:flLipt2} holds, and the ``mean-0'' and the ``sub-Gaussian norm'' is intended w.r.t. the probability space of $(x_i)_{i=1}^N$.

	An application of Lemma \ref{lemma:matiidrows} gives that
	\begin{equation}
		\opnorm{\tilde B_{L-1} \tilde B_{L-1}^\top} = \bigO{(N + n_{L-1})\log^2 n_{L-1} },
	\end{equation}
	with probability at least $1 - 2 \exp{(-c n_{L-1})}$ over $(x_i)_{i=1}^N \sim_{\rm i.i.d.} P_X$. Taking into account the previous conditioning, we conclude that
	\begin{equation}
	    \opnorm{\tilde B_{L-1} \tilde B_{L-1}^\top} = \bigO{(N + n_{L-1})\log^2 n_{L-1} },
	\end{equation}
	with probability at least $1 - (C' + 2)\exp{(-c n_{L-1})}$ over $(W_k)_{k=1}^L$ and $(x_i)_{i=1}^N \sim_{\rm i.i.d.} P_X$.
\end{proof}

	\begin{lemma}\label{lemma:lipschitz}
		Let $\rho$ be a standard Gaussian random variable, then we have that
		\begin{equation}
			\varphi_1(c) := \E_\rho \left[ \phi(c\rho) \right],
		\end{equation}
		and
		\begin{equation}
			\varphi_2(c) := \E_\rho \left[ \phi^2(c\rho) \right],
		\end{equation}
		are continuous functions in $c$. Furthermore, $\varphi_1(c)$ is Lipschitz in $c$, and
		\begin{equation}
		    \left|\varphi_2(c_1) - \varphi_2(c_2) \right| \leq C_1 \left| c_1 - c_2 \right| + C_2 |c_1^2 - c_2^2|,
		\end{equation}
		where $C_1$ and $C_2$ are numerical constants (independent of $c_1, c_2$).
	\end{lemma}
	\begin{proof}
	Let $p(\rho)=\frac{1}{\sqrt{2\pi}}e^{-\rho^2/2}$. Then, we have
		\begin{equation}
			\begin{aligned}
				\left| \varphi_1(c + \varepsilon) - \varphi_1(c) \right| &\leq \int p(\rho) \left| \phi((c + \varepsilon) \rho) -  \phi(c\rho) \right| d\rho \\
				&\leq \int p(\rho) \left| M \varepsilon \rho \right| d\rho \\
				&= M \varepsilon \E_\rho \left[ |\rho| \right]\\
				&= C \varepsilon,
			\end{aligned}
		\end{equation}
		where in the second line we use that $\phi$ is $M$-Lipschitz by Assumption \ref{ass:activationfunc}.
		Similarly, we have
		\begin{equation}
			\begin{aligned}
				\left| \varphi_2(c + \varepsilon) - \varphi_2(c) \right| &\leq \int p(\rho) \left| \phi^2((c + \varepsilon) \rho) -  \phi^2(c\rho) \right| d\rho \\
				&= \int p(\rho) \left| \phi((c + \varepsilon) \rho) -  \phi(c\rho) \right|\left| \phi((c + \varepsilon) \rho) +  \phi(c\rho) \right| d\rho \\
				&\leq \int p(\rho) \left| M \varepsilon \rho \right| \left(  2\left| \phi(0)\right| + M (|c + \varepsilon|+|\varepsilon|)|\rho|
				\right) d\rho \\
				&= C_1 \varepsilon \E_\rho \left[ |\rho| \right]
				+ C_2 \varepsilon |c| \E_\rho \left[ \rho^2 \right]
				+ C_3 \varepsilon^2\E_\rho \left[ \rho^2 \right]\\
				&= C_4 \varepsilon + C_2 |c| \varepsilon + C_3 \varepsilon^2 \\
				&\leq C_5 \varepsilon + C_6 \left| (c + \varepsilon)^2 - c^2 \right|.
			\end{aligned}
		\end{equation}
	\end{proof}

	\begin{lemma}\label{lemma:constantexpectationphi}
		Let $\rho$ be a standard Gaussian distribution, and $c \neq 0$ be an absolute constant. Then, we have
		\begin{equation}
			|\E_{\rho}\left[ \phi(c \rho)\right]| = \bigO{1}.
		\end{equation}
		It also holds
		\begin{equation}
			|\E_{\rho}\left[ \phi'(c \rho)\right]| = \bigO{1}.
		\end{equation}
	\end{lemma}
	\begin{proof}
		For the first statement, we exploit the fact that $\phi$ is Lipschitz:
		\begin{equation}\label{eq:ubreuse2}
			|\E_{\rho}\left[ \phi(c \rho)\right]| \leq  \E_{\rho}\left[ |\phi(0)| + M|c \rho|\right] =  |\phi(0)| + M |c|\E_{\rho}\left[ |\rho|\right] = C_1.
		\end{equation}
		
		The statement on $\phi'$ is easily derived following the same proof and using that $\norm{\phi'}_{\Lip} \leq M'$.
	\end{proof}

	\begin{lemma}\label{lemma:constantexpectationphi2}
		Let $\rho$ be a standard Gaussian distribution, and $c \neq 0$ be an absolute constant. Then, we have
		\begin{equation}
			\E_{\rho}\left[ \phi^2(c \rho)\right] = \Theta(1),
		\end{equation}
		and
		\begin{equation}
			\E_{\rho}\left[ (\phi'(c \rho))^2\right] = \Theta(1).
		\end{equation}
	\end{lemma}
	\begin{proof}
		For the upper-bound of the first statement, we exploit the fact that $\phi$ is Lipschitz:
		\begin{equation}\label{eq:ubreuse}
			\E_{\rho}\left[ \phi^2(c \rho)\right] \leq  \E_{\rho}\left[ (|\phi(0)| + M|c \rho|)^2\right] =  \phi^2(0) + 2M |c| |\phi(0)|\E_{\rho}\left[ |\rho|\right] + M^2 c^2 \E_{\rho}\left[ \rho^2\right] = C_1.
		\end{equation}
		
		For the lower bound, since $\phi$ is non-zero and continuous, we have that there exist a strictly positive constant $c'>0$ and an interval $[c_1, c_2]$ with $c_2>c_1$ such that $\phi^2(x) \geq c'$ for each $x \in [c_1, c_2]$. Therefore, we have
		\begin{equation}
			\E_{\rho}\left[ \phi^2(c \rho)\right] \geq c' \PP(c_1 \leq c \rho \leq c_2) = C_2.
		\end{equation}
		
		The second statements is proved in the same way, as $\phi'$ is a non-zero Lipschitz function.
	\end{proof}

	\begin{lemma}\label{lemma:squarenonsquare}
		Let $\varphi: \R^d \to \R$ a Lipschitz function, and let $x\sim P_X$. Then,
		\begin{equation}
			\E_x^2 \left[ \varphi(x) \right] \geq \E_x \left[ \varphi(x)^2 \right] - c \norm{\varphi}_{\Lip}^2,
		\end{equation}
		where $c$ is a numerical constant.	
	\end{lemma}
	\begin{proof}
		We have
		\begin{equation}
			\begin{aligned}
				\E_x^2 \left[ \varphi(x) \right] &= \E_x \left[ (\varphi(x))^2 \right] - \E_x \left[ \left( 
				\varphi(x) - \E_x \left[ \varphi(x) \right]
				\right)^2 \right] \\
				&= \E_x \left[ \varphi(x)^2 \right] - \int_0^{+\infty} \PP \left( \left( 
				\varphi(x) - \E_x \left[ \varphi(x) \right]
				\right)^2 > t \right) dt \\
				&= \E_x \left[ \varphi(x)^2 \right] - \int_0^{+\infty} \PP \left( \left| 
				\varphi(x) - \E_x \left[ \varphi(x) \right]
				\right| > \sqrt{t} \right) dt \\
				&\geq \E_x \left[ \varphi(x)^2 \right] - \int_0^{+\infty} 2 \exp\left(-Ct / \norm{\varphi}_\Lip^2 \right) dt \\
				&= \E_x \left[ \varphi(x)^2 \right] - 2 \norm{\varphi}_\Lip^2 /C,
			\end{aligned}
		\end{equation}
		where the inequality is a consequence of Assumption \ref{ass:data_dist2}.
	\end{proof}

\begin{lemma}\label{lemma:cclosetoexpect}
	Let $x \sim P_X$, and define
	$\tilde c_l(x) = \beta_l \norm{f_l(x)} / \sqrt{n_l}$.
	
	Then, we have
	\begin{equation}
		\E_x \left[ \left| \tilde c_l(x) - \E_x \left[ \tilde c_l(x) \right]  \right|\right] \leq C \frac{\norm{f_l}_\Lip}{\sqrt{n_l}},
	\end{equation}
	and
	\begin{equation}
		\E_x \left[ \left( \tilde c_l(x) - \E_x \left[ \tilde c_l(x) \right]  \right)^2 \right] \leq C \frac{\norm{f_l}^2_\Lip}{n_l},
	\end{equation}
	where $C$ is a numerical constant.
\end{lemma}
\begin{proof}
We have that
	\begin{equation}
		\begin{aligned}
			\E_x \left[ \left| \tilde c_l(x) - \E_x \left[ \tilde c_l(x) \right]  \right|\right] &= \int_0^{+\infty} \PP \left( \left| \tilde c_l(x) - \E_x \left[ \tilde c_l(x) \right]  \right| > t \right) dt \\
			&= \int_0^{+\infty} \PP \left( \left| \norm{f_l(x)} - \E_x \left[ \norm{f_l(x)} \right]  \right| > \sqrt{n_l} t / \beta_l \right) dt \\
			&\leq \int_0^{+\infty} 2 \exp \left(-c n_l t^2 / \norm{f_l}_\Lip^2 \right) dt \\
			&= C \frac{\norm{f_l}_\Lip}{\sqrt{n_l}},
		\end{aligned}
	\end{equation}
where $c$ and $C$ are numerical constants, and the third line is justified by Assumption \ref{ass:data_dist2}.

Similarly, we have
\begin{equation}
	\begin{aligned}
		\E_x \left[ \left( \tilde c_l(x) - \E_x \left[ \tilde c_l(x) \right]  \right)^2 \right] &= \int_0^{+\infty} \PP \left( \left( \tilde c_l(x) - \E_x \left[ \tilde c_l(x) \right]  \right)^2 > t \right) dt \\
		&= \int_0^{+\infty} \PP \left( \left| \norm{f_l(x)} - \E_x \left[ \norm{f_l(x)} \right]  \right| > \sqrt{n_l t} / \beta_l \right) dt \\
		&\leq \int_0^{+\infty} 2 \exp \left(-c n_l t / \norm{f_l}_\Lip^2 \right) dt \\
		&= C \frac{\norm{f_l}^2_\Lip}{n_l},
	\end{aligned}
\end{equation}
where, again, $c$ and $C$ are numerical constants and the third line is justified by Assumption \ref{ass:data_dist2}.
\end{proof}

\begin{lemma}\label{lemma:quadlipschphimixed}
	Let $\rho_1$ and $\rho_2$ be two standard Gaussian random variables, possibly correlated. Then, we have
	\begin{equation}
		\begin{aligned}
			& \left| \E_{\rho_1 \rho_2} \left[ \phi(\rho_1 x_1) \phi(\rho_2 x_2) - \phi(\rho_1 y_1) \phi(\rho_2 y_2) \right] \right| \leq \\[3pt]
			& \hspace{1cm} \leq C_1 \left| x_1 - y_1 \right| + C_2 \left|x_2 \right| \left| x_1 - y_1 \right| + C_3 \left| x_2 - y_2 \right| + C_4 \left|y_1 \right| \left| x_2 - y_2 \right|,
		\end{aligned}
	\end{equation}
	where $C_1, C_2, C_3, C_4$ are numerical constants (which do not depend on $x_1, x_2, y_1, y_2$). Furthermore, the same result holds with $\phi'$ instead of $\phi$.
\end{lemma}
\begin{proof}
We have
\begin{equation}
	\begin{aligned}
		& \left| \E_{\rho_1 \rho_2} \left[ \phi(\rho_1 x_1) \phi(\rho_2 x_2) - \phi(\rho_1 y_1) \phi(\rho_2 y_2) \right] \right| \\[3pt]
		& \leq \left| \E_{\rho_1 \rho_2} \left[ \phi(\rho_1 x_1) \phi(\rho_2 x_2) - \phi(\rho_1 y_1) \phi(\rho_2 x_2) \right] + \E_{\rho_1 \rho_2} \left[ \phi(\rho_1 y_1) \phi(\rho_2 x_2) - \phi(\rho_1 y_1) \phi(\rho_2 y_2) \right] \right| \\[3pt]
		& \leq \E_{\rho_1 \rho_2} \left[ \left|  \phi(\rho_1 x_1) - \phi(\rho_1 y_1) \right| \left| \phi(\rho_2 x_2) \right| \right] + \E_{\rho_1 \rho_2} \left[ \left| \phi(\rho_2 x_2) - \phi(\rho_2 y_2) \right| \left| \phi(\rho_1 y_1) \right| \right] \\[3pt]
		& \leq \E_{\rho_1 \rho_2} \left[ \left| M \rho_1 (x_1 - y_1) \right| \left( \left| \phi(0) \right| + M\left| \rho_2 x_2 \right| \right) \right] + \E_{\rho_1 \rho_2} \left[ \left| M \rho_2 (x_2 - y_2) \right| \left( \left| \phi(0) \right| + M\left| \rho_1 y_1 \right| \right) \right] \\[3pt]
		& \leq C_1 \left| x_1 - y_1 \right| \E \left[ |\rho_1 | \right] + C_2 \left|x_2 \right| \left| x_1 - y_1 \right| \E \left[ |\rho_1 ||\rho_2| \right]  \\[3pt]
		& \hspace{1cm} + C_3 \left| x_2 - y_2 \right| \E \left[ |\rho_2 | \right] + C_4 \left|y_1 \right| \left| x_2 - y_2 \right| \E \left[ |\rho_1 ||\rho_2| \right] \\[3pt]
		& \leq C_1 \left| x_1 - y_1 \right| + C_2 \left|x_2 \right| \left| x_1 - y_1 \right| + C_3 \left| x_2 - y_2 \right| + C_4 \left|y_1 \right| \left| x_2 - y_2 \right|,
	\end{aligned}
\end{equation}
where in third inequality we use that $\phi$ is $M$-Lipschitz, and in the last inequality we use that the quantities $\E \left[ |\rho_1 | \right]$, $\E \left[ |\rho_2 | \right]$ and $\E \left[ |\rho_1 ||\rho_2| \right]$ are all smaller than $1$ (regardless of the correlation between $\rho_1$ and $\rho_2$). Since we only used the fact that $\phi$ is $M$-Lipschitz, the same result holds with $\phi'$ in place of $\phi$.

\end{proof}

\section{Concentration of $\ell_2$ Norms}\label{app:concentration}
	
	In this appendix, we state and prove a number of high-probability estimates on the $\ell_2$ norms of feature and backpropagation vectors. More specifically, our results can be summarized as follows:
	\begin{itemize}
	    \item Lemma \ref{lemma:concnorm} gives tight bounds on $\norm{f_l(x)}_2$, \ie the $\ell_2$ norm of the feature vector at layer $l$. The statement holds with high probability over $x$ and $(W_k)_{k=1}^l$.
	    
	    \item Lemmas \ref{lemma:concexpnorm2} and \ref{lemma:concexpnorm} give tight bounds on $\E_x \left[ \norm{f_{l}(x)}_2^2 \right]$ and $\E_x \left[ \norm{f_{l}(x)}_2 \right]$, respectively. These quantities represent the \emph{expectation} with respect to $x$ of the (squared) $\ell_2$ norm of the feature vector at layer $l$. The statements hold with high probability over $(W_k)_{k=1}^l$. 
	    
	    
	    
	    \item Lemma \ref{lemma:concnormcentered} focuses on the \emph{centered} feature vector $f_{l}(x) - \E_x \left[f_{l}(x)\right]$, and it gives tight bounds on \emph{(i)} its expected (w.r.t. $x$) squared $\ell_2$ norm $\E_x \left[ \norm{f_{l}(x) - \E_x \left[f_{l}(x)\right] }_2^2\right]$, \emph{(ii)} its expected (w.r.t. $x$) $\ell_2$ norm $\E_x \left[ \norm{f_{l}(x) - \E_x \left[f_{l}(x)\right] }_2\right]$, and \emph{(iii)} its $\ell_2$ norm $\norm{f_{l}(x) - \E_x \left[f_{l}(x)\right] }_2$. The first two statements hold with high probability over $(W_k)_{k=1}^l$, and the probability in the last statement is also over $x$. 
	    
	    \item Lemma \ref{lemma:concnormcenteredwithD} focuses on the \emph{centered} backpropagation vector at layer $L-1$, and it gives tight bounds on its $\ell_2$ norm $\norm{D_L\phi'(g_{L-1}(x)) - \E_x \left[D_L\phi'(g_{L-1}(x)) \right] }_2$. This statement holds with high probability over $x$ and $(W_k)_{k=1}^l$.
	\end{itemize}
	Throughout this appendix, we always assume that $P_X$ satisfies Assumptions \ref{ass:data_dist} and \ref{ass:data_dist2}, and that the layer widths satisfy Assumption \ref{ass:topology}. Furthermore, we use that the activation $\phi$ and its derivative $\phi'$ are Lipschitz (see Assumption \ref{ass:activationfunc}). 

\begin{lemma}[$\ell_2$ norm of features]\label{lemma:concnorm}
		Let $x\sim P_X$.
		Then, for every $0 \leq l \leq L-1$,
		\begin{equation}
			\norm{f_l(x))}_2 = \Theta(\sqrt{n_{l}}),
		\end{equation}
		with probability at least $1 - C \exp(-c n_{L-1})$ over $x$ and $(W_k)_{k=1}^l$. As usual, $\phi$ is applied component-wise on $g_l(x)$, and $c$ and $C$ are numerical constants.
	\end{lemma}
	\begin{proof}
		We prove this by induction over $l$, and we start with the base case ($l=0$). Recall that we have defined $f_0(x) := x$. As the $\ell_2$ norm is a 1-Lipschitz function, by Assumption \ref{ass:data_dist2}, we have that
		\begin{equation}\label{eq:ncon}
			\PP\left( \abs{\norm{x}_2- \E [\norm{x}_2] } > t \right) \leq 2e^{-ct^2}.
		\end{equation}
		Furthermore, Assumption \ref{ass:data_dist} implies that $\E [\norm{x}_2] = \Theta(\sqrt d)$, hence setting $t = \E [\norm{x}_2] / 2$ in \eqref{eq:ncon} proves the desired result for the base case (recalling that $n_{L-1} = \bigO{d}$ by Assumption \ref{ass:topology}).
		
		By inductive hypothesis, we have
		\begin{equation}\label{eq:1indhp}
			\norm{f_{l-1}(x))}_2 = \Theta(\sqrt{ n_{l-1}}),
		\end{equation}
		with probability at least $1 - C \exp(- c n_{L-1})$.
		
		Define $\tilde{c} := \beta_l \norm{f_{l-1}(x)}_2 / \sqrt{n_{l-1}}$. From now on, we condition on a realization of $x$ and $(W_k)_{k=1}^{l-1}$ such that $\tilde{c}=\Theta(1)$. By \eqref{eq:1indhp}, this happens with probability at least $1 - C \exp(- c n_{L-1})$.

		To ease the notation, we use the shorthands $f := f_{l-1}(x)$ and $W := W_l$. Then, we can write
		\begin{equation}\label{eq:1concnormendwith2terms}
			\begin{aligned}
				\norm{f_l(x)}_2 &=			\norm{\phi (W^\top f)}_2 = \sqrt{ n_l} \sqrt{\frac{1}{n_l} \sum_{i=1}^{n_{l}} \phi^2((W^\top)_{i:}  f)} .
			\end{aligned}
		\end{equation}
		Recall that $(W_l)_{i,j}\distas{}_{\rm i.i.d.}\mathcal{N}(0, \beta^2_l/n_{l-1})$ and that the Gaussian distribution is rotationally invariant. Thus, the RHS of \eqref{eq:1concnormendwith2terms} has the same distribution as 
		\begin{equation}\label{eq:1concnormendwith2terms20}
			\begin{aligned}
				\sqrt{ n_l} \sqrt{\frac{1}{n_l} \sum_{i=1}^{n_{l}} \phi^2\left(\tilde{c} \rho_i \right)} = \sqrt{ n_l} \sqrt{ \E_{\rho_1} \left[ \phi^2\left(\tilde{c} \rho_1 \right)\right] + \frac{1}{n_l} \sum_{i=1}^{n_l} Z_i},
			\end{aligned}
		\end{equation}
		where $(\rho_i)_{i=1}^{n_l}\sim_{\rm i.i.d.}\mathcal N(0, 1)$ and also independent of $f$, and 
		we have defined the independent, mean-0 random variables
		\begin{equation}\label{eq:1Zinorm}
			Z_i = \phi^2\left(\tilde{c} \rho_i \right) - \E_{\rho_1} \left[ \phi^2\left(\tilde{c} \rho_1 \right)\right].
		\end{equation}
		Note that, in the definition of $Z_i$, the randomness comes only from $\rho_i$, since we are conditioning on $\tilde{c}$.

		We have that
		\begin{equation}\label{eq:1subGn}
			\subGnorm{ \phi \left( \tilde{c}\rho_i \right) } 
			\leq \subGnorm{\phi \left( \tilde{c}\rho_i \right)  - \E_{\rho_i} \left[\phi \left( \tilde{c}\rho_i \right) \right]} 
			+ \subGnorm{\E_{\rho_i} \left[ \phi \left( \tilde{c}\rho_i \right) \right]} \leq C_1 + C_2 =  C_3,
		\end{equation}
		where the first term is bounded by a constant by Theorem 5.2.2 in \cite{vershynin2018high}, and the bound on the second term follows from Lemma \ref{lemma:constantexpectationphi}.
		As a consequence, we have
		\begin{equation}\label{eq:1fromsubGtosubE}
			\begin{aligned}
				\subEnorm{Z_i} = & \subEnorm{\phi^2\left(\tilde{c} \rho_i \right) - \E_{\rho_i} \left[ \phi^2\left(\tilde{c} \rho_i \right)\right]} \\
				\leq & C_4 \subEnorm{\phi^2\left(\tilde{c} \rho_i \right)} \\
				= & C_4 \subGnorm{\phi\left(\tilde{c} \rho_i \right)}^2 \\
				\le & C_5,
			\end{aligned}
		\end{equation}
		where the inequality in the second line follows from Exercise 2.7.10 of \cite{vershynin2018high}, the equality in the third line follows from Lemma 2.7.6 of \cite{vershynin2018high}, and the inequality in the last line follows from \eqref{eq:1subGn}. Hence, the $Z_i$-s are i.i.d. sub-exponential random variables, with sub-exponential norm bounded by a numerical constant. An application of Bernstein inequality (cf. Corollary 2.8.3. in \cite{vershynin2018high}) gives that 
		\begin{equation}\label{eq:1concnormbigpart0}
			\P \left(\left |\frac{1}{n_l} \sum_{i=1}^{n_l} Z_i\right | > t\right) \leq 2 \exp \left( -c \min \left(\frac{t^2}{C_6^2},\frac{t}{C_6}\right) n_l \right),
		\end{equation}
		where $c, C_6$ are numerical constants. Furthermore, by Lemma \ref{lemma:constantexpectationphi2}, we have
		\begin{equation}\label{eq:1concnormbigpart}
			\E_{\rho_1} \left[ \phi^2\left(\tilde{c} \rho_1 \right)\right] = \Theta(1).
		\end{equation}
		By setting $t = \E_{\rho_1} \left[ \phi^2\left(\tilde{c} \rho_1 \right)\right] / 2$ into \eqref{eq:1concnormbigpart0} and using \eqref{eq:1concnormendwith2terms20} and \eqref{eq:1concnormbigpart}, we conclude that
		\begin{equation}
			\norm{\phi(W^\top f)}_2 = \Theta(\sqrt{n_l}),
		\end{equation}
		with probability at least $1 - C\exp(- c n_{L-1}) -  2 \exp (-c n_l) \geq 1 - C_1\exp(- c n_{L-1})$, for some numerical constant $c$ and $C_1$, which concludes the proof.
	\end{proof}

\begin{lemma}[Expected squared $\ell_2$ norm of features]\label{lemma:concexpnorm2}
	Let $x\sim P_X$.
	Then, for every $0 \leq l \leq L-1$,
	\begin{equation}\label{eq:statproof}
		\E_x \left[ \norm{f_{l}(x)}_2^2 \right]  = \Theta(n_{l}),
	\end{equation}
	with probability at least $1 - C\exp(- c n_{L-1})$ over $(W_k)_{k=1}^l$. As usual, $c$ and $C$ are numerical constants.
\end{lemma}
\begin{proof}

    The argument is by induction over $l$. The base case is a direct consequence of Assumption \ref{ass:data_dist}, since $f_0(x) = x$. 
	
	By inductive hypothesis, we have
	\begin{equation}\label{eq:2indhp}
		\E_x\left[ \norm{f_{l-1}(x))}^2_2 \right] = \Theta( n_{l-1}),
	\end{equation}
	with probability at least $1 - C\exp(- c n_{L-1})$. Define $\tilde{c}(x) := \beta_l \norm{f_{l-1}(x)}_2 / \sqrt{n_{l-1}}$. From now on, we condition on a realization of $(W_k)_{k=1}^{l-1}$ such that $\E_x \left[ \tilde{c}^2(x) \right] =\Theta(1)$. By \eqref{eq:2indhp}, this happens with probability at least $1 - C\exp(- c n_{L-1})$.

	To ease the notation, we use the shorthands $f := f_{l-1}(x)$, $W := W_l$ and $w_i = W_{:i}$. Then, we can write
	\begin{equation}\label{eq:2concnormendwith2terms}
		\begin{aligned}
			\E_x \left[ \norm{f_l(x))}^2_2 \right] &= \E_x \left[ \norm{\phi (W^\top f)}^2_2 \right] \\
			&= n_l \left( \frac{1}{n_l} \sum_{i=1}^{n_{l}} \E_x \left[ \phi^2((W^\top)_{i:}  f) \right] \right) \\
			&=  n_l \left(  \E_{w_1} \E_x \left[ \phi^2(w_1^\top f) \right]  + \frac{1}{n_l} \sum_{i=1}^{n_{l}} Z_i \right),
		\end{aligned}
	\end{equation}
	where we use that the $w_i$-s are equally distributed and we have defined the independent, mean-0 random variables
	\begin{equation}\label{eq:2Zinorm}
		Z_i = \E_x \left[ \phi^2\left(w_i^\top f(x) \right) \right] - \E_{w_i} \E_x \left[ \phi^2\left(w_i^\top f(x) \right)\right].
	\end{equation}
	Note that, in the definition of $Z_i$, the randomness comes only from $w_i$, since we are conditioning on $(W_k)_{k=1}^{l-1}$.

	We have that
	\begin{equation}\label{eq:2fromsubGtosubE}
		\begin{aligned}
			\subEnorm{Z_i} \leq & \E_x \left[ \subEnorm{\phi^2\left(w_i^\top f(x) \right) - \E_{w_i} \left[ \phi^2\left(w_i^\top f(x) \right)\right]} \right] \\
			\leq & \E_x \left[ C_1 \subEnorm{\phi^2\left(w_i^\top f(x) \right)} \right] \\
			= & C_1 \E_x \left[\subGnorm{\phi\left(w_i^\top f(x) \right)}^2 \right],
		\end{aligned}
	\end{equation}
	where the first line follows from Jensen's inequality as $\subEnorm{\cdot}$ is convex, the inequality in the second line follows from Exercise 2.7.10 of \cite{vershynin2018high}, and the equality in the third line follows from Lemma 2.7.6 of \cite{vershynin2018high}.
	
	Recall that $(W_l)_{i,j}\distas{}_{\rm i.i.d.}\mathcal{N}(0, \beta^2_l/n_{l-1})$ and that the Gaussian distribution is rotationally invariant. Thus, $\phi\left(w_i^\top f(x) \right)$ has the same distribution as $\phi\left(\tilde c(x) \rho_i \right)$, where $(\rho_i)_{i=1}^{n_l}\sim_{\rm i.i.d.}\mathcal N(0, 1)$ and also independent of $\tilde c(x)$.
	We now condition on a realization of $x$ and $(W_k)_{k=1}^{l-1}$ and provide an upper bound on the sub-Gaussian norm $\subGnorm{\phi\left(w_i^\top f(x) \right)}$, where the only randomness comes again from $w_i$ (and, hence, from $\rho_i$). We have that
	\begin{equation}\label{eq:2subGn}
		\begin{aligned}
			\subGnorm{\phi\left(w_i^\top f(x) \right)} &= \subGnorm{ \phi \left( \tilde{c}(x)\rho_i \right) } \\
			&\leq \subGnorm{\phi \left( \tilde{c}(x)\rho_i \right)  - \E_{\rho_i} \left[\phi \left( \tilde{c}(x)\rho_i \right) \right]} 
			+ \subGnorm{\E_{\rho_i} \left[ \phi \left( \tilde{c}(x)\rho_i \right) \right]} \\
			&\leq C_1\tilde{c}(x) + C_2\tilde{c}(x) +C_3=  C_4(\tilde{c}(x)+1).
		\end{aligned}
	\end{equation}
	where the first term in the RHS in the second line is bounded by $C_1\tilde{c}(x)$ by Theorem 5.2.2 in \cite{vershynin2018high}, and the second term is bounded by $C_2\tilde{c}(x)+C_3$ by following the same proof of Lemma \ref{lemma:constantexpectationphi} as $\phi$ is Lipschitz.	By combining \eqref{eq:2fromsubGtosubE} and \eqref{eq:2subGn}, we get
	\begin{equation}
		\subEnorm{Z_i} \leq C_4^2 \E_x \left[ (\tilde c(x)+1)^2 \right] \leq C_5,
	\end{equation}
	where we use that $\E_x \left[ \tilde{c}^2(x) \right] =\Theta(1)$. 
	
	Hence, the $Z_i$-s are i.i.d. sub-exponential random variables, with sub-exponential norm bounded by a numerical constant. An application of Bernstein inequality (cf. Corollary 2.8.3. in \cite{vershynin2018high}) gives that 
	\begin{equation}\label{eq:2concnormbigpart0}
		\P \left(\left |\frac{1}{n_l} \sum_{i=1}^{n_l} Z_i\right | > t\right) \leq 2 \exp \left( -c \min \left(\frac{t^2}{C_5^2},\frac{t}{C_5}\right) n_l \right),
	\end{equation}
	where $c, C_5$ are numerical constants.
	
	Let us consider the first term in \eqref{eq:2concnormendwith2terms}:
	\begin{equation}
		\E_x \E_{w_1} \left[ \phi^2(w_1^\top f) \right] = \E_x  \E_{\rho_1} \left[ \phi^2\left(\tilde{c}(x) \rho_1 \right)\right],
	\end{equation}
	where the equality comes again from the rotational invariance of the Gaussian distribution of $w_1$. We will show that
	\begin{equation}\label{eq:lbub}
	    \E_x  \E_{\rho_1} \left[ \phi^2\left(\tilde{c}(x) \rho_1 \right)\right]=\Theta(1),
	\end{equation}
	with probability at least $1-C\exp(-cn_{L-1})$ over $(W_k)_{k=1}^{l-1}$.
	
	The upper bound in \eqref{eq:lbub} follows from the same passages in \eqref{eq:ubreuse}, as $\mathbb E_x[\tilde{c}^2(x)]=\Theta(1)$. We now prove the lower bound.
	By Lemma \ref{lemma:concnorm}, we have that there exist numerical constants $c_2>c_1>0$ such that $\tilde c(x) \in [c_1, c_2]$ with probability at least $1-C\exp(-cn_{L-1})$ over $x$ and $(W_k)_{k=1}^{l-1}$. Hence, with probability at least $1-2C\exp(-cn_{L-1})$ over $(W_k)_{k=1}^{l-1}$, we have that
	\begin{equation}\label{eq:pnew}
	    \mathbb P_x(\tilde c(x) \in [c_1, c_2])\ge 1/2,
	\end{equation}
	where we use the symbol $\mathbb P_x$ to highlight that this last probability is taken over $x$. Let us condition on a realization of $(W_k)_{k=1}^{l-1}$ s.t. \eqref{eq:pnew} holds. Then, we have
	\begin{equation}
		  \E_x  \E_{\rho_1} \left[ \phi^2\left(\tilde{c}(x) \rho_1 \right)\right] \geq \frac{1}{2} \inf_{c \in [c_1, c_2]} \E_{\rho_1} \left[ \phi^2\left(c \rho_1 \right)\right].
	\end{equation}
	By Lemma \ref{lemma:lipschitz}, we have that $\varphi(c) = \E_{\rho_1} \left[ \phi^2\left(c \rho_1 \right)\right]$ is continuous in $c$. Therefore, by Weierstrass theorem, there exists a strictly positive $c^* \in [c_1, c_2]$ such that $\inf_{c \in [c_1, c_2]} \E_{\rho_1} \left[ \phi^2\left(c \rho_1 \right)\right] = \E_{\rho_1} \left[ \phi^2\left(c^* \rho_1 \right)\right]$. Thus,
	\begin{equation}\label{eq:2concnormbigpart}
		\E_x  \E_{\rho_1} \left[ \phi^2\left(\tilde{c}(x) \rho_1 \right)\right] \geq \frac{1}{2} \E_{\rho_1} \left[ \phi^2\left(c^* \rho_1 \right)\right] = \Theta(1),
	\end{equation}
	where the last equality is a consequence of Lemma \ref{lemma:constantexpectationphi2}. This concludes the proof of the lower bound in \eqref{eq:lbub}.
	
	By setting $t = \E_{\rho_1} \left[ \phi^2\left(c^* \rho_1 \right)\right] / 4$ into \eqref{eq:2concnormbigpart0} and using \eqref{eq:lbub} and \eqref{eq:2concnormendwith2terms}, we conclude that
	\begin{equation}
			\E_x \left[ \norm{f_l(x))}^2_2 \right] = \Theta(n_l),
	\end{equation}
	with probability at least $1 - C\exp(- c n_{L-1})$, for some numerical constants $C$ and $c$, which concludes the proof.
\end{proof}

\begin{lemma}[Expected $\ell_2$ norm of features]\label{lemma:concexpnorm}
	Let $x\sim P_X$.
	Then, for every $0 \leq l \leq L-1$,
	\begin{equation}
		\E_x \left[ \norm{f_{l}(x)}_2 \right]  = \Theta(\sqrt{n_{l}}),
	\end{equation}
	with probability at least $1 -C \exp(- c n_{L-1})$ over $(W_k)_{k=1}^l$. As usual, $c$ and $C$ are numerical constants.
\end{lemma}
\begin{proof}
	We condition on the events
	\begin{equation}
		\E_x \left[ \norm{f_{l}(x)}_2^2 \right]  = \Theta(n_{l}),
	\end{equation}
	and
	\begin{equation}
		\norm{\norm{f_l(x)}_2}_\Lip = \bigO{1},
	\end{equation}
	which happen with probability at least $1 - C\exp(- c n_{L-1})$ over $(W_k)_{k=1}^l$ by Lemma \ref{lemma:concexpnorm2} and \ref{lemma:lipschitzconst}.
	
	The upper bound is a direct consequence of Jensen's inequality:
	\begin{equation}
		\E_x \left[ \norm{f_{l}(x)}_2 \right] \leq \sqrt{\E_x \left[ \norm{f_{l}(x)}_2^2 \right]} = \Theta(\sqrt{n_l}).
	\end{equation}

	For the lower bound, we use Lemma \ref{lemma:squarenonsquare}, and we obtain
	\begin{equation}
		\E_x \left[ \norm{f_{l}(x)}_2 \right] \geq \sqrt{\E_x \left[ \norm{f_{l}(x)}_2^2 \right] - c \norm{\norm{f_l(x)}_2}_\Lip^2} = \Theta(\sqrt{n_l}).
	\end{equation}
\end{proof}

\begin{lemma}[$\ell_2$ norms of centered features]\label{lemma:concnormcentered}
	Let $x\sim P_X$. Then, for every $0 \leq l \leq L-1$, the following results hold.
	\begin{enumerate}
	    \item With probability at least $1 - C \exp(-c n_{L-1})$ over $(W_k)_{k=1}^{l}$, we have that 
	    \begin{equation}\label{eq:res1}
		\E_x \left[ \norm{f_{l}(x) - \E_x \left[f_{l}(x)\right] }_2^2 \right]  = \Theta(n_{l}).
	\end{equation}
		
		\item With probability at least $1 - C \exp(-c n_{L-1})$ over $(W_k)_{k=1}^{l}$, we have that
			\begin{equation}\label{eq:res2}
		\E_x \left[ \norm{f_{l}(x) - \E_x \left[f_{l}(x)\right] }_2 \right]  = \Theta(\sqrt{n_{l}}).
	\end{equation}
	
	\item With probability at least $1 - C \exp(-c n_{L-1})$ over $(W_k)_{k=1}^{l}$ and $x$, we have that
	\begin{equation}\label{eq:res3}
		\norm{f_{l}(x) - \E_x \left[f_{l}(x)\right] }_2  = \Theta(\sqrt{n_{l}}).
	\end{equation}	
	\end{enumerate}
	\end{lemma}
\begin{proof}
	The argument is by induction over $l$. The base case for \eqref{eq:res1} follows directly from Assumption \ref{ass:data_dist}, since $f_0(x) = x$. Since the $\ell_2$ norm is a 1-Lipschitz function, from Jensen inequality and Lemma \ref{lemma:squarenonsquare} we readily obtain the base case for \eqref{eq:res2}. Note that $\norm{\norm{x - \E_x[x]}_2}_\Lip \le 1$. Then, the base case for \eqref{eq:res3} is a direct consequence of Assumption \ref{ass:data_dist2} on $x$ and of the base case of \eqref{eq:res2}, and it holds with probability at least $1 - C' \exp(-cd) \geq 1 - C' \exp(-c n_{L-1})$ over $x$.
	
	By inductive hypothesis, we assume the three statements to be true for layer $l-1$, for $l \in [L-1]$. We will now prove \eqref{eq:res1} for layer $l$.
	
	To ease the notation, we use the shorthands $f(x) := f_{l-1}(x)$, $f = \E_x[f(x)]$, $\tilde f(x) = f(x) - f$, $W := W_l$ and $w_i = W_{:i}$. We also define $\tilde c(x) = \beta_{l} \norm{f( x)}/\sqrt{n_{l-1}}$ and $\tilde c = \E_x [\tilde c(x)]$.
		
	We condition on the following events in the probability space of $(W_k)_{k=1}^{l-1}$:
	\begin{enumerate}
		\item[(a)] $\norm{f(x)}_\Lip = \bigO{1}$, which happens with probability at least $1 - C' \exp(-c n_{L-1})$ by Lemma \ref{lemma:lipschitzconst}.
		\item[(b)] $\tilde c = \Theta(1)$, which happens with probability at least $1 - C' \exp(-c n_{L-1})$ by Lemma \ref{lemma:concexpnorm}. Notice that by Jensen inequality this also implies $\norm{f}_2^2 = \bigO{n_{l-1}}$.
		\item[(c)] By inductive hypothesis, we have that, with probability at least $1 - C' \exp(-c n_{n-1})$ over $x$ and $(W_k)_{k=1}^{l-1}$, 
		\begin{equation}
			\norm{\tilde f(x)}_2^2 = \Theta(n_{l-1}).
		\end{equation}
		Hence, with probability at least $1 - 2C' \exp(-c n_{n-1})$ over $(W_k)_{k=1}^{l-1}$, we have that
		\begin{equation}\label{eq:lastconditioningonW}
			\mathbb P_x \left( c_1 n_{l-1}\le \norm{\tilde f(x)}_2^2 \le c_2 n_{l-1} \right) \geq 1/2,
		\end{equation}
		for some numerical constants $c_2>c_1>0$. In \eqref{eq:lastconditioningonW}, we use the symbol $\mathbb P_x$ to highlight that this last probability is taken over $x$. For the rest of the argument, we condition on a realization of $(W_k)_{k=1}^{l-1}$ s.t. \eqref{eq:lastconditioningonW} holds.
	\end{enumerate}
By taking a union bound, the events (a)-(c) happen with probability at least $1 - 4C' \exp(-c n_{n-1})$ over $(W_k)_{k=1}^{l-1}$.

Now, we can write
	\begin{equation}\label{eq:expansionofsquarednorm}
		\begin{aligned}
			\E_x \left[ \norm{f_l(x) - \E_x \left[f_{l}(x)\right] }_2^2 \right]  &= \E_x \left[ \norm{\phi(W^\top f(x)) - \E_x \left[\phi(W^\top f(x))\right] }_2^2 \right] \\
			&= n_l \left( \frac{1}{n_l} \sum_{i=1}^{n_{l}} \E_x \left[ \left( \phi((W^\top)_{i:} f(x)) - \E_x \left[\phi((W^\top)_{i:}f(x))\right] \right)^2 \right] \right) \\
			&= n_l \left( \E_{xw_1} \left[ \left( \phi(w_1^\top f(x)) - \E_x \left[\phi(w_1^\top f(x))\right] \right)^2 \right] + \frac{1}{n_l} \sum_{i=1}^{n_{l}} Z_i \right),
		\end{aligned}
	\end{equation}
	where we use that the $w_i$-s are identically distributed and we have defined the independent, mean-0 random variables
	\begin{equation}\label{eq:3Zinormxx}
		Z_i = \E_x \left[ \left( \phi(w_i^\top f(x)) - \E_x \left[\phi(w_i^\top f(x))\right] \right)^2 \right] - \E_{w_1x} \left[ \left( \phi(w_1^\top f(x)) - \E_x \left[\phi(w_1^\top f(x))\right] \right)^2 \right].
	\end{equation}
	As in the proof of Lemma \ref{lemma:concexpnorm2}, in the definition of $Z_i$, the randomness comes only from $w_i$, since we are conditioning on $(W_k)_{k=1}^{l-1}$.
	
	We have that
	\begin{equation}\label{eq:3fromsubGtosubE2}
		\begin{aligned}
			\subEnorm{Z_i} &\leq C_0 \subEnorm{\E_x \left[ \left( \phi(w_i^\top f(x)) - \E_x \left[\phi(w_i^\top f(x))\right] \right)^2 \right] } \\
			&\leq C_0 \E_x \left[ \subEnorm{ \left( \phi(w_i^\top f(x)) - \E_x \left[\phi(w_i^\top f(x))\right] \right)^2  } \right]\\
			&= C_0 \E_x \left[ \subGnorm{\phi(w_i^\top f(x)) - \E_x \left[\phi(w_i^\top f(x))\right]}^2 \right] \\
			&\leq C_0 \E_x \left[ \left( \subGnorm{\phi(w_i^\top f(x))} + \subGnorm{\E_x \left[\phi(w_i^\top f(x))\right]} \right)^2 \right] \\
			&\leq C_0 \E_x \left[ \left( \subGnorm{\phi(w_i^\top f(x))} + \E_x\left[ \subGnorm{\phi(w_i^\top f(x))}\right] \right)^2 \right],
		\end{aligned}
	\end{equation}
	where $C_0$ is a numerical constant, the first inequality follows from Exercise 2.7.10 of \cite{vershynin2018high}, the second line follows from Jensen's inequality as $\subEnorm{\cdot}$ is convex, the equality follows from Lemma 2.7.6 of \cite{vershynin2018high}, and the last line follows from Jensen's inequality as  $\subGnorm{\cdot}$ is convex.
	
	Recall that $(W)_{i,j}\distas{}_{\rm i.i.d.}\mathcal{N}(0, \beta^2_l/n_{L-1})$ and that the Gaussian distribution is rotationally invariant. Thus, $\phi'\left(w_i^\top f(x) \right)$ has the same distribution as $\phi'\left(\tilde c(x) \rho_i \right)$, where $(\rho_i)_{i=1}^{n_{L-1}}\sim_{\rm i.i.d.}\mathcal N(0, 1)$ and also independent of $\tilde c(x)$.
	Therefore,
	\begin{equation}\label{eq:3subGn}
		\begin{aligned}
			\subGnorm{\phi\left(w_i^\top f(x) \right)} &= \subGnorm{ \phi \left( \tilde{c}(x)\rho_i \right) } \\
			&\leq \subGnorm{\phi \left( \tilde{c}(x)\rho_i \right)  - \E_{\rho_i} \left[\phi \left( \tilde{c}(x)\rho_i \right) \right]} 
			+ \subGnorm{\E_{\rho_i} \left[ \phi \left( \tilde{c}(x)\rho_i \right) \right]} \\
			&\leq C_1\tilde{c}(x) + C_2\tilde{c}(x)+C_3 \leq C_4(\tilde{c}(x)+1),
		\end{aligned}
	\end{equation}
	where the first term in the RHS in the second line is bounded by $C_1\tilde{c}(x)$ for Theorem 5.2.2 in \cite{vershynin2018high}, and the second term is bounded by $C_2\tilde{c}(x)+C_3$ by following the same proof of Lemma \ref{lemma:constantexpectationphi}.

	Merging together \eqref{eq:3subGn} and \eqref{eq:3fromsubGtosubE2} we get
	\begin{equation}\label{eq:ZiinC5aresubE}
		\begin{aligned}
			\subEnorm{Z_i} \leq& C_0 \E_x \left[ \left( C_4 \tilde c(x) + C_4 \tilde c + C_5 \right)^2 \right] \\
			=& C_6 \E_x[(\tilde c(x) - \tilde c)^2] + C_4 \tilde c^2 + C_7 \tilde c + C_8 \\
 			=& C_6 \bigO{n_l^{-1}} + C_4 \tilde c^2 + C_7 \tilde c + C_8 \\
 			\leq& C_9,
		\end{aligned}
	\end{equation}
	where in the third line we use Lemma \ref{lemma:cclosetoexpect}.
	
	Hence, the $Z_i$-s are i.i.d. sub-exponential random variables, with sub-exponential norm bounded by a numerical constant. An application of Bernstein inequality (cf. Corollary 2.8.3. in \cite{vershynin2018high}) gives that 
	\begin{equation}\label{eq:bernsteincenteredfeatures}
		\P \left(\left |\frac{1}{n_l} \sum_{i=1}^{n_l} Z_i\right | > t\right) \leq 2 \exp \left( -c \min \left(\frac{t^2}{C^2},\frac{t}{C}\right) n_l \right),
	\end{equation}
	where $c, C$ are numerical constants. We recall that this probability is intended over $W_l$.
	
	Let's now focus on the first term in the last line of \eqref{eq:expansionofsquarednorm}, using the shorthand $w = w_1$, to ease the notation. We can rewrite this term as
	\begin{equation}\label{eq:firstthesis}
		\begin{aligned}\
			&\E_w \left[ \E_x \left[ \phi^2 (w^\top f(x)) \right] - \E_{xy} \left[\phi(w^\top f(x)) \phi (w^\top f(y)) \right] \right] \\
			& \hspace{1cm} = \E_x \E_w \left[ \phi^2 (w^\top f(x)) \right] - \E_{xy} \E_w \left[\phi(w^\top f(x)) \phi (w^\top f(y)) \right].
		\end{aligned}
	\end{equation}
	The aim of this part of the proof is to show that the quantity in \eqref{eq:firstthesis} is $\Theta(1)$.
	
Recall that $(W_l)_{i,j}\distas{}_{\rm i.i.d.}\mathcal{N}(0, \beta^2_l/n_{l-1})$ and that the Gaussian distribution is rotationally invariant. Thus, $\phi\left(w^\top f(x) \right)$ has the same distribution as $\phi\left(\tilde c(x) \rho \right)$, where $\rho\sim\mathcal N(0, 1)$ is independent of $\tilde c(x)$. We therefore have
	\begin{equation}\label{eq:dc1}
		\begin{aligned}
			& \left|\E_x \E_w \left[ \phi^2 \left(w^\top f(x) \right)-\phi^2 \left(w^\top f(x)\frac{\tilde{c}}{\tilde{c}(x)} \right)   \right]\right|  \\
			& =\left|\E_x \E_\rho \left[ \phi^2 (\rho \tilde c(x))-\phi^2 (\rho \tilde c)  \right]\right|  \\
			&\leq \E_x \left[ C_1 |\tilde c - \tilde c(x)| + C_2 \left| \tilde c^2 - \tilde c(x)^2\right| \right]\\
			&\leq \E_x \left[ C_1 |\tilde c - \tilde c(x)| + C_2 ( \tilde c - \tilde c(x))^2 + 2C_2 \tilde c \left| \tilde c - \tilde c(x) \right| \right]\\
			&= \bigO{n_{l-1}^{-1/2}},
		\end{aligned}
	\end{equation}
	where the third line follows from Lemma \ref{lemma:lipschitz}, and the last passage follows from Lemma \ref{lemma:cclosetoexpect}. Similarly, we have
	\begin{equation}\label{eq:dc2}
		\begin{aligned}
			& \left|\E_{xy} \E_w \left[  \phi(w^\top f(x)) \phi (w^\top f(y)) - \phi\left(w^\top f(x) \frac{\tilde{c}}{\tilde{c}(x)}\right) \phi \left(w^\top f(y) \frac{\tilde{c}}{\tilde{c}(y)}\right) \right] \right|\\
			& = \left|\E_{xy} \E_{\rho_1 \rho_2} \left[\phi(\rho_1 \tilde c(x)) \phi(\rho_2 \tilde c(y)) - \phi(\rho_1 \tilde c) \phi(\rho_2 \tilde c) \right] \right| \\
			&\leq \E_{xy} \left[ C_1 \left| \tilde c(x) - \tilde c \right| + C_2 \tilde c(y) \left| \tilde c(x) - \tilde c \right| + C_3 \left| \tilde c(y) - \tilde c \right| + C_4 \tilde c \left| \tilde c(y) - \tilde c \right| \right]\\
			&= \bigO{n_{l-1}^{-1/2}},
		\end{aligned}
	\end{equation}
	where $\rho_1$ and $\rho_2$ indicate two standard Gaussian random variables with correlation $f(x)^\top f(y) / (\norm{f(x)}\norm{f(y)})$, the third line follows from \ref{lemma:quadlipschphimixed}, and the last passage follows from Lemma \ref{lemma:cclosetoexpect}. By combining \eqref{eq:dc1} and \eqref{eq:dc2}, we have that
	\begin{equation}
		\begin{aligned}
			&\left|\E_w \left[ \E_x \left[ \phi^2 (w^\top f(x)) \right] - \E_{xy} \left[\phi(w^\top f(x)) \phi (w^\top f(y)) \right] \right]-\xi\right|= \bigO{n_{l-1}^{-1/2}},
		\end{aligned}
	\end{equation}
	with
	\begin{equation}\label{eq:dc3}
	\begin{split}
	    \xi:=&\E_x \E_w \left[ \phi^2 \left(w^\top f(x)\frac{\tilde{c}}{\tilde{c}(x)} \right)   \right]-\E_{xy} \E_w \left[   \phi\left(w^\top f(x) \frac{\tilde{c}}{\tilde{c}(x)}\right) \phi \left(w^\top f(y) \frac{\tilde{c}}{\tilde{c}(y)}\right) \right]\\
	    =&\E_{\rho_1} \left[ \tilde\phi^2 (\rho_1) \right] - \E_{xy} \left[ \E_{\rho_1 \rho_2} \left[ \tilde\phi(\rho_1)\tilde\phi(\rho_2) \right] \right],	\end{split}
	\end{equation}
	where we have set $\tilde \phi(t) = \phi(\tilde ct)$. Hence, in order to obtain that the quantity in \eqref{eq:firstthesis} is $\Theta(1)$, it suffices to prove that $\xi=\Theta(1)$.

As $\phi$ is Lipschitz and $\tilde{c}$ is $\Theta(1)$,  $\tilde \phi$ is also Lipschitz, which readily implies that 
$\xi=\bigO{1}$. We now prove that $\xi=\Omega(1)$. By  exploiting the Hermite expansion of $\tilde \phi$, we have that
	\begin{equation}\label{eq:sumher}
		\xi = \sum_{i=0}^\infty \mu^2_i \left( 1 - \E_{xy} \left[ \left( \frac{f(x)^\top f(y)}{\norm{f(x)}\norm{f(y)}} \right)^i \right] \right),
	\end{equation}
	where $\mu_i$ is the $i$-th Hermite coefficient of $\tilde \phi$. Note that, since we conditioned on $\tilde c = \Theta(1)$, these coefficients are numerical constants.
	As $\phi$ (and therefore $\tilde \phi$) is non constant, there exist $j>0$ such that $\mu_j \neq 0$. Furthermore, we have that the sum in \eqref{eq:sumher} contains only positive terms, as $|f(x)^\top f(y)|\le \norm{f(x)}\cdot \norm{f(y)}$ by Cauchy-Schwarz. Therefore, in order to show that $\xi=\Omega(1)$, it suffices to prove that, for all $j>0$,
	\begin{equation}\label{eq:expprove}
		\E_{xy} \left[ \left( \frac{\left| f(x)^\top f(y)\right| }{\norm{f(x)}\norm{f(y)}} \right)^j \right] \leq C_0 < 1,
	\end{equation}
	where $C_0$ is an absolute constant strictly smaller than $1$. Furthermore, \eqref{eq:expprove} is implied by the following:
	\begin{equation}\label{eq:probthesis}
		\PP_{x y} \left( \frac{\left| f(x)^\top f(y) \right|}{\norm{f(x)}\norm{f(y)}}  \leq C_1 \right) \geq c_1,
	\end{equation}
	where $C_1 < 1$ and $c_1 > 0$ are numerical constants.
	
	By writing $f(x)$ as $\tilde f(x) + f$ and $f(y)$ as $\tilde f(y) + f$, we have
	\begin{equation}\label{eq:5terms}
		\begin{aligned}
			\frac{\left| f(x)^\top f(y) \right| }{\norm{f(x)}\norm{f(y)}} &= \frac{\left| (\tilde f(x) + f)^\top (\tilde f(y) + f) \right| }{\norm{\tilde f(x) + f}\norm{\tilde f(y) + f}} \\
			&\leq \frac{\left| \tilde f(x)^\top (\tilde f(y) + f)\right|  + \left| f^\top \tilde f(y)\right|  + \norm{f}^2}{\min_{z \in \{x, y\}}\norm{\tilde f(z) + f}^2}\\
			&\leq \frac{\left| \tilde f(x)^\top f(y) \right| + \left| f^\top \tilde f(y) \right| + \norm{f}^2}{\min_{z \in \{x, y\}} \left( \norm{\tilde f(z)}^2 -2 \left| f^\top \tilde f(z)\right| \right) + \norm{f}^2}.
		\end{aligned}
	\end{equation}
Let us provide bounds on the various terms appearing in \eqref{eq:5terms}:

\begin{itemize}
    \item[(i)] Part (d) of the conditioning (cf. \eqref{eq:lastconditioningonW}) gives that
		\begin{equation}\label{eq:bdplug1}
			\mathbb P_{xy} \left(\min_{z \in \{x, y\}} \norm{\tilde f(z)}_2^2 \geq c n_{l-1} \right) \geq \frac{1}{4},
		\end{equation}
    for some numerical constant $c>0$.

    \item[(ii)] Part (b) of the conditioning gives that
		\begin{equation}\label{eq:normfub}
			\norm{f}_2^2 \leq C' n_{l-1}.
		\end{equation}
		
	\item[(iii)] 	Part (a) of the conditioning gives that $\norm{f_{l-1}(x)}_{\Lip} = \bigO{1}$, and part (b) of the conditioning gives that $\mathbb E_y\left[\norm{f_{l-1}(y)}\right] = \Theta(\sqrt{n_{l-1}})$. Hence, as $y\sim P_X$, Assumption \ref{ass:data_dist2} implies that
		\begin{equation}\label{eq:normy}
			\norm{f_{l-1}(y)} = \Theta(\sqrt{n_{l-1}}),
		\end{equation}
		with probability at least $1 - 2 \exp (-cn_{l-1})$ over $y$, where $c$ is a numerical constant. Furthermore, by recalling that $\E_x[\tilde f(x)]=0$ and using again Assumption \ref{ass:data_dist2}, we have that, for 
any fixed vector $u$,
		\begin{equation}\label{eq:ppx}
			\PP_x(|\tilde f(x)^\top u| > t) \leq 2 e^{-c_0 t^2 / \norm{u}_2^2},
		\end{equation}
		where $c_0$ is another numerical constant.
		Since $x$ and $y$ are independent, \eqref{eq:ppx} implies that
		\begin{equation}
			\PP_x(|\tilde f(x)^\top f(y)| > n_{l-1}^{3/4}) \leq 2 e^{-c_0 n_{l-1}^{3/2} / \norm{f(y)}_2^2} \leq 2 e^{-c_2 n_{l-1}^{1/2}},
		\end{equation}
		where the first inequality holds for every $y$, and the second inequality holds with probability at least $1 - 2 \exp (-cn_{l-1})$ over $y$ by \eqref{eq:normy}. As a result, we have
		\begin{equation}\label{eq:bdplug2}
			\PP_{xy}(|\tilde f(x)^\top f(y)| > n_{l-1}^{3/4}) \leq 2 e^{-c_2 \sqrt{n_{l-1}}} + 2 e^{-c n_{l-1}} \leq 4 e^{-c_3 \sqrt{n_{l-1}}}.
		\end{equation}

        \item[(iv)] By setting $t=n_{l-1}^{3/4}$ and $u=f$ into \eqref{eq:ppx}, we obtain
		\begin{equation}\label{eq:ppyfin}
			\PP_y(|\tilde f(y)^\top f| > n_{l-1}^{3/4}) \leq 2 e^{-c_4 \sqrt{n_{l-1}}},
		\end{equation}
		where $c_4$ is a numerical constant and we have also used \eqref{eq:normfub}.

    \item[(v)] Finally, as $x$ and $y$ are independent, \eqref{eq:ppyfin} implies that 
		\begin{equation}\label{eq:bdplug3}
			\PP_{xy}(\max_{z \in \{ x, y\}}|\tilde f(z)^\top f| > n_{l-1}^{3/4}) \leq 4 e^{-c_4 \sqrt{n_{l-1}}}.
		\end{equation}
\end{itemize}
By plugging into \eqref{eq:5terms} the bounds \eqref{eq:bdplug1}, \eqref{eq:bdplug2}, \eqref{eq:ppyfin} and \eqref{eq:bdplug3}, we obtain that
	\begin{equation}\label{eq:com5}
		\PP_{x, y} \left( \frac{\left|f(x)^\top f(y) \right|}{\norm{f(x)}\norm{f(y)}}  \leq \frac{2n_{l-1}^{3/4}+\norm{f}_2^2}{c\,n_{l-1}-2n_{l-1}^{3/4}+\norm{f}_2^2} \right) \geq 1/4 - 10 e^{-c_5 \sqrt{n_{l-1}}}.
	\end{equation}
	By using also \eqref{eq:normfub}, we have that \eqref{eq:probthesis} readily follows from \eqref{eq:com5}. From \eqref{eq:probthesis}, we have that $\xi=\Theta(1)$. Hence, the quantity in \eqref{eq:firstthesis} is $\Theta(1)$ and in particular it is lower bounded by a numerical constant, call it $C_0$.

	
	By setting $t = C_0 /2$ in \eqref{eq:bernsteincenteredfeatures}, we conclude that
	\begin{equation}\label{eq:stat1fin}
		\E_x \left[ \norm{f_{l}(x) - \E_x \left[f_{l}(x)\right] }_2^2 \right]  = \Theta(n_{l}),
	\end{equation}
	with probability at least $1 - 2\exp (-c n_{l-1})$ over $W_l$, where $c$ is an absolute constant. By taking into account the conditioning made at the beginning of the proof over the space $(W_k)_{k=1}^{l-1}$, we obtain that \eqref{eq:stat1fin} holds 
	with probability at least $1 - 5C' \exp(-c n_{n-1}) - 2\exp(-c n_{n-1}) \geq 1 - C \exp(-c n_{n-1})$ over $(W_k)_{k=1}^{l}$, where $C$ is a numerical constant, which concludes the proof of \eqref{eq:res1}.

	Finally, we prove \eqref{eq:res2} and \eqref{eq:res3}, again for layer $l$. By Lemma \ref{lemma:lipschitzconst}, we have that $\norm{f_l(x)}_{\Lip} = \bigO{1}$, with probability at least $1 - C' \exp(-c n_{L-1})$ over $(W_k)_{k=1}^l$. By conditioning on this event, we also have that $\norm{\norm{f_l(x) - \E_x [f_l(x) ]}_2}_{\Lip} = \bigO{1}$. Furthermore, we condition on a realization of $(W_k)_{k=1}^{l}$ such that \eqref{eq:res1} holds.
	
	To obtain \eqref{eq:res2}, we apply Jensen's inequality and Lemma \ref{lemma:squarenonsquare}, which give that
	\begin{equation}
		\E_x \left[ \norm{f_l(x) - \E_x [f_l(x)]}_2 \right] = \Theta(\sqrt{n_{l}}),
	\end{equation}
	with probability at least $1 - C' \exp(-c n_{L-1}) - 5C' \exp(-c n_{n-1}) - 2\exp(-c n_{n-1}) \geq 1 - C \exp(-c n_{n-1})$ over $(W_k)_{k=1}^l$.
	
	To obtain \eqref{eq:res3}, we condition on a realization of $(W_k)_{k=1}^{l}$ such that $\norm{\norm{f_l(x) - \E_x [f_l(x) ]}_2}_{\Lip} = \bigO{1}$ and \eqref{eq:res2} holds. Then, by Assumption \ref{ass:data_dist2}, we have that 
	\begin{equation}
		\begin{aligned}
			&\P_x \left( \left|\norm{f_l(x) - \E_x [f_l(x)]}_2 - \E_x \left[ \norm{f_l(x) - \E_x [f_l(x)]}_2 \right]\right| > \E_x \left[ \norm{f_l(x) - \E_x [f_l(x)]}_2 \right] / 2 \right) \\
			& \hspace{1cm} \leq 2 \exp(-c_1 n_l) \leq 2 \exp(-c n_{L-1}),
		\end{aligned}
	\end{equation}
	where $c$ is a numerical constant. This gives that
	\begin{equation}
		\norm{f_l(x) - \E_x [f_l(x)]} = \Theta(\sqrt{n_l}),
	\end{equation}
	with probability at least $1 - 6C' \exp(-c n_{n-1}) - 2\exp(-c n_{n-1}) - 2\exp(-c n_{n-1}) \geq 1 - C \exp(-c n_{n-1})$ over $x$ and $(W_k)_{k=1}^l$, which concludes the proof.
\end{proof}

\begin{lemma}[$\ell_2$ norms of centered backpropagation]\label{lemma:concnormcenteredwithD}
	Let $x\sim P_X$. Then, we have
	\begin{equation}\label{eq:statementC7}
		\norm{D_L\phi'(g_{L-1}(x)) - \E_x \left[D_L\phi'(g_{L-1}(x)) \right] }_2  = \Theta(\sqrt{n_{L-1}}),
	\end{equation}
	with probability at least $1 - 10 \exp(-c \log^2 n_{L-1})- C \exp(-cn_{L-1})$ over $x$ and $(W_k)_{k=1}^L$ over $(W_k)_{k=1}^{l}$ and $x$.
	\end{lemma}
\begin{proof}
	An application of  Lemma \ref{lemma:concnormcentered} for $l=L-2$ gives that
	\begin{equation}\label{eq:res3xx}
		\norm{f_{L-2}(x) - \E_x \left[f_{L-2}(x)\right] }_2  = \Theta(\sqrt{n_{L-2}}).
	\end{equation}
	with probability at least $1 - C' \exp (-c n_{L-1})$ over $(W_k)_{k=1}^{L-2}$ and $x$.
	
	To ease the notation, we use the shorthands $f(x) := f_{L-2}(x)$, $f = \E_x[f(x)]$, $\tilde f(x) = f(x) - f$, $W := W_{L-1}$ and $w_i = W_{:i}$. We also define $\tilde c(x) = \beta_{l} \norm{f( x)}/\sqrt{n_{l-1}}$ and $\tilde c = \E_x [\tilde c(x)]$.
	
	As in Lemma \ref{lemma:concnormcentered}, we condition on the 3 events (a)-(c), which jointly happen with probability at least $1 - 4C' \exp (-c n_{L-1})$ over $(W_k)_{k=1}^{L-2}$. Note that, to condition on the event (c), we use \eqref{eq:res3xx}.

    Now, we can write
	\begin{equation}\label{eq:expansionofsquarednorm2}
		\begin{aligned}
			& \E_x \left[ \norm{D_L\phi'(g_{L-1}(x)) - \E_x \left[D_L\phi'(g_{L-1}(x)) \right] }_2^2 \right] \\
			\hspace{1cm} &= \E_x \left[ \norm{D_L\phi'(W^\top f_{L-2}(x)) - \E_x \left[D_L\phi'(W^\top f_{L-2}(x)) \right] }_2^2 \right] \\
			\hspace{1cm} &= n_{L-1} \left( \frac{1}{n_{L-1}} \sum_{i=1}^{n_{L-1}} (D_L)_{ii}^2 \E_x \left[ \left( \phi'(w_i^\top f(x)) - \E_x \left[\phi'(w_i^\top f(x))\right] \right)^2  \right] \right) \\
			\hspace{1cm} &= n_{L-1} \left(\frac{1}{n_{L-1}} \sum_{i=1}^{n_{L-1}} (D_L)_{ii}^2 \E_{xw_1} \left[ \left( \phi'(w_1^\top f(x)) - \E_x \left[\phi'(w_1^\top f(x))\right] \right)^2 \right] + \frac{1}{n_{L-1}} \sum_{i=1}^{n_{L-1}} Z_i \right),
		\end{aligned}
	\end{equation}
	where we use that the $w_i$-s are identically distributed and we have defined the independent, mean-0 random variables
	\begin{equation}\label{eq:3Zinormyy}
	    \begin{aligned}
	        Z_i = & (D_L)_{ii}^2 \E_x \left[ \left( \phi'(w_i^\top f(x)) -  \E_x \left[\phi'(w_i^\top f(x))\right] \right)^2 \right] \\
	        \hspace{1cm} & - (D_L)_{ii}^2 \E_{w_1x} \left[ \left( \phi'(w_1^\top f(x)) - \E_x \left[\phi'(w_1^\top f(x))\right] \right)^2 \right].
	    \end{aligned}
	\end{equation}
	Note that in the definition of $Z_i$ the randomness comes only from $w_i$ and $(D_L)_{ii}$, since we are conditioning on $(W_k)_{k=1}^{L-2}$.

	If we fix the $(D_L)_{ii}$-s and follow the same argument in \eqref{eq:3fromsubGtosubE2}-\eqref{eq:ZiinC5aresubE} (cf. the proof of Lemma \ref{lemma:concnormcentered}), we have
	\begin{equation}
	    \subEnorm{Z_i} \le C_0 (D_L)_{ii}^2,
	\end{equation}
	where $C_0$ is a numerical constant and we have used that $\phi'$ is Lipschitz. Let $\mathcal E_{\rm bad}$ be the event s.t. $\max_i (D_L)_{ii}^2 > \log^2 n_{L-1}$. Then, by following the same argument as in Lemma \ref{lemma:opnormlog}, we have that 
	\begin{equation}\label{eq:cdprobxx}
	    \mathbb P(\mathcal E_{\rm bad})\le 2 \exp(- c\log^2 n_{L-1}).
	\end{equation}
	Hence, by conditioning on $\mathcal E_{\rm bad}^c$, we have that
	\begin{equation}
	    \max_i \subEnorm{Z_i} \leq C_0 \log^2 n_{L-1}.
	\end{equation}
    By applying Bernstein inequality (cf. Corollary 2.8.3. in \cite{vershynin2018high}), we get
	\begin{equation}\label{eq:concnormsmallpart2xx}
		\P \left(\left |\frac{1}{n_{L-1}} \sum_{i=1}^{n_{L-1}} Z_i\right | > \frac{1}{\sqrt[4]{n_{L-1}}}\, \Bigg | \,\mathcal E_{\rm bad}^c\right) \leq 2 \exp \left(-c \frac{\sqrt{ n_{L-1}}}{\log^4 n_{L-1}}\right),
	\end{equation}
	for some numerical constant $c$. By combining  \eqref{eq:cdprobxx} and \eqref{eq:concnormsmallpart2xx}, we obtain that 
	\begin{equation}\label{eq:concnormsmallpart2newxx}
	\begin{split}
		\P \left(\left |\frac{1}{n_{L-1}} \sum_{i=1}^{n_{L-1}} Z_i\right | > \frac{1}{\sqrt[4]{n_{L-1}}}\right) &\leq 2 \exp \left(-c \frac{\sqrt{ n_{L-1}}}{\log^4 n_{L-1}}\right) + 2 \exp\left(- c \log^2 n_{L-1}\right)\\
		&\leq 4\exp\left(- c \log^2 n_{L-1}\right),
		\end{split}
\end{equation}
	where this probability is over $W_{L-1}$ and $W_L$.

    Let's now focus on the first term in the last line of \eqref{eq:expansionofsquarednorm2}. In particular, we have that
	\begin{equation}\label{eq:centralterm}
			\xi = \E_{xw_1} \left[ \left( \phi'(w_1^\top f(x)) - \E_x \left[\phi'(w_1^\top f(x))\right] \right)^2 \right] = \Theta(1).
	\end{equation}
	This can be proven by following the same argument in \eqref{eq:dc1}-\eqref{eq:com5} (cf. the proof of Lemma \ref{lemma:concnormcentered}), as $\phi'$ is Lipschitz and non-constant.
	

	Next, we re-write \eqref{eq:expansionofsquarednorm2} as 
	\begin{equation}\label{eq:DLnormlong22xx}
		\begin{aligned}
		 & \E_x \left[ \norm{D_L\phi'(g_{L-1}(x)) - \E_x \left[D_L\phi'(g_{L-1}(x)) \right] }_2^2 \right] \\
			\hspace{1cm} &= n_{L-1} \left(\frac{1}{n_{L-1}} \sum_{i=1}^{n_{L-1}} (D_L)_{ii}^2 \E_{xw_1} \left[ \left( \phi'(w_1^\top f(x)) - \E_x \left[\phi'(w_1^\top f(x))\right] \right)^2 \right] + \frac{1}{n_{L-1}} \sum_{i=1}^{n_{L-1}} Z_i \right) \\
			\hspace{1cm} &= n_{L-1} \left(\xi \E_{W_L}\left[(D_L)_{11}^2\right] + \xi \frac{1}{n_{L-1}} \sum_{i=1}^{n_{L-1}} \tilde Z_i +   \frac{1}{n_{L-1}} \sum_{i=1}^{n_{L-1}} Z_i \right) \\
			\hspace{1cm} &= n_{L-1} \left(\xi \beta_L^2 + \xi \frac{1}{n_{L-1}}\sum_{i=1}^{n_{L-1}} \tilde Z_i +  \frac{1}{n_{L-1}} \sum_{i=1}^{n_{L-1}} Z_i \right),
		\end{aligned}
	\end{equation}
	where we have defined the independent, mean-0, sub-exponential random variables
	\begin{equation}
		\tilde Z_i = (D_L)_{ii}^2 - \E_{W_L}\left[(D_L)_{11}^2\right].
	\end{equation}
	Since the $(D_L)_{ii}$-s are standard Gaussian, we have
	\begin{equation}\label{eq:b0xx}
	    \subEnorm{\tilde Z_i} \leq \subEnorm{(D_L)_{ii}^2} = \subGnorm{(D_L)_{ii}}^2 = C_1.
	\end{equation}
	
	Hence, another application of Bernstein inequality (cf. Corollary 2.8.3. in \cite{vershynin2018high}) allows us to conclude that
	\begin{equation}\label{eq:b1xx}
	    \left|\frac{1}{n_{L-1}}\sum_{i=1}^{n_{L-1}} \tilde Z_i\right| = \bigO{n_{L-1}^{-1/4}},
	\end{equation}
	with probability at least $1 - 2 \exp (-c \sqrt{ n_{L-1}} )$ over $W_L$.
	
	Thus, by using \eqref{eq:concnormsmallpart2newxx} and \eqref{eq:b1xx}, and taking into account the initial conditioning over $(W_k)_{k=1}^{L-2}$, we conclude that
	\begin{equation}\label{eq:partialresultC7}
	    \E_x \left[ \norm{D_L\phi'(g_{L-1}(x)) - \E_x \left[D_L\phi'(g_{L-1}(x)) \right] }_2^2 \right] =  \Theta( n_{L-1}),
	\end{equation}
	with probability at least $1 - 6 \exp (-c \log^2 n_{L-1}) - 6 C' \exp(-c n_{L-1})$ over $(W_k)_{k=1}^{L}$.

    Proceeding in a similar fashion as in Lemma \ref{lemma:concnormcentered}, we apply Lemma \ref{lemma:lipconstbacklastlayer}, which gives that $\norm{D_L\phi'(g_{L-1}(x))}_{\Lip} = \bigO{\log n_{L-1}}$, with probability at least $1 - 2 \exp (- c \log ^2 n_{L-1}) - C' \exp{(-n_{L-1})}$ over $(W_k)_{k=1}^L$. By conditioning on this event, we also have that $\norm{\norm{D_L\phi'(g_{L-1}(x)) - \E_x \left[D_L\phi'(g_{L-1}(x)) \right] }_2}_{\Lip} = \bigO{\log n_{L-1}}$. Furthermore, we condition on a realization of $(W_k)_{k=1}^{L}$ such that \eqref{eq:partialresultC7} holds.
	
	We can now apply Jensen's inequality and Lemma \ref{lemma:squarenonsquare}, to obtain that
	\begin{equation}\label{eq:lesspartialresultC7}
		\E_x \left[ \norm{D_L\phi'(g_{L-1}(x)) - \E_x \left[D_L\phi'(g_{L-1}(x)) \right] }_2 \right] = \Theta(\sqrt{n_{l}}),
	\end{equation}
	with probability at least $1 - 8 \exp(-c \log^2 n_{L-1})- 7 C' \exp(-cn_{L-1})$ over $(W_k)_{k=1}^L$.

	Finally, we condition on a realization of $(W_k)_{k=1}^{L}$ such that $\norm{\norm{D_L\phi'(g_{L-1}(x)) - \E_x \left[D_L\phi'(g_{L-1}(x)) \right]}_2}_{\Lip} = \bigO{\log n_{L-1}}$ and \eqref{eq:lesspartialresultC7} hold. Then, by Assumption \ref{ass:data_dist2}, we have that 
	\begin{equation}
		\begin{aligned}
			&\P_x \big( \left|\norm{D_L\phi'(g_{L-1}(x)) - \E_x \left[D_L\phi'(g_{L-1}(x)) \right]}_2 - \E_x \left[ \norm{D_L\phi'(g_{L-1}(x)) - \E_x \left[D_L\phi'(g_{L-1}(x)) \right]}_2 \right]\right| \\
			& \hspace{5cm} > \E_x \left[ \norm{D_L\phi'(g_{L-1}(x)) - \E_x \left[D_L\phi'(g_{L-1}(x)) \right]}_2 \right] / 2 \big) \\
			& \hspace{10cm} \leq 2 \exp^(-c n_{L-1}).
		\end{aligned}
	\end{equation}
	This gives that
	\begin{equation}
		\norm{D_L\phi'(g_{L-1}(x)) - \E_x \left[D_L\phi'(g_{L-1}(x)) \right]}_2 = \Theta(\sqrt{n_{L-1}}),
	\end{equation}
	with probability at least $1 - 10 \exp(-c \log^2 n_{L-1}) - 8C' \exp(-cn_{L-1})$ over $x$ and $(W_k)_{k=1}^L$. This concludes the proof.
\end{proof}

\section{Proofs for Part 1: Centering}\label{app:centeringJ}

\subsection{Step (a): Centering $F_{L-2}$ and $B_{L-1}$}\label{app:centeringJ1}

\begin{lemma}[Centering $F_{L-2}$ and $B_{L-1}$]\label{lemma:cent1}
Consider the setting of Theorem \ref{thm:main}, let $F_{L-2}\in \mathbb R^{N\times n_{L-2}}$ be the feature matrix at layer $L-2$, and let $B_{L-1}$ contain the backpropagation terms from layer $L-1$, \ie $(B_{L-1})_{i:}= D_L \phi'(g_{L-1}(x_i))$. Let $J_{L-2}J_{L-2}^\top = F_{L-2}F_{L-2}^\top \circ B_{L-1}B_{L-1}^\top$ and $\cJ_{FB} \cJ_{FB}^\top = \Tilde F_{L-2} \Tilde F_{L-2}^\top \circ \Tilde B_{L-1} \Tilde B_{L-1}^\top$, where $\tilde{F}_{L-2}=F_{L-2}-\E_X[F_{L-2}]$ and $\tilde{B}_{L-1}=B_{L-1}-\E_X[B_{L-1}]$. Then, we have that
	\begin{equation}
	\evmin{J_{L-2}J_{L-2}^\top} \geq \evmin{\cJ_{FB} \cJ_{FB}^\top} - o(n_{L-1}n_{L-2}),
	\end{equation}
	with probability at least $1 - C \exp (- c n_{L-1}) - 4 \exp (-c \log^2 N) - 8 \exp (-c \log^2 n_{L-1})$ over $(W_k)_{k=1}^{L}$ and $(x_i)_{i=1}^N\sim_{\rm i.i.d.}P_X$, where $c$ and $C$ are numerical constants.
\end{lemma}
\begin{proof}
\simone{
    By Lemma \ref{lemma:lipschitzconst} and Lemma \ref{lemma:lipconstbacklastlayer}, we have that $\norm{f_{L-2}}_{\Lip} = \bigO{1}$ and $\norm{D_L\phi'(g_{L-1}(x))}_{\Lip} = \bigO{\log n_{L-1}}$ with probability $1 - C \exp (- n_{L-1}) - 2 \exp (-c \log^2 n_{L-1})$ over $(W_k)_{k=1}^L$. We will condition on these events for the rest of the proof.
    }
    
    \simone{
    Let's define $\cJ_{F} \cJ_{F}^\top = \Tilde F_{L-2} \Tilde F_{L-2}^\top \circ B_{L-1} B_{L-1}^\top$. We can now re-write the quantity $J_{L-2}J_{L-2}^\top$ as follows:
	\begin{equation}\label{eq:firstcentering}
		\begin{aligned}
			J_{L-2}J_{L-2}^\top &= \cJ_F \cJ_F^\top + \E[F_{L-2}] \E[F_{L-2}]^\top \circ B_{L-1}B_{L-1}^\top  \\
			&\hspace{-3em}+ \left( (F_{L-2} - \E[F_{L-2}]) \E[F_{L-2}]^\top + \E[F_{L-2}] (F_{L-2} - \E[F_{L-2}])^\top \right) \circ B_{L-1}B_{L-1}^\top \\
			& = \cJ_F \cJ_F^\top + \norm{\nu}_2^2 B_{L-1}B_{L-1}^\top + \left( \Lambda \mathbf{1} \mathbf{1}^\top + \mathbf{1}\mathbf{1}^\top \Lambda \right) \circ B_{L-1}B_{L-1}^\top\\
			& = \cJ_F \cJ_F^\top + \left(\norm{\nu}_2 \mathbf{1} + \frac{\Lambda \mathbf{1}}{\norm{\nu}_2} \right)\left(\norm{\nu}_2 \mathbf{1} + \frac{\Lambda \mathbf{1}}{\norm{\nu}_2} \right)^\top \circ B_{L-1}B_{L-1}^\top \\ 
			& \hspace{0.5cm} - \frac{\Lambda \mathbf{1} \mathbf{1}^\top \Lambda}{\norm{\nu}_2^2} \circ B_{L-1}B_{L-1}^\top \\
			& \succeq \cJ_F \cJ_F^\top - \frac{\Lambda \mathbf{1} \mathbf{1}^\top \Lambda}{\norm{\nu}_2^2} \circ B_{L-1}B_{L-1}^\top,
		\end{aligned}
	\end{equation}
	where $\nu = \E_{x_i}[(F_{L-2})_{i:}]\in\mathbb R^{n_{L-2}}$ (independent on $i$, since the $x_i$-s are i.i.d.),  $\Lambda$ is a diagonal matrix such that $\Lambda_{ii} = \nu^\top (F_{L-2})_{i:} - \norm{\nu}_2^2 =: \mu(x_i)$, and $\mathbf{1} \in \R^N$ is a vector full of ones. The last step is justified since the Hadamard product of PSD matrices is PSD by the Schur product theorem. Notice that we are assuming $\norm{\nu}_2 \neq 0$. In fact, if $\norm{\nu}_2 = 0$, then we immediately have that $J = \tilde J_{F}$.
	}
	
	\simone{
	Expanding in an analogous way the term $\cJ_F \cJ_F^\top$, we get
	\begin{equation}\label{eq:secondcentering}
		\begin{aligned}
			J_{L-2}J_{L-2}^\top & \succeq \cJ_F \cJ_F^\top - \frac{\Lambda \mathbf{1} \mathbf{1}^\top \Lambda}{\norm{\nu}_2^2} \circ B_{L-1}B_{L-1}^\top \\
			&= \cJ_{FB} \cJ_{FB}^\top + \left(\norm{\eta}_2 \mathbf{1} + \frac{\Gamma \mathbf{1}}{\norm{\eta}_2} \right) \left(\norm{\eta}_2 \mathbf{1} + \frac{\Gamma \mathbf{1}}{\norm{\eta}_2} \right)^\top \circ \tilde F_{L-2} \Tilde F_{L-2}^\top \\
			&  \hspace{0.5cm} - \frac{\Gamma \mathbf{1} \mathbf{1}^\top \Gamma}{\norm{\eta}_2^2} \circ \tilde F_{L-2} \Tilde F_{L-2}^\top - \frac{\Lambda \mathbf{1} \mathbf{1}^\top \Lambda}{\norm{\nu}_2^2} \circ B_{L-1}B_{L-1}^\top \\
			& \succeq \cJ_{FB} \cJ_{FB}^\top + \left(\norm{\eta}_2 \mathbf{1} + \frac{\Gamma \mathbf{1}}{\norm{\eta}_2} \right) \left(\norm{\eta}_2 \mathbf{1} + \frac{\Gamma \mathbf{1}}{\norm{\eta}_2} \right)^\top \circ \tilde F_{L-2} \left( \frac{\nu \nu^\top}{\norm{\nu}_2^2} \right) \Tilde F_{L-2}^\top \\
			&  \hspace{0.5cm} - \frac{\Gamma \mathbf{1} \mathbf{1}^\top \Gamma}{\norm{\eta}_2^2} \circ \tilde F_{L-2} \Tilde F_{L-2}^\top - \frac{\Lambda \mathbf{1} \mathbf{1}^\top \Lambda}{\norm{\nu}_2^2} \circ B_{L-1}B_{L-1}^\top
		\end{aligned}
	\end{equation}
	where $\eta = \E_{x_i}[(B_{L-1})_{i:}]\in\mathbb R^{n_{L-1}}$ (independent on $i$, since the $x_i$-s are i.i.d.),  $\Gamma$ is a diagonal matrix such that $\Gamma_{ii} = \eta^\top (B_{L-1})_{i:} - \norm{\eta}_2^2 =: \zeta(x_i)$. The last step is justified by the fact that the following matrix is PSD
	$$
	\left(\norm{\eta}_2 \mathbf{1} + \frac{\Gamma \mathbf{1}}{\norm{\eta}_2} \right) \left(\norm{\eta}_2 \mathbf{1} + \frac{\Gamma \mathbf{1}}{\norm{\eta}_2} \right)^\top \circ \tilde F_{L-2} \left(I- \frac{\nu \nu^\top}{\norm{\nu}_2^2} \right) \Tilde F_{L-2}^\top,
	$$
	since it is the Hadamard product of two PSD matrices. Notice that we are assuming $\norm{\eta}_2 \neq 0$. In fact, if $\norm{\eta}_2 = 0$, then we immediately have that $\tilde J_{F} = \tilde J_{FB}$.
	}
	
	\simone{
	Taking into account the following relations
	\begin{equation}
	    \Lambda \mathbf{1} = \Tilde F_{L-2} \nu, \qquad 
	    \Gamma \mathbf{1} = \Tilde B_{L-1} \eta, \qquad
	    \E_X[B_{L-1}] = \mathbf{1} \eta^\top,
	\end{equation}
	we can simplify the second and the fourth term of the RHS of equation \eqref{eq:secondcentering} as follows
	\begin{equation}
	\begin{aligned}
   	    & \left(\norm{\eta}_2 \mathbf{1} + \frac{\Gamma \mathbf{1}}{\norm{\eta}_2} \right) \left(\norm{\eta}_2 \mathbf{1}+
   	    \frac{\Gamma \mathbf{1}}{\norm{\eta}_2} \right)^\top \circ \tilde F_{L-2} \left( \frac{\nu \nu^\top}{\norm{\nu}_2^2}
   	    \right) \Tilde F_{L-2}^\top - B_{L-1}B_{L-1}^\top \circ \frac{\Lambda \mathbf{1} \mathbf{1}^\top \Lambda}{\norm{\nu}_2^2} \\
   	    & = \left(\norm{\eta}_2 \mathbf{1} + \frac{\Gamma \mathbf{1}}{\norm{\eta}_2} \right) \left(\norm{\eta}_2 \mathbf{1}+
   	    \frac{\Gamma \mathbf{1}}{\norm{\eta}_2} \right)^\top \circ \frac{\Lambda \mathbf{1} \mathbf{1}^\top \Lambda}{\norm{\nu}_2^2} \\
   	    & \hspace{1cm} - \left(\mathbf{1} \eta^\top + \Tilde B_{L-1} \right)\left(\mathbf{1} \eta^\top + \Tilde B_{L-1} \right)^\top \circ \frac{\Lambda \mathbf{1} \mathbf{1}^\top \Lambda}{\norm{\nu}_2^2} \\
   	    &= \left(\frac{\Gamma \mathbf{1} \mathbf{1}^\top \Gamma}{\norm{\eta}_2^2} - \Tilde{B}_{L-1} \Tilde{B}_{L-1}^\top \right) \circ \frac{\Lambda \mathbf{1} \mathbf{1}^\top \Lambda}{\norm{\nu}_2^2} \succeq - \Tilde{B}_{L-1} \Tilde{B}_{L-1}^\top \circ \frac{\Lambda \mathbf{1} \mathbf{1}^\top \Lambda}{\norm{\nu}_2^2}.
	\end{aligned}
	\end{equation}
	Merging this last relation with \eqref{eq:secondcentering} we get
	\begin{equation}\label{eq:firstsecondcentering}
	\begin{aligned}
   		J_{L-2}J_{L-2}^\top & \succeq \cJ_{FB} \cJ_{FB}^\top - \frac{\Gamma \mathbf{1} \mathbf{1}^\top \Gamma}{\norm{\eta}_2^2} \circ \tilde F_{L-2} \Tilde F_{L-2}^\top - \frac{\Lambda \mathbf{1} \mathbf{1}^\top \Lambda}{\norm{\nu}_2^2} \circ \Tilde B_{L-1} \Tilde B_{L-1}^\top \\
   		& = \cJ_{FB} \cJ_{FB}^\top - \left(\frac{\Gamma}{\norm{\eta}_2}\tilde F_{L-2} \right)\left(\frac{\Gamma}{\norm{\eta}_2}\tilde F_{L-2} \right)^\top - \left(\frac{\Lambda}{\norm{\nu}_2}\tilde B_{L-1} \right)\left(\frac{\Lambda}{\norm{\nu}_2}\tilde B_{L-1} \right)^\top.
	\end{aligned}
	\end{equation}
}
	
\simone{	
	Note that $\norm{\mu}_{\Lip} \leq \norm{f_{L-2}}_{\Lip}\norm{\nu}_2$, and that $\E_{x_i}[\mu(x_i)] = 0$ for all $i\in [N]$. Thus, by using Assumption \ref{ass:data_dist2} on $x_i$ and exploiting the initial conditioning on the weights, we have that 
	\begin{equation}
		\P(|\mu(x_i)|/\norm{\nu}_2 > t) < 2 \exp (-c t^2),
	\end{equation}
	where the probability is intended over $x_i\sim P_X$. Thus, following the same argument of Lemma \ref{lemma:opnormlog}, the last relation implies that
    \begin{equation}
        \opnorm{\Lambda / \norm{\nu}_2} = \bigO{\log N},
    \end{equation}
    with probability at least $1 - 2 \exp (-c \log^2 N)$, where $c$ is a numerical constant, and the probability is intended over $\{ x_i\}_{i=1}^N$. This implies that
    \begin{equation}\label{eq:firstopnorm}
    \begin{aligned}
        & \opnorm{\left(\frac{\Lambda}{\norm{\nu}_2}\tilde B_{L-1} \right)\left(\frac{\Lambda}{\norm{\nu}_2}\tilde B_{L-1}\right)^\top} \\
        & \hspace{1cm} \leq \opnorm{\frac{\Lambda}{\norm{\nu}_2}}^2 \opnorm{\tilde B_{L-1} \tilde B_{L-1}^\top}
        = \bigO{(N + n_{L-1})\cdot \log^2 N \cdot \log^2 n_{L-1}} = o(n_{L-2} n_{L-1}),
    \end{aligned}
    \end{equation}
    where the second equality is justified by Lemma \ref{lemma:opnormcentbackprop}, and the last by Lemma \ref{lemma:alphabeta}. This result holds with probability $1 - C \exp (- c n_{L-1}) - 2 \exp (-c \log^2 N) - 4 \exp (-c \log^2 n_{L-1})$ over $(W_k)_{k=1}^{L}$ and $(x_i)_{i=1}^N$.
    }
    
    \simone{
    Note that $\norm{\zeta}_{\Lip} \leq \norm{D_L\phi'(g_{L-1}(x))}_{\Lip} \norm{\eta}_2$, and that $\E_{x_i}[\zeta(x_i)] = 0$ for all $i\in [N]$. Thus, by using Assumption \ref{ass:data_dist2} on $x_i$ and exploiting the initial conditioning on the weights, we have that
	\begin{equation}
		\P(|\zeta(x_i)|/\norm{\eta}_2 > t\cdot \log{n_{L-1}}) < 2 \exp (-c t^2),
	\end{equation}
	where the probability is intended over $x_i\sim P_X$. Thus, following the same argument of Lemma \ref{lemma:opnormlog}, the last relation implies that
    \begin{equation}
        \opnorm{\Gamma / \norm{\eta}_2} = \bigO{\log N \cdot \log n_{L-1}},
    \end{equation}
    with probability at least $1 - 2 \exp (-c \log^2 N)$, where $c$ is a numerical constant, and the probability is intended over $\{ x_i\}_{i=1}^N$. This implies that
    \begin{equation}\label{eq:secondopnorm}
    \begin{aligned}
    & \opnorm{\left(\frac{\Gamma}{\norm{\eta}_2}\tilde F_{L-2} \right)\left(\frac{\Gamma}{\norm{\eta}_2}\tilde F_{L-2} \right)^\top} \\
    & \hspace{1cm} \leq \opnorm{\frac{\Gamma}{\norm{\eta}_2}}^2 \opnorm{\tilde F_{L-2} \tilde F_{L-2}^\top} = \bigO{(N + n_{L-2})\cdot\log^2 N \cdot \log^2 n_{L-1} } = o(n_{L-2} n_{L-1}),
    \end{aligned}
    \end{equation}
    where the second equality is justified by Lemma \ref{lemma:opnormcentfeat}, and the last by Lemma \ref{lemma:alphabeta}. This result holds with probability $1 - C \exp (- c n_{L-1}) - 2 \exp (-c \log^2 N) - 4 \exp (-c \log^2 n_{L-1})$ over $(W_k)_{k=1}^{L}$ and $(x_i)_{i=1}^N$.
    }
    
    \simone{
    By merging \eqref{eq:firstopnorm} and \eqref{eq:secondopnorm} with \eqref{eq:firstsecondcentering}, we readily obtain the desired result.
    }
\end{proof}

\subsection{Step (b): Centering everything}\label{app:centeringJ3}

\begin{lemma}[Centering everything]\label{lemma:cent3}
Consider the setting of Theorem \ref{thm:main}, let $F_{L-2}\in \mathbb R^{N\times n_{L-2}}$ be the feature matrix at layer $L-2$, and let $B_{L-1}$ contain backpropagation terms from layer $L-1$, \ie $(B_{L-1})_{i:}= D_L \phi'(g_{L-1}(x_i))$. Let $\cJ_{FB} \cJ_{FB}^\top = \Tilde F_{L-2} \Tilde F_{L-2}^\top \circ \Tilde B_{L-1}\Tilde B_{L-1}^\top$ and $\cJ_{L-2} \cJ_{L-2}^\top = \Tilde F_{L-2} \Tilde F_{L-2}^\top \circ \Tilde B_{L-1}\Tilde B_{L-1}^\top-\mathbb E_X[\Tilde F_{L-2} \Tilde F_{L-2}^\top \circ \Tilde B_{L-1}\Tilde B_{L-1}^\top]$, where $\tilde{F}_{L-2}=F_{L-2}-\E_X[F_{L-2}]$ and $\tilde{B}_{L-1}=B_{L-1}-\E_X[B_{L-1}]$.
Then, we have that 
    \begin{equation}
        \evmin{\cJ_{FB} \cJ_{FB}^\top} \geq  \evmin{\cJ_{L-2} \cJ_{L-2}^\top} - o(n_{L-1}n_{L-2}),
    \end{equation}
    with probability at least $1 - C\exp(-n_{L-1}) - 2 \exp(-c \log ^2 n_{L-1}) - 2 \exp (-c \log^2 N)$ over $(W_k)_{k=1}^L$ and $(x_i)_{i=1}^N$.
\end{lemma}

\begin{proof}
Note that the $i$-th row of $\cJ_{FB}$ is now in the form
\begin{equation}
	(\cJ_{FB})_{i:} = \tilde f_{L-2}(x_i) \otimes (D_L \tilde{\phi'}(g_{L-1}(x_i))),
\end{equation}
where we recall that $\tilde f_{L-2}(x_i)= f_{L-2}(x_i)-\E_{x_i}[ f_{L-2}(x_i)]$ and $\tilde{\phi'}(g_{L-1}(x_i))=\phi'(g_{L-1}(x_i))-\E[\phi'(g_{L-1}(x_i))]$.
Furthermore, $(\cJ_{L-2})_{i:} = (\cJ_{FB})_{i:} - \E_{x_i}[(\cJ_{FB})_{i:}]$. \simone{Then, by following similar passages as in \eqref{eq:firstcentering}, we have
\begin{equation}\label{eq:lastcentering}
	\begin{aligned}
		\cJ_{FB} \cJ_{FB}^\top & = \cJ_{L-2} \cJ_{L-2}^\top + \E[\cJ_{FB}] \E[\cJ_{FB}]^\top \\
		& \hspace{3em} + (\cJ_{FB} - \E[\cJ_{FB}]) \E[\cJ_{FB}]^\top + \E[\cJ_{FB}] (\cJ_{FB} - \E[\cJ_{FB}])^\top \\
		&=  \cJ_{L-2} \cJ_{L-2}^\top + \norm{A}_F^2 \mathbf{1}\mathbf{1}^\top + \Lambda \mathbf{1}\mathbf{1}^\top + \mathbf{1}\mathbf{1}^\top \Lambda \\ 
		&= \cJ_{L-2} \cJ_{L-2}^\top  + \left( \norm{A}_F \mathbf{1} + \frac{\Lambda \mathbf{1}}{\norm{A}_F}\right) \left( \norm{A}_F \mathbf{1} + \frac{\Lambda \mathbf{1}}{\norm{A}_F}\right)^\top - \frac{\Lambda \mathbf{1} \mathbf{1}^\top \Lambda}{\norm{A}_F^2} \\
		&\succeq \cJ_{L-2} \cJ_{L-2}^\top - \frac{\Lambda \mathbf{1} \mathbf{1}^\top \Lambda}{\norm{A}_F^2},
	\end{aligned}
\end{equation}
where we have defined
\begin{equation}\label{eq:A}
	A = \E_{x_i} \left[\tilde f_{L-2}(x_i)(D_L \tilde{\phi'}(g_{L-1}(x_i)))^\top \right],
\end{equation}
which is independent on $i$ (as the $x_i$-s are i.i.d.), and $\Lambda$ is a diagonal matrix that contains in the $i$-th position
\begin{equation}
	\Lambda_{ii} = \tilde f_{L-2}(x_i)^\top A (D_L \tilde{\phi'}(g_{L-1}(x_i))) - \E_{x_i} \left[ \tilde f_{L-2}(x_i) ^\top A (D_L \tilde{\phi'}(g_{L-1}(x_i))) \right].
\end{equation}
The last passage of \eqref{eq:lastcentering} is true since we are subtracting a PSD matrix.
}

\simone{
An application of Lemma \ref{lemma:lipschitzconst} gives that $\| f_{L-2}(x_i)\|_{\Lip}$ and $\|\phi'(g_{L-1}(x_i))\|_{\Lip}$ are upper bounded by a numerical constant both with probability at least $1 - C\exp(-n_{L-1})$ over $(W_k)_{k=1}^{L-1}$. Let us condition on this event on the probability space of $(W_k)_{k=1}^{L-1}$. Then, we can apply Lemma \ref{lemma:HW} with $u(x):=\tilde{f}_{L-2}(x)$ and $v(x):=\tilde{\phi'}(g_{L-1}(x))$, which implies that
\begin{equation}\label{eq:Lii}
	\norm{\Lambda_{ii}}_{\psi_1} < C \norm{A D_L}_F \leq C \norm{A}_F \opnorm{D_L} = \bigO{\log n_{L-1}} \norm{A}_F.
\end{equation}
In \eqref{eq:Lii}, $C$ is a numerical constant and the last equality holds with probability at least $1 - 2 \exp(-c \log ^2 n_{L-1})$ over $W_L$ by Lemma \ref{lemma:opnormlog}. Thus,
\begin{equation}
	\P(|\Lambda_{ii}|/\norm{A}_F > t\cdot \log{n_{L-1}} ) < 2 \exp (- c t),
\end{equation}
where the probability is intended over $x_i\sim P_X$ and $c$ is a numerical constant. Thus, following the same argument of Lemma \ref{lemma:opnormlog}, the last relation implies that
\begin{equation}
    \opnorm{\Lambda/\norm{A}_F} = \bigO{\log^2 N \cdot \log n_{L-1}},
\end{equation}
with probability at least $1 - 2 \exp (-c \log^2 N)$, where $c$ is a numerical constant, and the probability is intended over $\{ x_i\}_{i=1}^N$.
}

\simone{
Thus, with probability $1 - C\exp(-n_{L-1}) - 2 \exp(-c \log ^2 n_{L-1}) - 2 \exp (-c \log^2 N)$ over $(W_k)_{k=1}^L$ and $(x_i)_{i=1}^N$, we have
\begin{equation}\label{eq:smallopnormlastcent}
    \opnorm{\frac{\Lambda \mathbf{1} \mathbf{1}^\top \Lambda}{\norm{A}_F^2}} \leq \opnorm{\mathbf{1} \mathbf{1}} \opnorm{\frac{\Lambda}{\norm{A}_F}}^2 =  \bigO{N\cdot \log^4 N \cdot \log^2 n_{L-1}} = o(n_{L-2}n_{L-1}),
\end{equation}
where the last equality is a consequence of Lemma \ref{lemma:alphabeta}.
}

\simone{
The desired result follows from merging \eqref{eq:smallopnormlastcent} with \eqref{eq:lastcentering}.}
\end{proof}

\section{Proofs for Part 2: Bounding the Centered Jacobian}\label{app:E}

\subsection{$\ell_2$ and sub-exponential norms of centered Jacobian}\label{app:bdrow}

We start by providing an upper bound on the quantity $\E_x \left[\tilde f_{L-2}(x) (D_L \tilde{\phi'}(g_{L-1}(x))) ^\top \right]$. This preliminary result will be useful when bounding the $\ell_2$ norm of the rows of the centered Jacobian.

\begin{lemma}\label{lemma:Asmall}
	Consider the setting of Theorem \ref{thm:main}, let $x \sim P_X$, and let $A$ be defined as
	\begin{equation}\label{eq:definitionAnoi}
		A = \E_x \left[\tilde f_{L-2}(x) (D_L \tilde{\phi'}(g_{L-1}(x))) ^\top \right].
	\end{equation}
	Then, we have
	\begin{equation}
		\norm{A}_F = \bigO{\sqrt{n_{L-1}}\log (n_{L-1})},
	\end{equation}
	with probability at least $1 - 2 \exp(-c \log^2 n_{L-1}) - C \exp(-c n_{L-1})$ over $(W_k)_{k=1}^L$, where $c$ is an absolute constant.
\end{lemma}
\begin{proof}
	We condition on $\norm{f_{L-2}(x)}_\Lip = \bigO{1}$ and on $\norm{\phi'(g_{L-1}(x))}_\Lip = \bigO{1}$. 	By Lemma \ref{lemma:lipschitzconst}, these two conditions hold with probability at least $1 - C' \exp(-c n_{L-1})$ over $(W_k)_{k=1}^{L-1}$. Then, as $P_X$ satisfies Assumption \ref{ass:data_dist2}, we have that 
	$\norm{\tilde f_{L-2}(x)}_{\psi_2}=\bigO{1}$ and  $\norm{\tilde{\phi'}(g_{L-1}(x))}_{\psi_2}=\bigO{1}$. Hence, 
	an application of Lemma \ref{lemma:opnormcovariance} gives that
	\begin{equation}\label{eq:bdopnm1}
	    \opnorm{\E_x \left[\tilde f_{L-2}(x) \tilde{\phi'}(g_{L-1}(x)) ^\top \right] }\le C_1,
	\end{equation}
	where $C_1$ is a numerical constant.
	
 The following chain of inequalities holds:
	\begin{equation}
		\begin{aligned}
			\norm{A}_F &\leq \norm{\E_x \left[\tilde f_{L-2}(x) \tilde{\phi'}(g_{L-1}(x)) ^\top \right] }_F \opnorm{D_L} \\
			&\leq \sqrt{n_{L-1}} \opnorm{\E_x \left[\tilde f_{L-2}(x) \tilde{\phi'}(g_{L-1}(x)) ^\top \right] } \opnorm{D_L} \\
			&\leq C_1 \sqrt{n_{L-1}} \opnorm{D_L} \\
			&= \bigO{\sqrt{n_{L-1}}\log (n_{L-1})},
		\end{aligned}
	\end{equation}
	where the third line uses \eqref{eq:bdopnm1}, and the last holds with probability $1 - 2 \exp(-c \log^2 n_{L-1})$ over $W_L$, because of Lemma \ref{lemma:opnormlog}. Taking into account the initial conditioning, we get the desired result.
\end{proof}

The next two results provide bounds on the $\ell_2$ norm and on the sub-exponential $\psi_1$ norm of the rows of $\cJ_{L-2}$, respectively.

\begin{lemma}[$\ell_2$ norm of rows of centered Jacobian]\label{lemma:l2normsJacobian}
	Consider the setting of Theorem \ref{thm:main}, let $x \sim P_X$, and let $\tilde J_x$ be defined as
	\begin{equation}\label{eq:defJx}
		\tilde J_x = \tilde f_{L-2}(x) \otimes (D_L \tilde{\phi'}(g_{L-1}(x))) - \E_{x} \left[ \tilde f_{L-2}(x) \otimes D_L \tilde{\phi'}(g_{L-1}(x)) \right].
	\end{equation}
	Then, we have
	\begin{equation}
		\norm{\tilde J_x}_2 = \Theta (\sqrt{n_{L-1}n_{L-2}}),
	\end{equation}
	with probability at least $1 - C \exp(-cn_{L-1}) - 12 \exp(-c \log^2 n_{L-1})$ over $x$ and $(W_k)_{k=1}^L$ and $x$.
\end{lemma}
\begin{proof}
	We have that
	\begin{equation}
		\begin{aligned}
			\norm{\tilde{J}_x}_2 =& \norm{\tilde f_{L-2}(x) \otimes (D_L \tilde{\phi'}(g_{L-1}(x))) - \E_{x} \left[ \tilde f_{L-2}(x) \otimes D_L \tilde{\phi'}(g_{L-1}(x)) \right]}_2 \\
			=& \norm{\tilde f_{L-2}(x)(D_L \tilde{\phi'}(g_{L-1}(x)))^\top - \E_{x} \left[ \tilde f_{L-2}(x) (D_L \tilde{\phi'}(g_{L-1}(x)))^\top \right]}_F \\
			=& \norm{\tilde f_{L-2}(x)^\top (D_L \tilde{\phi'}(g_{L-1}(x))) - A}_F,
		\end{aligned}
	\end{equation}
	where $A$ is defined in \eqref{eq:definitionAnoi}. The second equality is justified by the identity $\norm{u \otimes v}_2 = \norm{uv^\top}_F$ that holds for any vectors $u, v$. An application of the triangle inequality gives that
	\begin{equation}\label{eq:t11}
		\eta - \norm{A}_F \leq \norm{\tilde{J}_x}_2 \leq \eta + \norm{A}_F, \quad \mbox{ with }\eta = \norm{\tilde f_{L-2}(x)}_2 \norm{D_L \tilde{\phi'}(g_{L-1}(x))}_2.
	\end{equation}
	Lemma \ref{lemma:concnormcentered} gives that
	\begin{equation}\label{eq:con1}
		\norm{\tilde f_{L-2}(x)}_2 = \norm{f_{L-2}(x) - \E_x \left[f_{L-2}(x)\right] }_2  = \Theta(\sqrt{n_{L-2}}),
	\end{equation}
with probability at least $1 - C' \exp(-c n_{L-1})$ over $(W_k)_{k=1}^{L}$ and $x$. Furthermore,
	Lemma \ref{lemma:concnormcenteredwithD} gives that 
	\begin{equation}\label{eq:con2}
		\norm{D_L \tilde{\phi'}(g_{L-1}(x))}_2 = \norm{D_L \left( \phi'(g_{L-1}(x)) - \E_x \left[ \phi'(g_{L-1}(x)) \right] \right) }_2  = \Theta(\sqrt{n_{L-1}}),
	\end{equation}
	with probability at least  $1 - 10 \exp(-c \log^2 n_{L-1})- C' \exp(-cn_{L-1})$ over $x$ and $(W_k)_{k=1}^L$. By combining \eqref{eq:t11}, \eqref{eq:con1}, \eqref{eq:con2} and the bound on $\norm{A}_F$ provided by Lemma \ref{lemma:Asmall}, we conclude that 
	\begin{equation}
		 \norm{\tilde{J}_x}_2 = \Theta(\sqrt{n_{L-1}n_{L-2}}),
	\end{equation}
	with probability at least
	\begin{equation}
	\begin{aligned}
	    &1 - C' \exp(-c n_{L-1}) - 10 \exp(-c \log^2 n_{L-1}) \\
	    & \hspace{1cm} - C' \exp(-cn_{L-1}) - 2 \exp(-c \log^2 n_{L-1}) - C' \exp(-c n_{L-1})\\
	    & \geq 1 - C \exp(-cn_{L-1}) - 12 \exp(-c \log^2 n_{L-1}),
	\end{aligned}
	\end{equation}
	over $x$ and $(W_k)_{k=1}^L$, which gives the desired result.
\end{proof}

\begin{lemma}[Sub-exponential norm of rows of centered Jacobian]\label{lemma:psi1normacobian}
	Consider the setting of Theorem \ref{thm:main}, let $x \sim P_X$, and let $\tilde J_x$ be defined as in \eqref{eq:defJx}. Fix a realization of $(W_k)_{k=1}^L$. Then, with probability at least $1 - 2 \exp(-c \log^2 n_{L-1}) - C \exp(-c n_{L-1})$ over this realization ($c$ being a numerical constant), we have that
	\begin{equation}
		\subEnorm{\tilde J_x} = \bigO{\log{n_{L-1}}}.
	\end{equation}
\end{lemma}
\begin{proof}
	We condition on $\norm{f_{L-2}(x)}_\Lip = \bigO{1}$ and on $\norm{\phi'(g_{L-1}(x))}_\Lip = \bigO{1}$. By Lemma \ref{lemma:lipschitzconst}, these two conditions hold with probability at least $1 - C \exp(-c n_{L-1})$ over $(W_k)_{k=1}^{L-1}$. Then, we have
	\begin{equation}
		\begin{aligned}
			\subEnorm{\tilde J_x} &= \sup_{u \text{ s.t. } \norm{u}_2 = 1} \subEnorm{u^\top \tilde J_x}  \\
			&= \sup_{U \text{ s.t. } \norm{U}_F = 1} \subEnorm{\tilde f_{L-2}(x) U (D_L \tilde{\phi'}(g_{L-1}(x))) - \E_{x} \left[ \tilde f_{L-2}(x) U D_L \tilde{\phi'}(g_{L-1}(x)) \right]} \\
			&\leq C_0 \sup_{U \text{ s.t. } \norm{U}_F = 1} \norm{UD_L}_F \\
			&\leq C_0 \opnorm{D_L}\\
			&\leq C_0 \log{n_{L-1}},
		\end{aligned}
	\end{equation}
	where the third line follows from Lemma \ref{lemma:HW} and the last inequality holds
	with probability at least $1 - 2 \exp(-c \log^2 n_{L-1})$ over $W_L$ by Lemma \ref{lemma:opnormlog}. Taking into account the initial conditioning, we get the desired result.
\end{proof}

\subsection{Proof of Proposition \ref{cor:hammer}}\label{app:hammproof}

\begin{proof}[Proof of Proposition \ref{cor:hammer}]
    Following the notation in \cite{hammer}, we define
    \begin{equation}
        B := \sup_{z\in \mathbb R^N : \norm{z}_2=1} \left| \norm{\sum_{i=1}^N z_i \cJ_{i:}}_2^2 - \sum_{i=1}^N z_i^2 \norm{\cJ_{i:}}_2^2 \right|^\frac{1}{2}.
    \end{equation}
    Then, for any $z\in \mathbb R^N$ with unit norm, we have that
    \begin{equation}
    \begin{split}
        \norm{\cJ z}_2^2 &= \norm{\sum_{i=1}^N z_i \cJ_{i:}}_2^2-\sum_{i=1}^N z_i^2 \norm{\cJ_{i:}}_2^2 + \sum_{i=1}^N z_i^2 \norm{\cJ_{i:}}_2^2  \ge \min_i\norm{\cJ_{i:}}_2^2 - B^2,
    \end{split}
    \end{equation}
    which implies that
    \begin{equation}\label{eq:Beig}
        \evmin{\cJ \cJ^\top}= \inf_{z\in \mathbb R^N : \norm{z}_2=1}\norm{\cJ z}_2^2\geq \min_{i} \norm{\cJ_{i:}}_2^2 - B^2.
    \end{equation}
    In our case, $\cJ_{i:} \in \R ^{n_{L-2}n_{L-1}}$. Notice that this dimension is indicated with $n$ in Theorem 3.2 of \cite{hammer}. In the statement of the mentioned Theorem, let's fix $r=1$, $m=N$, and $\theta = (N / (n_{L-1}n_{L-2}))^{1/4} < 1/4$. Then, we have that the condition required to apply Theorem 3.2 is satisfied, \ie
    \begin{equation}
        N \log^2 \left(2 \sqrt[4]{\frac{n_{L-2}n_{L-1}}{N}}\right) \leq \sqrt{N n_{L-2}n_{L-1}},
    \end{equation}
    where the inequality follows from Assumption \ref{ass:overparam}. By combining \eqref{eq:Beig} with the upper bound on $B$ given by Theorem 3.2 of \cite{hammer}, the desired result readily follows.
\end{proof}

\subsection{Proof of Theorem \ref{thm:centered}}\label{app:pfce}

\begin{proof}[Proof of Theorem \ref{thm:centered}]
By Lemma \ref{lemma:psi1normacobian}, we have that, with probability at least $1 - 2 \exp(-c \log^2 n_{L-1}) - C' \exp(-c n_{L-1})$ over $(W_k)_{k=1}^L$, the rows of $\cJ$ are sub-exponential (with respect to the randomness in $(x_i)_{i=1}^N$). In particular, we have that
\begin{equation}\label{eq:pnormoverx0}
    \psi := \max_i \subEnorm{\cJ_{i:}} \leq C_1 \log n_{L-1}.
\end{equation}
Furthermore, by Lemma \ref{lemma:l2normsJacobian}, we have that
\begin{equation}
    \norm{\tilde{J}_{i:}}_2 = \Theta(\sqrt{n_{L-2} n_{L-1}}),
\end{equation}
with probability at least $1 - p$ over $x_i$ and $(W_k)_{k=1}^L$, where to ease the notation we have defined $p := C' \exp(-c_0 n_{L-1}) + 12 \exp(-c_0 \log^2 n_{L-1})$. Hence, with probability at least $1 - \sqrt{p}$ over $(W_k)_{k=1}^L$, we have that
\begin{equation}\label{eq:pnormoverx}
	\mathbb P_{x_i} \left( c_1 \sqrt{n_{L-2} n_{L-1}} \le \norm{\tilde{J}_{i:}}_2 \le c_2 \sqrt{n_{L-2} n_{L-1}} \right) \geq 1 - \sqrt{p},
\end{equation}
for some numerical constants $c_2 > c_1 > 0$. In \eqref{eq:pnormoverx}, we use the symbol $\mathbb P_{x_i}$ to highlight that this probability is taken over $x_i$. For the rest of the argument, we condition on a realization of $(W_k)_{k=1}^L$ s.t. \eqref{eq:pnormoverx0} and \eqref{eq:pnormoverx} hold. Then, by performing a union bound over the samples, we have that
\begin{equation}
    \eta_{\textup{min}} = \min_i \norm{\cJ_{i:}}_2 \geq c_1 \sqrt{n_{L-2} n_{L-1}},
\end{equation}
and
\begin{equation}\label{eq:boundetamax}
    \eta_{\textup{max}} = \max_i \norm{\cJ_{i:}}_2 \leq c_2 \sqrt{n_{L-2} n_{L-1}},
\end{equation}
with probability at least $1 - N \sqrt{p}$ over $(x_i)_{i=1}^N$.

\simone{Next, we apply Proposition \ref{cor:hammer} with $K = 1$, $K' = c_2$ and
\begin{equation}
\begin{aligned}
    \Delta &= C_1 (\psi K + K')^2 N^{1/4} (n_{L-1}n_{L-2})^{3/4}  \\
    &\leq C_2 \log^2 n_{L-1} N^{1/4} (n_{L-1}n_{L-2})^{3/4} \\
    &= o(n_{L-1}n_{L-2}).
\end{aligned}
\end{equation}
Note that \eqref{eq:lemmarel1} in Lemma \ref{lemma:alphabeta} gives that $N^{1/4}\cdot \log^2 n_{L-1}=o((n_{L-1}n_{L-2})^{1/4})$, which justifies the last line.} Thus, \eqref{eq:hm} implies that 
\begin{equation}
    \evmin{\cJ\cJ^\top} \geq \eta_{\textup{min}}^2 - \Delta \geq c_1 n_{L-2} n_{L-1} - o(n_{L-1}n_{L-2}) = \Theta(n_{L-2} n_{L-1}),
\end{equation}
with probability at least
\begin{equation}
\begin{aligned}
    & 1 -  \exp \left( -cK \sqrt{N} \log \left( \frac{2 n_{L-1}n_{L-2}}{N} \right) \right) - \mathbb{P}\left(\eta_{\text{max}} \geq K' \sqrt{n_{L-1}n_{L-2}}\right) \\
    &\hspace{2cm} \geq  1 -  \exp \left( -c \sqrt{N} \right) - \mathbb{P}\left(\eta_{\text{max}} \geq c_2 \sqrt{n_{L-1}n_{L-2}}\right) \\
    &\hspace{2cm} \geq  1 -  \exp \left( -c \sqrt{N} \right) - N \sqrt{p},
\end{aligned}
\end{equation}
where the last inequality follows from \eqref{eq:boundetamax}. By taking into account the conditioning over $(W_k)_{k=1}^L$ made in order to guarantee \eqref{eq:pnormoverx0} and \eqref{eq:pnormoverx}, we conclude that $\evmin{\cJ \cJ^\top} = \Omega(n_{L-1}n_{L-2})$ with probability at least
\begin{equation}
\begin{aligned}
    & 1 -  \exp \left( -c \sqrt{N} \right) - N \sqrt{p} - \sqrt{p} - 2 \exp(-c \log^2 n_{L-1}) - C' \exp(-c n_{L-1})\\
    & \hspace{0.25cm} = 1 -  \exp \left( -c \sqrt{N} \right) - (N+1) \sqrt{C' \exp(-c_0n_{L-1}) + 12 \exp(-c_0 \log^2 n_{L-1})}\\
        &\hspace{6.5cm}- 2 \exp(-c \log^2 n_{L-1}) - C' \exp(-c n_{L-1})\\
         & \hspace{.25cm} \geq 1 - \exp \left( -c \sqrt{N} \right) - (N+1) \left( \sqrt{C'} \exp(-c_0 n_{L-1} / 2) + \sqrt{12} \exp(-c_0 \log^2 n_{L-1}/ 2)\right) \\
        & \hspace{6.5cm}  -  2 \exp(-c \log^2 n_{L-1}) - C' \exp(-c n_{L-1})\\
    & \hspace{.25cm}\geq 1 - \exp \left( -c \sqrt{N} \right) - C''N \exp(-c_1 n_{L-1}) - C''N \exp(-c \log^2 n_{L-1}),
 \end{aligned}
\end{equation}
over $(x_i)_{i=1}^N$ and $(W_k)_{k=1}^L$, which gives the desired result.
\end{proof}

\section{Proof of the Upper Bound \ref{eq:ubNTK}}\label{app:pfub}

Before giving the proof of the upper bound \ref{eq:ubNTK}, we provide again its statement for the reader's convenience. 

\begin{lemma}[Upper bound on the smallest NTK eigenvalue]
    Consider the setting of Theorem \ref{thm:main},
    and let $K$ be the NTK Gram matrix \eqref{eq:NTKgramdef}.
Then, we have 
	\begin{equation}
		\evmin{K} = \bigO{d n_{L-1}},
	\end{equation}
	with probability at least
	$1 - C \exp(-c n_{L-1})$ over $(x_i)_{i=1}^N$ and $(W_k)_{k=1}^L$, where $c$ and $C$ are numerical constants.
\end{lemma}
\begin{proof}
By using the expression in \eqref{eq:sumterms}, we have that
\begin{equation}\label{eq:initineq}
    \evmin{K} = \evmin{JJ^\top} \leq (J J^\top)_{11} = \sum_{l=0}^{L-1} \norm{(F_l)_{1:}}_2^2 \norm{(B_{l+1})_{1:}}_2^2.
\end{equation}
An application of Lemma \ref{lemma:concnorm} gives that 
\begin{equation}\label{eq:upboundF}
    \norm{(F_l)_{1:}}_2^2 = \norm{f_l(x_1)}_2^2 = \Theta(n_l),
\end{equation}
with probability at least $1 - C' \exp(-c n_{L-1})$ over $(W_k)_{k=1}^l$ and $x_1$. We condition on the event such that \eqref{eq:upboundF} holds for all $l\in \{0, \ldots, L-1\}$. This happens with probability at least $1 - C'' \exp(-c n_{L-1})$ over $(W_k)_{k=1}^L$ and $x_1$.

By definition, we have that $\norm{B_L}_2 = 1$ and that, for $l \in [L-1]$,
\begin{equation}
    \norm{(B_{l})_{1:}}_2^2 = \norm{\prod_{k=l}^{L-1} \Sigma_k(x_1)W_{k+1}}_2^2.
\end{equation}

Since $\Sigma_k(x_1)=\diag([\phi'(g_{k,j}(x_1))]_{j=1}^{n_k})$, by Assumption \ref{ass:activationfunc}, we have that
\begin{equation}
    \opnorm{\Sigma_k(x_1)} \leq M. 
\end{equation}

Let us now condition on the following two events: \emph{(i)} $\opnorm{W_k} = \Theta(1)$, for all $k \in [L-1]$ (this happens with probability at least $1 - C' \exp (-c n_{L-1})$ over $(W_k)_{k=1}^{L-1}$, see \eqref{eq:opWknm} in the proof of Lemma \ref{lemma:lipschitzconst}), and \emph{(ii)} $\norm{W_L}_2 = \Theta(\sqrt{n_{L-1}})$ (this happens with probability at least $1 - \exp (-c n_{L-1})$ over $W_L$, by Theorem 3.1.1 in \cite{vershynin2018high}). Then, we readily get
\begin{equation}
    \norm{(B_{l})_{1:}}_2^2 = \bigO{n_{L-1}}.
\end{equation}
Taking the intersection of all the events over which we conditioned, we finally obtain
\begin{equation}\label{eq:fineq}
    \sum_{l=0}^{L-1} \norm{(F_l)_{1:}}_2^2 \norm{(B_{l+1})_{1:}}_2^2 = \bigO{n_{L-1} \sum_{l=0}^{L-1} n_l} = \bigO{d n_{L-1}},
\end{equation}
with probability at least $1 - (1 + C' + C'') \exp(-c n_{L-1})$ over $(W_k)_{k=1}^L$ and $x_1$, where in the last step we have used Assumption \ref{ass:topology}. By combining \eqref{eq:initineq} and \eqref{eq:fineq}, the desired result follows.
\end{proof}

\section{Proof of Corollary \ref{cor:memcap}}\label{app:newmem}

\begin{proof}[Proof of Corollary \ref{cor:memcap}]
    By Theorem \ref{thm:main}, we have that the smallest eigenvalue of $JJ^\top$ is bounded away from zero with probability at least $1 - p$ over $(x_i)_{i=1}^N$ and $(W_k)_{k=1}^L$, where
    \begin{equation}
        p:=C\,N e^{-c \log^2 n_{L-1}} - C e^{-c \log^2 N}.
    \end{equation}
    Hence, with probability at least $1 - p$ over $(x_i)_{i=1}^N$, there exists a set of parameters $\theta_0$ such that $J(\theta_0)$ has a right inverse. Thus, for all $Y \in \R^N$, there exists $\theta'$ such that
	\begin{equation}
		J(\theta_0) \theta' = \frac{\partial F_L(\theta)}{\partial \theta}\bigg\rvert_{\theta = \theta_0} \theta' = Y.
	\end{equation}
	This can also be written, for all $i \in [N]$, as
	\begin{equation}
		y_i = \frac{\partial f_L(\theta, x_i)}{\partial \theta}\bigg\rvert_{\theta = \theta_0}^\top \theta' = \lim_{h \to 0} \frac{f_L(\theta_0 + h \theta', x_i) - f_L(\theta_0, x_i)}{h}.
	\end{equation}
	Then, for all $\varepsilon>0$, there exists $h^*$ such that, for all $i \in [N]$, 
	\begin{equation}
		|y_i - f^*(x_i)| \le \frac{\varepsilon}{\sqrt{N}},
	\end{equation}
	where
	\begin{equation}
		f^*(x_i) := \frac{f_L(\theta_0 + h^* \theta', x_i) - f_L(\theta_0, x_i)}{h^*}.
	\end{equation}
	Finally, the desired result follows by noticing that $f^*$ can be implemented by a network with the same depth and twice more neurons at every hidden layer.
\end{proof}

\section{Proof of Theorem \ref{thm:optimization}}\label{app:pfopt}

\paragraph{Notation for this appendix.} In this appendix, we use $J(\theta)$ to denote the Jacobian of the network output $F_L$, evaluated in $\theta$. We recall that $J(\theta)$ is a matrix with $N$ rows and $\sum_{l=0}^{L-3} n_l n_{l+1}   +   2 n_{L-2}n_{L-1} + 2n_{L-1}$ columns (for the optimization result, we assume that the $(L-1)$-th layer has an even number of neurons and denote its width as $2 n_{L-1}$). Let $K(\theta) = J(\theta) (J(\theta))^\top$ be the associated empirical NTK Gram matrix, and let $\theta_0$ be the initialization defined in \eqref{eq:theta0}. We also make the dependence on $\theta$ explicit for feature vectors and backpropagation terms: the feature vector at the $l$-th layer with input $x_i$ and network parameter $\theta$ is denoted by $f_l (\theta, x_i)$, and the corresponding backpropagation term is denoted by $b_l(\theta, x_i)$, where
$b_l(\theta, x_i) = (B_l(\theta))_{i:}$. Finally, we use $W_l(\theta)$ to denote the weights of the $l$-th layer evaluated at the parameter $\theta$.


A straightforward application of Theorem 2.1 in \cite{oymak2019overparameterized} gives the following proposition. 

\begin{proposition}\label{prop:optim}
 Consider solving the least-squares optimization problem
    \begin{equation}\label{eq:loss2}
        \min_\theta \mathcal L(\theta) := \frac{1}{2} \min_\theta \norm{F_L(\theta) - Y}_2^2,
    \end{equation}
    by running gradient descent updates of the form 
    $\theta_{t+1} = \theta_t - \eta \nabla \mathcal L(\theta_t)$,
    with some initialization $\tilde\theta_0$. Assume there exists $\alpha, \beta\in\mathbb R$, such that, if we define $\mathcal D = \mathcal B(\tilde\theta_0, R)$ as the $\ell_2$ ball centered in $\tilde\theta_0$ with radius $R$, with
\begin{equation}
    R := \frac{4 \norm{F_L(\tilde\theta_0) - Y}_2}{\alpha},
\end{equation}
the following holds
\begin{equation}\label{eq:optass1}
    \forall \theta \in \mathcal D: \; \alpha \leq \sigma_{\textup{min}}(J(\theta)) \leq \opnorm{J(\theta)} \leq \beta, 
\end{equation}
\begin{equation}\label{eq:optass2}
    \forall \theta_1, \theta_2 \in \mathcal D: \; \opnorm{J(\theta_1) - J(\theta_2)} \leq \frac{\alpha^2}{2 \beta}.
\end{equation}
Then, by setting $\eta \leq 1/ (2 \beta^2)$, we have that, for all $t\ge 1$,
\begin{equation}
    \mathcal L(\theta_t) \leq \left( 1 -  \frac{\eta \alpha^2}{2}\right)^t \mathcal L(\tilde\theta_0).
\end{equation}
\end{proposition}

In order to apply this proposition with initialization $\tilde\theta_0=\theta_0$, we need to prove that the necessary assumptions hold. We will do so by showing the following intermediate results:

\begin{itemize}
    \item Lemma \ref{lemma:firstlemmaopt} shows that, at the initial point $\theta_0$, the network output is $0$ and the smallest NTK eigenvalue is lower bounded.
    
    \item Lemma \ref{lemma:optW} gives a tight estimate on the operator norm of the weights inside a ball $\mathcal D$ centered at $\theta_0$ and with radius $R=o(1)$.
    
    \item Lemma \ref{lemma:featuresinball1} gives an upper bound on the distance between a feature vector in $\mathcal D$ and the feature vector at $\theta_0$.
    
    \item Lemma \ref{lemma:featuresinball2} gives upper bounds on the $\ell_2$ norm and the $\ell_2$ distance between feature vectors in $\mathcal D$.
    
    \item Lemmas \ref{lemma:binball} and \ref{lemma:binball2} give upper bounds on the $\ell_2$ norm and the $\ell_2$ distance of backpropagation terms in $\mathcal D$, respectively.
    
    \item Lemma \ref{lemma:jacobianinball} gives an upper bound on the difference in operator norm between Jacobians in $\mathcal D$.
    
    \item Finally, Lemma \ref{lemma:opnormsigminball} gives upper and lower bounds on the NTK spectrum in $\mathcal D$. 
\end{itemize}

\begin{lemma}[Network output and smallest NTK eigenvalue at initialization]\label{lemma:firstlemmaopt}
   Let $\theta_0$ be defined in \eqref{eq:theta0}. Then, we have that, for all $x \in \R^d$,
    \begin{equation}\label{eq:fl0}
        f_{L}(x, \theta_0) = 0.
    \end{equation}
    Furthermore, we have that
    \begin{equation}\label{eq:J00}
        \sigma_{\textup{min}}(J(\theta_0)) \ge c_1 \sqrt{\gamma n_{L-2} n_{L-1}}),
    \end{equation}
    with probability at least $1 - C\,N e^{-c \log^2 n_{L-1}} - C e^{-c \log^2 N}$ over $(x_i)_{i=1}^N\sim_{\rm i.i.d.}P_X$ and $\theta_0$, where $c,c_1$ and $C$ are numerical constants.
\end{lemma}

\begin{proof}
By definition \eqref{eq:theta0} of the initialization $\theta_0$, we have that
\begin{equation}
\begin{aligned}
    f_{L}(x, \theta_0) &= (W_{L}^{(1)}(\theta_0))^\top \phi((W_{L-1}^{(1)}(\theta_0))^\top f_{L-2}(\theta_0, x)) \\
    & \hspace{0.5cm} + (W_{L}^{(2)}(\theta_0))^\top \phi((W_{L-1}^{(2)}(\theta_0))^\top f_{L-2}(\theta_0, x))\\
    &= (W_{L}^{(1)}(\theta_0))^\top \phi((W_{L-1}^{(1)}(\theta_0))^\top f_{L-2}(\theta_0, x)) \\
    & \hspace{0.5cm} + (- W_{L}^{(1)}(\theta_0))^\top \phi((W_{L-1}^{(1)}(\theta_0))^\top f_{L-2}(\theta_0, x))\\
    &= 0,
\end{aligned}
\end{equation}
where in the second equality we use that $W_{L-1}^{(2)}(\theta_0)=W_{L-1}^{(1)}(\theta_0)$ and that $W_{L}^{(2)}(\theta_0)=-W_{L}^{(1)}(\theta_0)$. This proves \eqref{eq:fl0}. 

Let us now compute the Jacobian at initialization $J(\theta_0)$. For $l \in [L-2]$, we have that 
\begin{equation}
\begin{aligned}
    \frac{\partial f_L(x)}{\partial (W_l)_{ij}} \bigg\rvert_{\theta = \theta_0} \hspace{-1.75em} &= (W_{L}^{(1)}(\theta_0))^\top \hspace{-.4em} \left(\phi'\left((W_{L-1}^{(1)}(\theta_0))^\top f_{L-2}(\theta_0, x)\right) \left((W_{L-1}^{(1)}(\theta_0))^\top \frac{\partial f_{L-2}(\theta, x)}{\partial (W_l)_{ij}}\bigg\rvert_{\theta = \theta_0} \right) \right) \\
    &\hspace{-1cm} + (W_{L}^{(2)}(\theta_0))^\top \left(\phi'\left((W_{L-1}^{(2)}(\theta_0))^\top f_{L-2}(\theta_0, x)\right) \left((W_{L-1}^{(2)}(\theta_0))^\top \frac{\partial f_{L-2}(\theta, x)}{\partial (W_l)_{ij}}\bigg\rvert_{\theta = \theta_0} \right) \right) \\
    &= (W_{L}^{(1)}(\theta_0))^\top\hspace{-.4em} \left(\phi'\left((W_{L-1}^{(1)}(\theta_0))^\top f_{L-2}(\theta_0, x)\right) \left((W_{L-1}^{(1)}(\theta_0))^\top \frac{\partial f_{L-2}(\theta, x)}{\partial (W_l)_{ij}}\bigg\rvert_{\theta = \theta_0} \right) \right) \\
    &\hspace{-1cm} - (W_{L}^{(1)}(\theta_0))^\top \left(\phi'\left((W_{L-1}^{(1)}(\theta_0))^\top f_{L-2}(\theta_0, x)\right) \left((W_{L-1}^{(1)}(\theta_0))^\top \frac{\partial f_{L-2}(\theta, x)}{\partial (W_l)_{ij}}\bigg\rvert_{\theta = \theta_0} \right) \right) \\
    &= 0,
\end{aligned}
\end{equation}
where in the second equality we use again that $W_{L-1}^{(2)}(\theta_0)=W_{L-1}^{(1)}(\theta_0)$ and that $W_{L}^{(2)}(\theta_0)=-W_{L}^{(1)}(\theta_0)$.

Let us define $f_{L-1}^{(k)}(\theta, x) :=\phi( (W_{L-1}^{(k)}(\theta))^\top f_{L-2}(\theta, x))$ for $k\in\{1, 2\}$. Then, for the $(L-1)$-th layer, by isolating the computation over $W_{L-1}^{(1)}$, we have that
\begin{equation}
\begin{aligned}
    \frac{\partial f_L(\theta, x)}{\partial (W_{L-1}^{(1)})_{ij}} \bigg\rvert_{\theta = \theta_0} &= (W_{L}^{(1)}(\theta_0))^\top \frac{\partial f_{L-1}^{(1)}(\theta, x)}{\partial (W_{L-1}^{(1)})_{ij}} \bigg\rvert_{\theta = \theta_0} + (W_{L}^{(2)}(\theta_0))^\top \frac{\partial f_{L-1}^{(2)}(\theta, x)}{\partial (W_{L-1}^{(1)})_{ij}} \bigg\rvert_{\theta = \theta_0} \\
    &= (W_{L}^{(1)}(\theta_0))^\top \frac{\partial f_{L-1}^{(1)}(\theta, x)}{\partial (W_{L-1}^{(1)})_{ij}} \bigg\rvert_{\theta = \theta_0} \\
    &=: J^{(1)}(\theta_0),
\end{aligned}
\end{equation}
where we use that $f_{L-1}^{(2)}(\theta, x)$ does not depend on the parameters $W_{L-1}^{(1)}$. 
Proceeding in the same way and using that $W_{L-1}^{(2)}(\theta_0)=W_{L-1}^{(1)}(\theta_0)$ and $W_{L}^{(2)}(\theta_0)=-W_{L}^{(1)}(\theta_0)$, we also obtain that 
\begin{equation}
    \frac{\partial f_L(\theta, x)}{\partial (W_{L-1}^{(2)})_{ij}} \bigg\rvert_{\theta = \theta_0} = - J^{(1)}(\theta_0).
\end{equation}

Finally, by observing that $f_L(\theta, x) = (W_{L}^{(1)})^\top f_{L-1}^{(1)}(\theta, x) + (W_{L}^{(2)})^\top f_{L-1}^{(2)}(\theta, x)$, we deduce 
\begin{equation}
    \frac{\partial f_L(\theta, x)}{\partial (W_{L}^{(k)})_{i}} \bigg\rvert_{\theta = \theta_0} = \left(f_{L-1}^{(k)}(\theta_0, x)\right)_i,\quad \mbox{for }k\in\{1, 2\}.
\end{equation}
Hence, the NTK at initialization $K(\theta_0)$ can be expressed as
\begin{equation}
\begin{aligned}
    K(\theta_0) &= J^{(1)}(\theta_0) (J^{(1)}(\theta_0))^\top + J^{(1)}(\theta_0) (J^{(1)}(\theta_0))^\top \\
    & \hspace{0.5cm} + F_{L-1}^{(1)}(\theta_0) (F_{L-1}^{(1)}(\theta_0))^\top + F_{L-1}^{(2)}(\theta_0) (F_{L-1}^{(2)}(\theta_0))^\top.
\end{aligned}
\end{equation}

Note that, by construction, $J^{(1)}(\theta_0)$ has the same distribution of $J_{L-2}$, whose rows are given by \eqref{eq:Jacobian}. Therefore, by combining the results from Theorem \ref{thm:maincentering} and Theorem \ref{thm:centered}, we conclude that 
\begin{equation}
    \sigma_{\text{min}}(J(\theta_0)) \geq 2 \sqrt{\gamma} \sigma_{\text{min}}(J_{L-2}) \geq c_1 \sqrt{\gamma n_{L-2} n_{L-1}},
\end{equation}
with probability at least
	$1 - C\,N e^{-c \log^2 n_{L-1}} - C e^{-c \log^2 N}
	$ over $(x_i)_{i=1}^N$ and $\theta_0$, where $c_1, C$ are numerical constants. This proves \eqref{eq:J00}, and concludes the proof of the lemma. 
\end{proof}

\begin{lemma}[Operator norm of weights in $\mathcal D$]\label{lemma:optW}
Let $\theta_0$ be defined in \eqref{eq:theta0}, let $\mathcal D = \mathcal B(\theta_0, R)$ and assume that $R = o(1)$. Then, for any $l \in [L-1]$,
    \begin{equation}
        \sup_{\theta\in\mathcal D}\opnorm{W_l(\theta)} = \bigO{1},
    \end{equation}
    with probability at least $1 - 2 \exp(-c n_{L-1})$ over $W_l(\theta_0)$.
\end{lemma}
\begin{proof}
By Weyl's theorem, we have that, for all $l \in [L-2]$,
\begin{equation}
\begin{aligned}
    \sup_{\theta\in\mathcal D}\opnorm{W_l(\theta)} &\leq \opnorm{W_l(\theta_0)} + \sup_{\theta\in\mathcal D}\opnorm{W_l(\theta) - W_l(\theta_0)} \\
    &\leq \opnorm{W_l(\theta_0)} + \sup_{\theta\in\mathcal D}\norm{W_l(\theta) - W_l(\theta_0)}_F \\
    &\leq \opnorm{W_l(\theta_0)} + \sup_{\theta\in\mathcal D}\norm{\theta - \theta_0}_2 \\
    &= \opnorm{W_l(\theta_0)} + o(1)\\
    &= \bigO{1},
\end{aligned}
\end{equation}
where in the fourth line we use that $R=o(1)$, and the result of the last line holds with probability at least $1 - 2 \exp(-c n_{L-1})$ over $W_l(\theta_0)$ by Theorem 4.4.5 of \cite{vershynin2018high}. By following the same argument, we have that, with probability at least $1 - 2 \exp(-c n_{L-1})$ over $W_{L-1}(\theta_0)$,
\begin{equation}
\sup_{\theta\in\mathcal D}\opnorm{W_{L-1}^{(k)}(\theta)}= \bigO{1}, \quad \mbox{for } k\in\{1, 2\},
\end{equation}
which readily implies that $\sup_{\theta\in\mathcal D}\opnorm{W_{L-1}(\theta)}=\bigO{1}$ and concludes the proof.
\end{proof}

\begin{lemma}[Distance of features in $\mathcal D$ from initialization]\label{lemma:featuresinball1}
Let $\theta_0$ be defined in \eqref{eq:theta0}, $x \sim P_X$, $\mathcal D = \mathcal B(\theta_0, R)$ and assume that $R = o(1)$. Then, for any $0 \leq l \leq L-1$, we have 
    \begin{equation}\label{eq:featb}
        \sup_{\theta\in\mathcal D}\norm{f_l(\theta, x) - f_l(\theta_0, x)}_2 \le C\,R\, \sqrt{d},
    \end{equation}
    with probability at least $1 - C \exp (-c n_{L-1})$ over $(W_k(\theta_0))_{k=1}^l$ and $x$, where $c, C$ are numerical constants.
\end{lemma}
\begin{proof}
We prove the claim by induction over $l$. For the base case, we have $f_0(\theta, x) = f_0(\theta_0, x)$, hence \eqref{eq:featb} holds with probability 1.

For the induction case, let $l > 0$. Then,
\begin{equation}\label{eq:lip00}
\begin{aligned}
\sup_{\theta\in\mathcal D}\norm{f_l(\theta, x) - f_l(\theta_0, x)}_2 &= \sup_{\theta\in\mathcal D}\norm{\phi\left( (W_l(\theta))^\top f_{l-1}(\theta, x )\right) - \phi\left( (W_l(\theta_0))^\top f_{l-1}(\theta_0, x)\right)}_2 \\
&\leq M\sup_{\theta\in\mathcal D}\norm{ (W_l(\theta))^\top f_{l-1}(\theta, x)-  (W_l(\theta_0))^\top f_{l-1}(\theta_0, x)}_2 \\
&\leq M\sup_{\theta\in\mathcal D}\norm{ (W_l(\theta))^\top f_{l-1}(\theta, x) -  (W_l(\theta))^\top f_{l-1}(\theta_0, x)}_2 \\
& \hspace{0.5cm} + M\sup_{\theta\in\mathcal D}\norm{ (W_l(\theta))^\top f_{l-1}(\theta_0, x)-  (W_l(\theta_0))^\top f_{l-1}(\theta_0, x)}_2\\
&\leq M\sup_{\theta\in\mathcal D}\opnorm{W_l(\theta)} \sup_{\theta\in\mathcal D}\norm{f_{l-1}(\theta, x) -  f_{l-1}(\theta_0, x)}_2 \\
& \hspace{0.5cm} + M\sup_{\theta\in\mathcal D}\opnorm{W_l(\theta) - W_l(\theta_0)} \norm{f_{l-1}(\theta_0, x)}_2.
\end{aligned}
\end{equation}
By Lemma \ref{lemma:optW}, we have that
\begin{equation}\label{eq:lip0}
\sup_{\theta\in\mathcal D}\opnorm{W_l(\theta)} = \bigO{1},    
\end{equation}
with probability at least $1 - 2 \exp(-c n_{L-1})$ over $W_l(\theta_0)$. By inductive hypothesis, we have
\begin{equation}\label{eq:lip1}
 \sup_{\theta\in\mathcal D}\norm{f_{l-1}(\theta, x) -  f_{l-1}(\theta_0, x)}_2 \le C \,R\,\sqrt{d},   
\end{equation}
with probability at least $1 - C\exp(-c n_{L-1})$ over $(W_k(\theta_0))_{k=1}^{l-1}$ and $x$. Clearly, we also have that
\begin{equation}\label{eq:lip2}
 \sup_{\theta\in\mathcal D}\opnorm{W_l(\theta) - W_l(\theta_0)} \leq \sup_{\theta\in\mathcal D}\norm{W_l(\theta) - W_l(\theta_0)}_F \leq \sup_{\theta\in\mathcal D}\norm{\theta - \theta_0} \le R.  
\end{equation}
Furthermore, an application of Lemma \ref{lemma:concnorm} gives that
\begin{equation}\label{eq:lip3}
    \norm{f_{l-1}(\theta_0, x)}_2 = \Theta(\sqrt{n_{l-1}}) = \bigO{\sqrt{d}},
\end{equation}
with probability at least $1 - C \exp(-c n_{L-1})$ over $(W_k(\theta_0))_{k=1}^{l-1}$ and $x$. 
By combining \eqref{eq:lip00}, \eqref{eq:lip0}, \eqref{eq:lip1}, \eqref{eq:lip2} and \eqref{eq:lip3}, we obtain that
\begin{equation}
\begin{aligned}
 \sup_{\theta\in\mathcal D}\norm{f_l(\theta, x) - f_l(\theta_0, x)}_2 &\leq C\,R\sqrt{d},
\end{aligned}
\end{equation}
with probability at least $1 - C \exp (-c n_{L-1})$ over $(W_k(\theta_0))_{k=1}^l$ and $x$, which completes the proof.
\end{proof}

\begin{lemma}[$\ell_2$ norm and $\ell_2$ distance of features in $\mathcal D$]\label{lemma:featuresinball2}
Let $\theta_0$ be defined in \eqref{eq:theta0}, $x \sim P_X$, $\mathcal D = \mathcal B(\theta_0, R)$ and assume that $R = o(1)$. Then, for any $0 \leq l \leq L-1$, we have
    \begin{equation}\label{eq:firststat1}
     \sup_{\theta\in\mathcal D}   \norm{f_l(\theta, x)}_2 = \bigO{\sqrt{d}},
    \end{equation}
    with probability at least $1 - C \exp(-c n_{L-1})$ over $(W_k(\theta_0))_{k=1}^l$ and $x$, where $c,C$ are numerical constants. Furthermore,
    \begin{equation}\label{eq:statind}
      \sup_{\theta_1, \theta_2\in\mathcal D}  \norm{f_l(\theta_1, x) - f_l(\theta_2, x)}_2 \le C\,R\,\sqrt{d},
    \end{equation}
    with probability at least $1 - C \exp (-c n_{L-1})$ over $(W_k(\theta_0))_{k=1}^l$ and $x$.
\end{lemma}

\begin{proof}
The first statement follows from the chain of inequalities below:
\begin{equation}
   \sup_{\theta\in\mathcal D}  \norm{f_l(\theta, x)}_2 \leq  \norm{f_l(\theta_0, x)}_2 + \sup_{\theta\in\mathcal D}\norm{f_l(\theta) - f_l(\theta_0)}_2 \le C\sqrt{n_l} + C\,R\sqrt{d},
\end{equation}
where the second inequality holds with probability at least $1 - C \exp(-c n_{L-1})$ over $(W_k(\theta_0))_{k=1}^l$ and $x$ by combining Lemma \ref{lemma:concnorm} and Lemma \ref{lemma:featuresinball1}.

We prove the second statement by induction over $l$. For the base case, we have $f_0(\theta_1, x) = f_0(\theta_2, x)$, hence \eqref{eq:statind} holds with probability 1.

For the induction case, let $l > 0$. Then,
\begin{equation}\label{eq:lip00x}
\begin{aligned}
\sup_{\theta_1, \theta_2\in\mathcal D}\norm{f_l(\theta_1, x) \hspace{-.2em}-\hspace{-.2em} f_l(\theta_2, x)}_2 \hspace{-.2em}&=\hspace{-.5em} \sup_{\theta_1, \theta_2\in\mathcal D}\hspace{-.2em}\norm{\phi\left( (W_l(\theta_1))^\top \hspace{-.1em}f_{l-1}(\theta_1, x )\right) \hspace{-.2em}- \hspace{-.2em}\phi\left( (W_l(\theta_2))^\top\hspace{-.1em} f_{l-1}(\theta_2, x)\right)}_2 \\
&\leq M\sup_{\theta_1, \theta_2\in\mathcal D}\norm{ (W_l(\theta_1))^\top f_{l-1}(\theta_1, x)-  (W_l(\theta_2))^\top f_{l-1}(\theta_2, x)}_2 \\
&\leq M\sup_{\theta_1, \theta_2\in\mathcal D}\norm{ (W_l(\theta_1))^\top f_{l-1}(\theta_1, x) -  (W_l(\theta_1))^\top f_{l-1}(\theta_2, x)}_2 \\
& \hspace{0.25cm} +M \sup_{\theta_1, \theta_2\in\mathcal D}\norm{ (W_l(\theta_1))^\top f_{l-1}(\theta_2, x)-  (W_l(\theta_2))^\top f_{l-1}(\theta_2, x)}_2\\
&\leq M\sup_{\theta_1\in\mathcal D}\opnorm{W_l(\theta_1)} \sup_{\theta_1, \theta_2\in\mathcal D}\norm{f_{l-1}(\theta_1, x) -  f_{l-1}(\theta_2, x)}_2 \\
& \hspace{0.25cm} +M \sup_{\theta_1, \theta_2\in\mathcal D}\opnorm{W_l(\theta_1) - W_l(\theta_2)}\sup_{\theta_2\in\mathcal D} \norm{f_{l-1}(\theta_2, x)}_2.
\end{aligned}
\end{equation}
By Lemma \ref{lemma:optW}, we have that
\begin{equation}\label{eq:lip0x}
\sup_{\theta_1\in\mathcal D}\opnorm{W_l(\theta_1)} = \bigO{1}, 
\end{equation}
with probability at least $1 - 2 \exp(-c n_{L-1})$ over $W_l(\theta_0)$. By inductive hypothesis, we have
\begin{equation}\label{eq:lip1x}
 \sup_{\theta_1, \theta_2\in\mathcal D}\norm{f_{l-1}(\theta_1, x) -  f_{l-1}(\theta_2, x)}_2 \le C \,R\,\sqrt{d},   
\end{equation}
with probability at least $1 - C\exp(-c n_{L-1})$ over $(W_k(\theta_0))_{k=1}^{l-1}$ and $x$. Clearly, we also have that
\begin{equation}\label{eq:lip2x}
 \sup_{\theta_1, \theta_2\in\mathcal D}\opnorm{W_l(\theta_1) - W_l(\theta_2)} \leq \sup_{\theta_1, \theta_2\in\mathcal D}\norm{W_l(\theta_1) - W_l(\theta_2)}_F \leq \sup_{\theta_1, \theta_2\in\mathcal D}\norm{\theta_1 - \theta_2} \le R.  
\end{equation}
Furthermore, by using \eqref{eq:firststat1}, we have that
\begin{equation}\label{eq:lip3x}
\sup_{\theta_2\in\mathcal D}    \norm{f_{l-1}(\theta_2, x)}_2 \le C\,R\,\sqrt{d},
\end{equation}
with probability at least $1 - C \exp(-c n_{L-1})$ over $(W_k(\theta_0))_{k=1}^{l-1}$ and $x$. By combining \eqref{eq:lip00x}, \eqref{eq:lip0x}, \eqref{eq:lip1x}, \eqref{eq:lip2x} and \eqref{eq:lip3x}, we obtain that
\begin{equation}
\begin{aligned}
 \sup_{\theta_1, \theta_2\in\mathcal D}\norm{f_l(\theta_1, x) - f_l(\theta_2, x)}_2 &\leq C\,R\sqrt{d},
\end{aligned}
\end{equation}
with probability at least $1 - C \exp (-c n_{L-1})$ over $(W_k(\theta_0))_{k=1}^l$ and $x$, which completes the proof.
\end{proof}

\begin{lemma}[$\ell_2$ norm of backpropagation in $\mathcal D$]\label{lemma:binball}
Let $\theta_0$ be defined in \eqref{eq:theta0}, $x \sim P_X$, and $\mathcal D = \mathcal B(\theta_0, R)$. Assume that $R = o(1)$ and that $\gamma>1$. Then, for any $l \in [L]$, we have
    \begin{equation}\label{eq:statbin}
    \sup_{\theta\in\mathcal D}     \norm{b_l(\theta, x)}_2 \le C\sqrt{\gamma \cdot n_{L-1}},
    \end{equation}
    with probability at least $1 - C \exp (-c n_{L-1})$ over $(W_k(\theta_0))_{k = l+1}^L$, where $c, C$ are numerical constants.
\end{lemma}
\begin{proof}
We prove the claim by induction on $l\in \{L, L-1, \ldots, 1\}$. For the base case, we have that $\norm{b_L(\theta, x)}_2 = 1$, hence \eqref{eq:statbin} clearly holds. 

For the induction case, pick $l \in [L-1]$. Then,
\begin{equation}
\begin{aligned}
\sup_{\theta\in\mathcal D}     \norm{b_l(\theta, x)}_2 &= \sup_{\theta\in\mathcal D}\norm{\left( \prod_{k = l}^{L-2} \Sigma_{k}(\theta, x) W_{k+1}(\theta)\right) \Sigma_{L-1}(\theta, x) W_L(\theta)}_2 \\
    &\leq \sup_{\theta\in\mathcal D}  \opnorm{\left(\prod_{k = l}^{L-2}\Sigma_{k}(\theta, x) W_{k+1}(\theta)\right) \Sigma_{L-1}(\theta, x)}  \sup_{\theta\in\mathcal D}\norm{W_L(\theta)}_2 \\
    &\leq  \left( \prod_{k = l}^{L-2}\sup_{\theta\in\mathcal D}\opnorm{\Sigma_{k}(\theta, x)} \sup_{\theta\in\mathcal D}\opnorm{W_{k+1}(\theta)}\right)\sup_{\theta\in\mathcal D} \opnorm{\Sigma_{L-1}(\theta, x)} \sup_{\theta\in\mathcal D}\norm{W_L(\theta)}_2 \\
    &\leq M^{L-l} \left(\prod_{k=l+1}^{L-1} \sup_{\theta\in\mathcal D}\opnorm{W_{k}(\theta)} \right)\left(\norm{W_L(\theta_0)}_2 +\sup_{\theta\in\mathcal D} \norm{W_L(\theta) - W_L(\theta_0)}_2\right) \\
    &\leq C\, M^{L-l} (\norm{W_L(\theta_0)}_2 + \sup_{\theta\in\mathcal D}\norm{\theta - \theta_0}_2) \\
    &\leq C\,M^{L-l}  (\sqrt{\gamma n_{L-1}} + \sup_{\theta\in\mathcal D}\norm{\theta - \theta_0}_2) \\
    &= C\sqrt{\gamma n_{L-1}}.
\end{aligned}
\end{equation}
Here, the fourth line follows from Assumption \ref{ass:activationfunc}, which gives $\sup_{\theta\in\mathcal D}\opnorm{\Sigma_k(\theta, x)} \leq M$; the fifth line holds with probability $1 - C \exp (-c n_{L-1})$ over $(W_k(\theta_0))_{k=l+1}^{L-1}$ by Lemma \ref{lemma:optW}; the sixth line holds with probability at least $1 - \exp (-c n_{L-1})$ over $W_L(\theta_0)$ by Theorem 3.1.1 in \cite{vershynin2018high}; and the last line follows from $R=o(1)$. Taking the intersection of these events gives the desired result.
\end{proof}

\begin{lemma}[$\ell_2$ distance of backpropagation in $\mathcal D$]\label{lemma:binball2}
Let $\theta_0$ be defined in \eqref{eq:theta0}, $x \sim P_X$, and $\mathcal D = \mathcal B(\theta_0, R)$. Assume that $R = o(1)$ and that $\gamma>1$. Then, for any $l \in [L]$, we have
    \begin{equation}\label{eq:statbinbac}
    \sup_{\theta_1, \theta_2\in\mathcal D}     \norm{b_l(\theta_1, x) - b_l(\theta_2, x)}_2 \le C\, R\, \sqrt{\gamma d n_{L-1}},
    \end{equation}
    with probability at least $1 - C \exp(-cn_{L-1})$ over $(W_k(\theta_0))_{k = l+ 1}^L$ and $x$, where $c,C$ are numerical constants.
\end{lemma}
\begin{proof}
We prove the claim by induction on $l\in \{L, L-1, \ldots, 1\}$. For the base case, $b_L(\theta, x)$ does not depend on $\theta$, hence \eqref{eq:statbinbac} clearly holds. 
For the induction case, pick $l\in [L-1]$. Then,
\begin{equation}\label{eq:i1}
    \begin{aligned}
\sup_{\theta_1, \theta_2\in\mathcal D}   &     \norm{b_l(\theta_1, x) - b_l(\theta_2, x)}_2 \\
&= \sup_{\theta_1, \theta_2\in\mathcal D}\norm{\Sigma_l(\theta_1, x) W_{l+1}(\theta_1) b_{l + 1}(\theta_1, x) - \Sigma_l(\theta_2, x) W_{l+1}(\theta_2) b_{l + 1}(\theta_2, x)}_2\\
        &\leq \sup_{\theta_1, \theta_2\in\mathcal D}\norm{\Sigma_l(\theta_1, x) W_{l+1}(\theta_1) b_{l + 1}(\theta_1, x) - \Sigma_l(\theta_1, x) W_{l+1}(\theta_1) b_{l + 1}(\theta_2, x)}_2 \\
        & \hspace{0.5cm} + \sup_{\theta_1, \theta_2\in\mathcal D}\norm{\Sigma_l(\theta_1, x) W_{l+1}(\theta_1) b_{l + 1}(\theta_2, x) - \Sigma_l(\theta_2, x) W_{l+1}(\theta_2) b_{l + 1}(\theta_2, x)}_2\\
        &\leq \sup_{\theta_1\in\mathcal D}\opnorm{\Sigma_l(\theta_1, x)} \sup_{\theta_1\in\mathcal D}\opnorm{W_{l+1}(\theta_1)} \sup_{\theta_1, \theta_2\in\mathcal D}\norm{b_{l+1}(\theta_1, x) - b_{l+1}(\theta_2, x)}_2 \\
        & \hspace{0.5cm} + \sup_{\theta_1, \theta_2\in\mathcal D}\opnorm{\Sigma_l(\theta_1, x) W_{l+1}(\theta_1) - \Sigma_l(\theta_2, x) W_{l+1}(\theta_2)} \sup_{\theta_2\in\mathcal D}\norm{b_{l+1}(\theta_2, x)}_2 \\
        &\leq \sup_{\theta_1\in\mathcal D}\opnorm{\Sigma_l(\theta_1, x)} \sup_{\theta_1\in\mathcal D}\opnorm{W_{l+1}(\theta_1)} \sup_{\theta_1,\theta_2\in\mathcal D}\norm{b_{l+1}(\theta_1, x) - b_{l+1}(\theta_2, x)}_2 \\
        & \hspace{0.5cm} +\sup_{\theta_1,\theta_2\in\mathcal D} \opnorm{\left( \Sigma_l(\theta_1, x) - \Sigma_l(\theta_2, x)\right) W_{l+1}(\theta_2)} \sup_{\theta_2\in\mathcal D}\norm{b_{l+1}(\theta_2, x)}_2 \\
        & \hspace{0.5cm} + \sup_{\theta_1,\theta_2\in\mathcal D}\opnorm{\Sigma_l(\theta_2, x) \left( W_{l+1}(\theta_1) -  W_{l+1}(\theta_2) \right)} \sup_{\theta_2\in\mathcal D}\norm{b_{l+1}(\theta_2, x)}_2.
    \end{aligned}
\end{equation}
Furthermore, we have that the following results hold. 
\begin{itemize}
    \item[(i)] By Assumption \ref{ass:activationfunc} and Lemma \ref{lemma:optW}, $$\sup_{\theta_1\in\mathcal D}\opnorm{\Sigma_l(\theta_1, x)}\sup_{\theta_1\in\mathcal D}\opnorm{W_{l+1}(\theta_1)} = \bigO{1},$$ with probability $1 - 2 \exp (-c n_{L-1})$ over $W_{l+1}(\theta_0)$;
    \item[(ii)] By inductive hypothesis, $$\sup_{\theta_1,\theta_2\in\mathcal D}\norm{b_{l+1}(\theta_1, x) - b_{l+1}(\theta_2, x)}_2 \le C \,R\,\sqrt{\gamma d n_{L-1}},$$ with probability at least $1 - C \exp(-cn_{L-1})$ over $(W_k(\theta_0))_{k=l+2}^L$ and $x$;
    \item[(iii)] By the same argument of the second statement in Lemma \ref{lemma:featuresinball2} and again Lemma \ref{lemma:optW}, \begin{equation*}
        \begin{split}
            \sup_{\theta_1,\theta_2\in\mathcal D}&\opnorm{\left( \Sigma_l(\theta_1, x) - \Sigma_l(\theta_2, x)\right) W_{l+1}(\theta_2)} \\
            &\leq \sup_{\theta_1,\theta_2\in\mathcal D}\norm{\phi'\left( g_l(\theta_1, x)\right) - \phi'\left( g_l(\theta_2, x) \right)}_2 \sup_{\theta_2\in\mathcal D}\opnorm{W_{l+1}(\theta_2)} \\
            &\le C\,R\,\sqrt{d},
        \end{split}
    \end{equation*} with probability at least $1 - C \exp(-c n_{L-1})$ over $(W_k(\theta_0))_{k=1}^{l+1}$ and $x$;
    \item[(iv)] By Lemma \ref{lemma:binball}, $$\sup_{\theta_2\in\mathcal D}\norm{b_{l+1}(\theta_2, x)}_2 \le C\, \sqrt{\gamma n_{L-1}},$$ with probability at least $1 - C \exp (-c n_{L-1})$ over $(W_k(\theta_0))_{k = l+1}^L$;
    \item[(v)] By Assumption \ref{ass:activationfunc}, $$\sup_{\theta_1,\theta_2\in\mathcal D}\opnorm{\Sigma_l(\theta_2, x) \left( W_{l+1}(\theta_1) -  W_{l+1}(\theta_2) \right)} \le C\,R .$$
\end{itemize}
By combining (i)-(v) with \eqref{eq:i1}, we conclude that
\begin{equation}
\begin{aligned}
   \sup_{\theta_1,\theta_2\in\mathcal D} \norm{b_l(\theta_1,x) - b_l(\theta_2,x)}_2 &\leq C\,R\,\sqrt{\gamma d n_{L-1}},
\end{aligned}
\end{equation}
with probability at least $1 - C \exp(-cn_{L-1})$ over $x$ and $(W_k(\theta_0))_{k=1}^L$, which concludes the proof.
\end{proof}

\begin{lemma}[Difference of Jacobians in $\mathcal D$]\label{lemma:jacobianinball}
Let $\theta_0$ be defined in \eqref{eq:theta0}, $x \sim P_X$, and $\mathcal D = \mathcal B(\theta_0, R)$. Assume that $R = o(1)$ and that $\gamma>1$. Then, we have
    \begin{equation}
\sup_{\theta_1,\theta_2\in\mathcal D}        \opnorm{J(\theta_1) - J(\theta_2)} \le C \,R\,d\sqrt{ \gamma n_{L-1}N},
    \end{equation}
    with probability at least $1 - C N \exp (-c n_{L-1})$ over $(x_i)_{i=1}^N$ and $\theta_0$, where $c, C$ are numerical constants.
\end{lemma}
\begin{proof}
Pick $i \in [N]$. Then, we have
\begin{equation}
    \begin{aligned}
\sup_{\theta_1,\theta_2\in\mathcal D}    &    \norm{(J(\theta_1))_{i:} - (J(\theta_2))_{i:}}^2_2 \\
&\le \sum_{l=0}^{L-1} \sup_{\theta_1,\theta_2\in\mathcal D}\norm{(F_l(\theta_1))_{:i} \otimes (B_{l+1}(\theta_1))_{:i} - (F_l(\theta_2))_{:i} \otimes (B_{l+1}(\theta_2))_{:i} }_2^2 \\
        &= \sum_{l=0}^{L-1} \sup_{\theta_1,\theta_2\in\mathcal D}\norm{f_l(\theta_1, x_i) \otimes b_{l+1}(\theta_1, x_i) - f_l(\theta_2, x_i) \otimes b_{l+1}(\theta_2, x_i)}_2^2\\
        &\leq \sum_{l=0}^{L-1} \sup_{\theta_1,\theta_2\in\mathcal D}\norm{\left( f_l(\theta_1, x_i) - f_l(\theta_2, x_i) \right) \otimes b_{l+1}(\theta_1, x_i)}_2^2 \\
        & \hspace{0.5cm} + \sum_{l=0}^{L-1} \sup_{\theta_1,\theta_2\in\mathcal D}\norm{f_l(\theta_2, x_i) \otimes \left(b_{l+1}(\theta_1, x_i) - b_{l+1}(\theta_1, x_i) \right) }_2^2\\
        &\leq \sum_{l=0}^{L-1} \sup_{\theta_1,\theta_2\in\mathcal D}\norm{f_l(\theta_1, x_i) - f_l(\theta_2, x_i)}_2^2 \sup_{\theta_1\in\mathcal D}\norm{b_{l+1}(\theta_1, x_i)}_2^2 \\
        & \hspace{0.5cm} + \sum_{l=0}^{L-1}\sup_{\theta_2\in\mathcal D} \norm{f_l(\theta_2, x_i)}_2^2 \sup_{\theta_1,\theta_2\in\mathcal D} \norm{b_{l+1}(\theta_1, x_i) - b_{l+1}(\theta_1, x_i)}_2^2.
    \end{aligned}
\end{equation}
Since $x_i \sim P_X$, we can merge together the results from Lemmas \ref{lemma:featuresinball1}, \ref{lemma:featuresinball2}, \ref{lemma:binball} and \ref{lemma:binball2} and obtain
\begin{equation}\label{eq:optboundonrows}
    \begin{aligned}
       \sup_{\theta_1,\theta_2\in\mathcal D} \norm{(J(\theta_1))_{i:} - (J(\theta_2))_{i:}}^2_2 &\leq C\gamma R^2 d^2 n_{L-1},
    \end{aligned}
\end{equation}
with probability at least $1 - C \exp (-c n_{L-1})$ over $x_i$ and $\theta_0$.

Therefore, we have
\begin{equation}
\begin{aligned}
   \sup_{\theta_1,\theta_2\in\mathcal D}  \opnorm{J(\theta_1) - J(\theta_2)} &\leq \sup_{\theta_1,\theta_2\in\mathcal D} \norm{J(\theta_1) - J(\theta_2)}_F \\
    &\le \sqrt{\sum_{i = 1}^N \sup_{\theta_1,\theta_2\in\mathcal D} \norm{(J(\theta_1))_{i:} - (J(\theta_2))_{i:}}^2_2} \\
    &\leq C\,R\,d\sqrt{\gamma n_{L-1}N},
\end{aligned}
\end{equation}
with probability $1 - C N \exp (-c n_{L-1})$ over $(x_i)_{i=1}^N$ and $\theta_0$.
\end{proof}

\begin{lemma}[NTK spectrum in $\mathcal D$]\label{lemma:opnormsigminball}
  Let $\theta_0$ be defined in \eqref{eq:theta0}, $x \sim P_X$, and $\mathcal D = \mathcal B(\theta_0, R)$. Assume that $R = o(1)$ and that $\gamma>1$. Then, we have
    \begin{equation}
      \sup_{\theta\in\mathcal D}  \opnorm{K(\theta)} \le C\,\gamma\, N\, d\, n_{L-1},
    \end{equation}
    with probability at least $1 - C\,N \exp(-c n_{L-1})$ over $\theta_0$ and $(x_i)_{i=1}^N$, where $c,C$ are numerical constants. Furthermore, 
    \begin{equation}
      \inf_{\theta\in\mathcal D}   \sigma_{\textup{min}}(J(\theta)) \geq c_1\sqrt{\gamma n_{L-2} n_{L-1}} - C_1\,R\,d\sqrt{\gamma n_{L-1}N},
    \end{equation}
    with probability at least $1 - C\,N e^{-c \log^2 n_{L-1}} - C e^{-c \log^2 N}$ over $\theta_0$ and $(x_i)_{i=1}^N$, where $c_1, C_1$ are also numerical constants.
\end{lemma}

\begin{proof}
We have
\begin{equation}
    \begin{aligned}
    \sup_{\theta\in\mathcal D}    \opnorm{K(\theta)} &= \sup_{\theta\in\mathcal D}\opnorm{\sum_{l = 0}^{L-1} F_l(\theta) F_l^\top(\theta) \circ B_{l+1}(\theta) B_{l+1}^\top(\theta)} \\
        &\leq \sum_{l = 0}^{L-1} \sup_{\theta\in\mathcal D}\opnorm{F_l(\theta) F_l^\top(\theta) \circ B_{l+1}(\theta) B_{l+1}^\top(\theta)} \\
        &\leq \sum_{l = 0}^{L-1} \sup_{\theta\in\mathcal D}\opnorm{F_l(\theta) F_l^\top(\theta)} \sup_{\theta\in\mathcal D}\max_{i\in[N]} \norm{\left( B_{l+1}(\theta) \right)_{i:}}_2^2 \\
        &\leq \sum_{l = 0}^{L-1}\sup_{\theta\in\mathcal D} \norm{F_l(\theta)}_F^2 \sup_{\theta\in\mathcal D}\max_{i\in [N]} \norm{b_{l+1}(\theta, x_i)}_2^2 \\
        &\le \sum_{l = 0}^{L-1} \left(\sum_{i=1}^N \sup_{\theta\in\mathcal D}\norm{f_l(\theta, x_i)}_2^2\right) \sup_{\theta\in\mathcal D}\max_i \norm{b_{l+1}(\theta, x_i)}_2^2.
    \end{aligned}
\end{equation}
By Lemma \ref{lemma:featuresinball2}, we have that $\sup_{\theta\in\mathcal D}\norm{f_l(\theta, x_i)}_2^2 \le C\,d$ with probability at least $1 - C \exp(-c n_{L-1})$ over $(W_k(\theta_0))_{k=1}^l$ and $x_i$, for any $0 \leq l \leq L-1$ and $i \in [N]$. By Lemma \ref{lemma:binball}, we have that $\sup_{\theta\in\mathcal D}\norm{b_l(\theta, x_i)}_2^2 \le C \gamma n_{L-1}$ with probability at least $1 - C \exp (-c n_{L-1})$ over $(W_k(\theta_0))_{k = l+1}^L$, for any $l \in [L]$ and $i \in [N]$. Therefore, we obtain
\begin{equation}
\sup_{\theta\in\mathcal D}    \opnorm{K(\theta)} \le C\,\gamma\, N\, d\, n_{L-1},
\end{equation}
with probability at least $1 - C\,N \exp(-c n_{L-1})$ over $\theta_0$ and $(x_i)_{i=1}^N$, which gives the first statement of the lemma.

By using Weyl's inequality, we get
\begin{equation}
\begin{aligned}
 \inf_{\theta\in\mathcal D}   \sigma_{\text{min}}(J(\theta)) &\geq \sigma_{\text{min}}(J(\theta_0)) - \sup_{\theta\in\mathcal D}\opnorm{J(\theta_1) - J(\theta_0)} \\
    &\geq c_1\sqrt{\gamma n_{L-2} n_{L-1}} - C_1\,R\,d\,\sqrt{\gamma \,n_{L-1}\,N},
\end{aligned}
\end{equation}
where the last inequality follows from Lemma \ref{lemma:firstlemmaopt} and Lemma \ref{lemma:jacobianinball}, and it holds with probability $1 - C\,N e^{-c \log^2 n_{L-1}} - C e^{-c \log^2 N}$ over $(x_i)_{i=1}^N$ and $\theta_0$. This gives the second statement of the lemma and concludes the proof.
\end{proof}

Armed with Proposition \ref{prop:optim} and the intermediate estimates of Lemmas \ref{lemma:firstlemmaopt}-\ref{lemma:opnormsigminball}, we are finally ready to prove Theorem \ref{thm:optimization}.

\begin{proof}[Proof of Theorem \ref{thm:optimization}]
We show that there exist two absolute constants $\tilde c$ and $\tilde C$ such that
    \begin{equation}\label{eq:optalpha}
        \alpha = \tilde c \sqrt{\gamma n_{L-2} n_{L-1}}
    \end{equation}
    and
    \begin{equation}\label{eq:betadef}
        \beta = \tilde C \sqrt{\gamma Nd n_{L-1}}
    \end{equation}
    satisfy the two assumptions in Proposition \ref{prop:optim} with initialization $\tilde\theta_0:=\theta_0$, where $\theta_0$ is defined in \eqref{eq:theta0}. This holds with probability at least $1 - C\,N e^{-c \log^2 n_{L-1}} - C e^{-c \log^2 N}$ over $(x_i)_{i=1}^N$ and $\theta_0$. 
    
    Recall from Proposition \ref{prop:optim} that $R$ is defined as $ 4\norm{F_L(\theta_0) - Y}_2/\alpha$, since we have set $\tilde\theta_0=\theta_0$. For the moment, we assume that
    \begin{equation}\label{eq:Rassm}
        R=\bigO{\sqrt{\frac{N}{\gamma n_{L-2}n_{L-1}}}},
    \end{equation}
    and we will verify that this is the case later. Note that $\gamma=d^3N^2>1$ and, hence, \eqref{eq:Rassm} and Assumption \ref{ass:overparam} imply that $R=o(1)$. Thus, we can apply Lemma \ref{lemma:opnormsigminball} and obtain
\begin{equation}\label{eq:optboundsigmamin}
\begin{aligned}
   \inf_{\theta\in\mathcal D}  \sigma_{\text{min}}(J(\theta)) &\geq c_1\sqrt{\gamma n_{L-2} n_{L-1}} - C_1\,R\,d\sqrt{\gamma n_{L-1}N} \geq \tilde{c}\sqrt{\gamma n_{L-2} n_{L-1}},
\end{aligned}
\end{equation}
with probability at least $1 - C\,N e^{-c \log^2 n_{L-1}} - C e^{-c \log^2 N}$ over $(x_i)_{i=1}^N$ and $\theta_0$, where the last inequality uses \eqref{eq:Rassm}. This shows that the lower bound in \eqref{eq:optass1} holds. 

Now, by using \eqref{eq:optboundsigmamin}, we verify that \eqref{eq:Rassm} holds. Recall that, by assumption of the theorem, $\norm{Y}_2 = \Theta(\sqrt{N})$. Furthermore, by Lemma \ref{lemma:firstlemmaopt}, $F_L(\theta_0)$ is a vector of all zeros. Then, 

\begin{equation}
    R = \frac{4\norm{F_L(\theta_0) - Y}_2}{\alpha} = \frac{4\norm{Y}_2}{\alpha} = \bigO{\sqrt{\frac{N}{\gamma n_{L-2} n_{L-1}}}}.
\end{equation}

By Lemma \ref{lemma:opnormsigminball}, we have that
\begin{equation}
 \sup_{\theta\in\mathcal D}   \opnorm{J(\theta)} \le C \sqrt{\gamma N d n_{L-1}},
\end{equation}
with probability at least $1 - C\,N \exp(-c n_{L-1})$ over $\theta_0$. Thus, by our choice \eqref{eq:betadef} of $\beta$, we obtain that the upper bound in \eqref{eq:optass1} holds.   

Next, we verify the second assumption of Proposition \ref{prop:optim}. To do so, let us write 
\begin{equation}\label{eq:optalphabeta}
    \frac{\alpha^2}{2\beta} = \frac{\tilde{c}^2 n_{L-2}n_{L-1} \gamma}{2\tilde C \sqrt{\gamma Nd n_{L-1}}} = \Omega (\sqrt{n_{L-2}} N d ),
\end{equation}
where we have used Assumption \ref{ass:overparam}. Thus, 
\begin{equation}
  \sup_{\theta_1, \theta_2\in\mathcal D}      \opnorm{J(\theta_1) - J(\theta_2)} \le C \,R\,d\sqrt{ \gamma n_{L-1}N} = \bigO{\frac{d N}{\sqrt{n_{L-2}}}}\le  \frac{\alpha^2}{2\beta},
\end{equation}
with probability at least $1 - C\, N \exp (-c n_{L-1})$ over $(x_i)_{i=1}^N$ and $\theta_0$. Here, the first passage follows from Lemma \ref{lemma:jacobianinball}, in the second passage we use \eqref{eq:Rassm}, and in the last one we use \eqref{eq:optalphabeta}. This completes the proof of \eqref{eq:optass2} and also of the theorem, since the desired claim follows from an application of Proposition \ref{prop:optim}.
\end{proof}

\end{document}